\documentclass{article}
\usepackage{iclr2027_conference,times}

\usepackage[T1]{fontenc}
\usepackage[utf8]{inputenc}
\DeclareUnicodeCharacter{266B}{{\fontfamily{lmr}\selectfont\textmusicalnote}}%
\usepackage{microtype}

\usepackage{amsmath, amssymb, amsthm}
\usepackage{mathtools}
\usepackage{bm}

\usepackage{graphicx}
\usepackage{wrapfig}
\usepackage{tikz}   %
\usepackage{pifont} %
\makeatletter
\g@addto@macro\normalsize{\abovedisplayskip 4pt plus 1pt minus 2pt \belowdisplayskip\abovedisplayskip
  \abovedisplayshortskip 0pt plus 1pt \belowdisplayshortskip 2pt plus 1pt minus 1pt}
\makeatother
\usepackage{booktabs}
\usepackage{adjustbox}
\usepackage{multirow}
\usepackage{enumitem}
\usepackage{xcolor}
\usepackage{colortbl}
\usepackage{float}
\usepackage{tcolorbox}
\tcbuselibrary{breakable}
\tcbuselibrary{skins}
\usepackage{needspace}
\newtcolorbox{probbox}[1]{enhanced,colback=white,colframe=blue!40!black,colbacktitle=blue!40!black,
  coltitle=white,fonttitle=\bfseries\small,fontupper=\small,boxrule=0.6pt,arc=2pt,
  left=6pt,right=6pt,top=3pt,bottom=3pt,title={#1},before skip=6pt,after skip=6pt}
\newtcolorbox{roundbox}[1]{enhanced,breakable,colback=white,colframe=black!72,colbacktitle=black!72,
  coltitle=white,fonttitle=\bfseries\small,fontupper=\small,boxrule=0.6pt,arc=2pt,
  left=6pt,right=6pt,top=4pt,bottom=4pt,title={#1},before skip=7pt,after skip=7pt,
  before upper={\raggedright},colbacklower=gray!6,fontlower=\small,before lower={\raggedright},
  segmentation style={gray!55,line width=0.4pt}}
\newtcolorbox{skillbox}[1]{enhanced,breakable,colback=white,colframe=gray!60,colbacktitle=gray!18,
  coltitle=black,fonttitle=\bfseries\small,fontupper=\footnotesize,boxrule=0.5pt,arc=2pt,
  left=6pt,right=6pt,top=3pt,bottom=3pt,title={#1},before skip=6pt,after skip=6pt}
\newtcolorbox{outbox}{enhanced,colback=gray!3,colframe=gray!55,boxrule=0.4pt,arc=2pt,
  left=6pt,right=6pt,top=3pt,bottom=3pt,fontupper=\small,before skip=3pt,after skip=4pt}
\newcommand{\rl}[1]{\par\noindent\textbf{#1:}\ }
\newcommand{\loghead}[1]{\par\Needspace{5\baselineskip}\vspace{4pt}\noindent{\normalsize\scshape #1}\par\nopagebreak\vspace{-4pt}}
\newcommand{\rsep}{\par\medskip\noindent\tikz\draw[dashed,gray!65,line width=0.4pt](0,0)--(\linewidth,0);\par\smallskip}
\newcommand{\nd}[1]{{\setlength{\fboxsep}{1.5pt}\fbox{#1}}}
\newcommand{\up}[1]{\textcolor{green!45!black}{#1}}
\newcommand{\dn}[1]{\textcolor{red!70!black}{#1}}
\newcommand{\cmark}{\up{\ding{51}~correct}}
\newcommand{\xmark}[1]{\dn{\ding{55}~#1}}
\newcommand{\kd}[1]{\textcolor{blue!70!black}{#1}}
\newcommand{\act}[1]{{\ttfamily\footnotesize #1}}
\newcommand{\slot}[1]{\textrm{\itshape$\langle$#1$\rangle$}}
\floatstyle{ruled}
\newfloat{algorithm}{htbp}{loa}
\floatname{algorithm}{Algorithm}
\floatstyle{plain}

\usepackage{hyperref}
\usepackage{url}
\usepackage{cleveref}

\newcommand{\E}{\mathbb{E}}
\newcommand{\Prob}{\mathbb{P}}
\newcommand{\indic}{\mathbb{I}}

\newcommand{\Slib}[1]{\mathcal{S}^{(#1)}}

\DeclareMathOperator{\clip}{clip}

\definecolor{ForestGreen}{rgb}{0.0,0.45,0.13}
\definecolor{darkgreen}{rgb}{0,0.5,0}
\definecolor{evopink}{HTML}{D6337F} %
\definecolor{evoblue}{HTML}{3C78A8} %

\theoremstyle{definition}
\newtheorem{definition}{Definition}
\newtheorem{assumption}{Assumption}
\theoremstyle{plain}
\newtheorem{proposition}{Proposition}
\newtheorem{lemma}[proposition]{Lemma}
\newtheorem{corollary}[proposition]{Corollary}
\theoremstyle{remark}
\newtheorem{remark}{Remark}

\title{\centering\hyphenpenalty=10000\exhyphenpenalty=10000
       EvoSteer: Online Self-Evolving Graph Orchestration
       via Reference-Anchored Credit Assignment}

\author{Mingda Zhang$^{1}$, Hanwen Zhang$^{2}$, Qiang Huang$^{3}$, Zijia Wang$^{4}$,\\
Pengfei Guo$^{5}$, Yuchen Zhang$^{6}$, Jionghao Zhu$^{1}$, Xiaoying Tang$^{1}$\\[2pt]
{\normalfont\footnotesize
$^{1}$The Chinese University of Hong Kong, Shenzhen, China\quad
$^{2}$Dalian University of Technology, China}\\
{\normalfont\footnotesize
$^{3}$Fudan University, China\quad
$^{4}$University of Oxford, UK\quad
$^{5}$North China Electric Power University, China}\\
{\normalfont\footnotesize
$^{6}$The University of Texas Health Science Center at Houston, USA}\\[3pt]
{\normalfont\footnotesize\texttt{mingdazhang@ieee.org}}}
\makeatletter
\def\@maketitle{\vbox{\hsize\textwidth
{\LARGE\sc \@title\par}
\vskip\ARXIVGAPA
{\centering\bfseries\begin{tabular}[t]{@{}c@{}}\@author\end{tabular}\par}
\vskip\ARXIVGAPB}}
\makeatother
\newlength{\ARXIVGAPA}
\newlength{\ARXIVGAPB}

\iclrfinalcopy  %

\newcolumntype{C}[1]{>{\centering\arraybackslash}p{#1}}%
\hypersetup{hidelinks}%

\begin{document}
\vspace*{-29pt}
\maketitle
\lhead{Under review as a conference paper at ICLR 2027}

\begin{abstract}
In recent years, LLM-based multi-agent systems have been widely applied to
orchestrate tool-using agents into executable communication graphs.
However, existing self-evolving orchestration still faces key challenges,
including \emph{post-hoc evolution} that revises the team only after the
trajectory ends, \emph{credit diffusion} that gives every action the same
terminal advantage under confounded baselines, and \emph{skill admission} that
is uncalibrated and never retired.
To address these challenges, we propose EvoSteer, a new paradigm of Online
Self-Evolving Graph Orchestration---the orchestrator builds a running team and repairs its
plausible but failing steps from execution features and a learned value estimate.
To support this paradigm, we introduce \emph{Anchored Trajectory Balance}
(AnchorTB), a regression-style flow-matching loss that assigns each
orchestration action a coefficient by balancing subtrajectories against
a frozen reference.
Built on the learned flow, we further propose \emph{Validated Skill Admission},
in which a candidate skill is tried before promotion and promoted only if
paired evidence passes a sequential test under a shared nominal testing budget.
Moreover, AnchorTB combines measured task-level reference reward
statistics with prefix-dependent corrections.
Experimental results on twelve datasets show that EvoSteer significantly
outperforms baselines across question answering, mathematical reasoning, code
generation, and interactive decision making.
Our code is available at \href{https://github.com/beita6969/evosteer}{\nolinkurl{https://github.com/beita6969/evosteer}}.
\end{abstract}

\section{Introduction}
\suppressfloats[t]

In recent years, a variety of powerful LLM-based multi-agent systems have been
applied to solve a wide range of complex tasks
\citep{yao2022react,hong2024metagpt,zhuge2024language}, gradually moving beyond
a single model call toward teams of tool-using agents that complete tasks
end-to-end.

\begin{wrapfigure}{r}{0.68\textwidth}
\centering
\includegraphics[width=\linewidth,trim=26 31 27 22,clip]{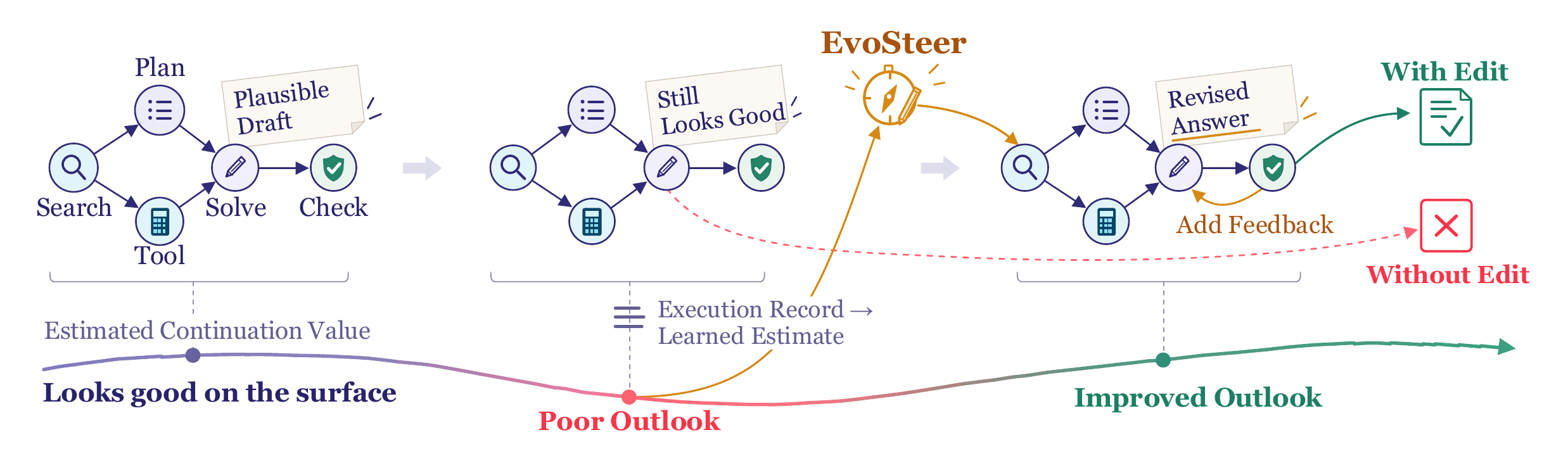}
\caption{Outputs that look plausible can still be headed for failure. EvoSteer reads
a low continuation estimate off the execution record and edits the running
team, here routing the checker's report back to the solver; dashed: the same
team left unedited.}
\label{fig:overview}
\end{wrapfigure}

In this process, graph orchestration has become a key bridge from task goals to
reproducible execution: by assigning each node a role and bound skills and each
edge a communication protocol on an agent communication graph (Fig.~\ref{fig:overview}, left), agents can
complete complex tasks with improved controllability, compositionality, and
reusability \citep{zhang2025aflow,zhang2026flowsteer}.
However, in practice, orchestration still evolves only between trajectories,
by revising skills, prompts, or topology after a batch of runs has ended
\citep{li2026experience,pan2026skillmas}, making a misconfigured team costly to detect, as its outputs still look
plausible, and impossible to repair while the workflow is still being executed.

\begin{figure}[t]
\centering
\includegraphics[width=\textwidth,trim=12 45 9 45,clip]{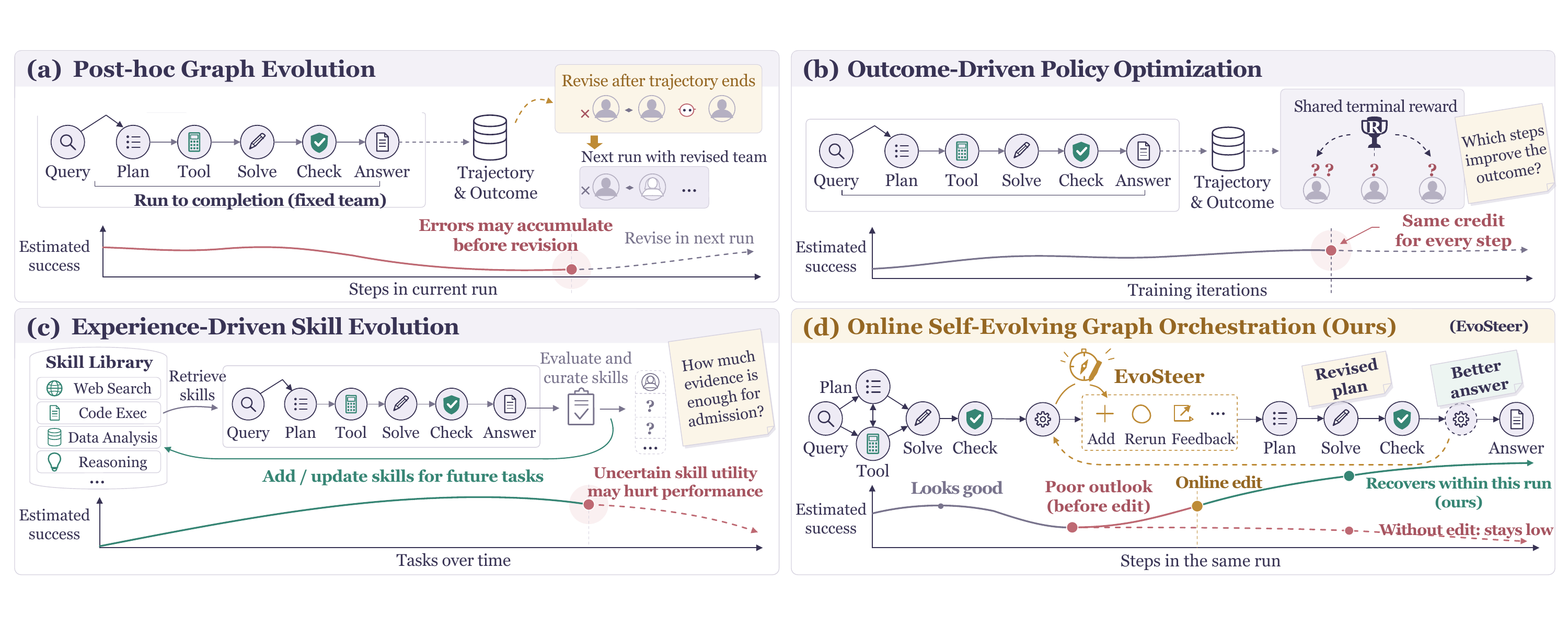}
\caption{Three lines of work and ours. \textbf{(a)} Post-hoc evolution revises the
team only after the run. \textbf{(b)} Outcome-driven optimization spreads one
terminal reward over every step. \textbf{(c)} Skill evolution admits by judges
or replay. \textbf{(d)} EvoSteer edits the team inside the same run, deciding
from measured execution features and a reference value estimate, and admits a
skill only after a sequential test.}
\label{fig:paradigms}
\end{figure}

To address these issues, three main lines of work have emerged, as shown in
Figure~\ref{fig:paradigms}.
First, \emph{self-evolving orchestration} adapts topologies, role prompts, and
skill libraries from the outcomes of completed trajectories
\citep{li2026experience,pan2026skillmas,zhang2026evochamber}, treating the team as an object of
learning.
Second, \emph{credit assignment for multi-agent LLMs} converts a shared terminal
reward into per-agent or per-message signals through learned critics,
leave-one-out baselines, counterfactual replay, and flow-matching objectives
\citep{chen2026exact,li2026counterfactual,shah2026causal,zhang2026skillflow}, telling which step was
responsible for the outcome.
Third, \emph{skill libraries} let capability grow across tasks by extracting
reusable procedures from trajectories and curating them with judges, replay
checks, or posteriors \citep{wang2023voyager,yang2026skillforge,wu2026bayesian}.

However, these methods still face three challenges.
\textbf{(i) Post-hoc evolution.} Self-evolving orchestration revises the team
after the trajectory ends, by rounds or batches \citep{pan2026skillmas,li2026experience};
during execution the orchestrator can only keep adding agents, and the decisive
error is locally indistinguishable from a recoverable step
\citep{zhang2026agentforesight}, so a wrong agent is neither rerun nor removed until
the task has failed.
\textbf{(ii) Credit diffusion.} Under a shared terminal reward, trajectory-level
methods give every action the same coefficient of a single terminal advantage
\citep{chen2026exact,zhang2026reinforcement}, and learned state values confound how often a task is
solved, execution noise, and structural choice, so the
orchestrator cannot tell whether a low score came from a task the reference
rarely solves or from a badly organized team; counterfactual
replay localizes the failing step \citep{shah2026causal,bonagiri2026causalflow} but does not
provide a training signal for every decision.
\textbf{(iii) Uncalibrated skill admission.} Skills are admitted by LLM
judgement or by success counts treated as reliable belief
\citep{yang2026skillforge,wu2026bayesian}; such point estimates cannot separate
insufficient evidence from a confirmed effect, so an accidental strategy from one
trajectory is consolidated \citep{chen2026skillcat}, and since no test tracks
whether an admitted skill still helps on later tasks, it stays in the library
and is never retired.

To address these challenges, we propose \textbf{EvoSteer} (Fig.~\ref{fig:paradigms}d), a new
paradigm of \emph{Online Self-Evolving Graph Orchestration}, in which the
orchestrator dynamically builds and repairs a running team from execution
features and a learned value estimate---replacing the run-then-revise loop.
The orchestrator issues one atomic edit per turn, rerun and drop included; each
is executed at once, and its feedback and the value estimate enter the state,
so a plausible but failing result can be redone mid-run.
To support this paradigm, we introduce \emph{Anchored Trajectory Balance}
(AnchorTB), a regression-style flow-matching loss that assigns each
orchestration action a coefficient by balancing subtrajectories against
a frozen reference.
It combines measured task-level reference reward statistics with
prefix-dependent corrections.
Built on the learned flow, we further propose \emph{Validated Skill Admission}:
a skill proposer distils candidates from scored experience, the orchestrator may
try a candidate before promotion, and paired rollouts differing only in the
candidate are judged by a sequential sign test under a shared nominal testing budget,
which promotes, retires, or defers.
Trained end-to-end on natural and paired trajectories with the reference frozen
throughout, EvoSteer turns construction, repair, and skill growth into
one set of learned decisions under a single training objective.

We evaluate on benchmarks across question answering, mathematical reasoning,
interactive decision making, and code generation.
Results show EvoSteer outperforms workflow search, flow- and RL-based
orchestration training, and skill-evolution methods on every benchmark and
improves all seven executors---a foundation for orchestration that evolves while
it runs.

\section{Related Work}

\paragraph{Agent Task Orchestration.}
Orchestration graphs are searched under execution feedback
\citep{zhang2025aflow,zhang2026flowsteer}. Later work evolves the team between runs
\citep{dang2026multi,hao2026evolve}, mutates topology between test-time rounds
\citep{xu2026tacomas}, repairs finished trajectories
\citep{luan2026repair,lu2026autonomous,zhao2026agenttether}, or audits unfolding prefixes
\citep{wang2026fail,jiang2026don,li2026last}, but keeps the running team fixed.
Credit comes from counterfactual replay \citep{chen2026lemon,deshmukh2026cosac},
group-relative objectives \citep{mishra2026policy}, or (sub)trajectory balance
\citep{malkin2022trajectory,madan2023learning,venkatraman2024amortizing} and its extensions
\citep{fawkes2026f,liu2026gflowrl,wang2026rtmc,liang2026granularity}. EvoSteer reruns and drops agents in the running team and
measures flows from a frozen reference.

\paragraph{Skill Self-Evolution.}
Skill libraries check candidates by replay \citep{he2026skillcommit}, paired
trajectories \citep{gao2026skillaudit}, leave-one-out attribution
\citep{shen2026dynamic,hu2026skillbrew,wang2026not}, or success-count posteriors
\citep{wu2026bayesian}. Others couple per-skill credit to library edits
\citep{yao2026maskills}, gate contaminated candidates \citep{zhang2026consistencygate}, or
gate self-modification with sequential or paired tests
\citep{shawn2026pace,sengupta2026self,shang2026hypothesis}. EvoSteer gates skills in a live episode under one
testing budget per run, and candidates stay selectable while tested.

\section{Preliminaries}

\noindent\textbf{Definition 1: Typed Agent Communication Graph.}
A typed agent communication graph (henceforth communication graph) is a
directed graph $G=(V,E,o)$ over agent nodes, where each node
$v_i=(c_i,\mathcal{K}_i)$ carries a role $c_i$ from a role catalogue and a set
of bound skills $\mathcal{K}_i$ drawn from the visible active skill set, each edge
$(v_i,v_j,\pi_{ij})$ carries a communication protocol $\pi_{ij}$ that decides
whether and how the source's result revises the target, and $o\in V$ together
with an output rule fixes the final answer (our typed edges and nodes extend the
untyped topologies of \citet{zhang2024g,zhang2025cut} and the
computational graphs of \citet{zhuge2024language}).

\noindent\textbf{Definition 2: Orchestration Trajectory.}
A complete sequence of $T$ orchestration actions from the empty team to a
legal stop defines an orchestration trajectory, in which every action is
executed as soon as it is issued, before the next action is chosen:
\begin{equation}
  \tau=\bigl\{(a_t,\ o_t^{\mathrm{exec}})\bigr\}_{t=0}^{T-1}\ \Rightarrow\ r(x),
  \qquad
  s_{t+1}=s_t\oplus(a_t,\ o_t^{\mathrm{exec}}),
\end{equation}
where $a_t$ is an atomic edit with type in $\{$\textsc{add\_agent},
\textsc{add\_edge}, \textsc{bind\_skill}, \textsc{set\_output},
\textsc{rerun\_agent}, \textsc{drop\_agent}, \textsc{stop}$\}$,
$o_t^{\mathrm{exec}}$ is the execution feedback returned by running the agent,
communication, repair, or output that the action triggers, and $x$ is the
completed history $x=s_T$ with task reward $r(x)\in[0,1]$ and $T\ge1$.
Action $a_t$ is chosen at $s_t$ and produces $s_{t+1}$, for
$t=0,\ldots,T-1$. Because each state retains the full history, every
non-root state has a unique parent. The deterministic feature record is
suppressed here and made explicit in Eq.~\eqref{eq:interleave}.

\noindent\textbf{Problem Statement.}
Given a task $q$, a frozen executor with tools, a frozen reference policy
$\rho$, and a skill library $\mathcal{S}$, we seek an orchestrator $\pi_\theta$
trained toward an ideal reward-proportional history distribution relative to $\rho$
\citep{venkatraman2024amortizing}:
\begin{equation}
  P^{*}(x\mid C)=\frac{\rho(x\mid C)\,R_\beta(x)}{Z(C)},
  \qquad
  R_\beta(x)=1+(e^{\beta}-1)\,r(x),
  \label{eq:target}
\end{equation}
where $0\le\beta<\infty$, and $C$ collects the task, tools, roles, and the skill and value-head snapshots
of the current batch, $\rho(x\mid C)$ is the history measure induced by the
reference policy and the executor's stochastic transitions, and $R_\beta$ keeps
a positive mass for failures while weighting a full score $e^{\beta}$ times a
zero score. Equation~\eqref{eq:target} is the training target; any deterministic executor,
and more generally one meeting the condition of
Proposition~\ref{prop:evs-realizability}, lets action choices alone realize it.

\section{Methodology: EvoSteer}

\begin{figure}[t]
\centering
\includegraphics[width=\textwidth,trim=15.5 12 15.5 0.5,clip]{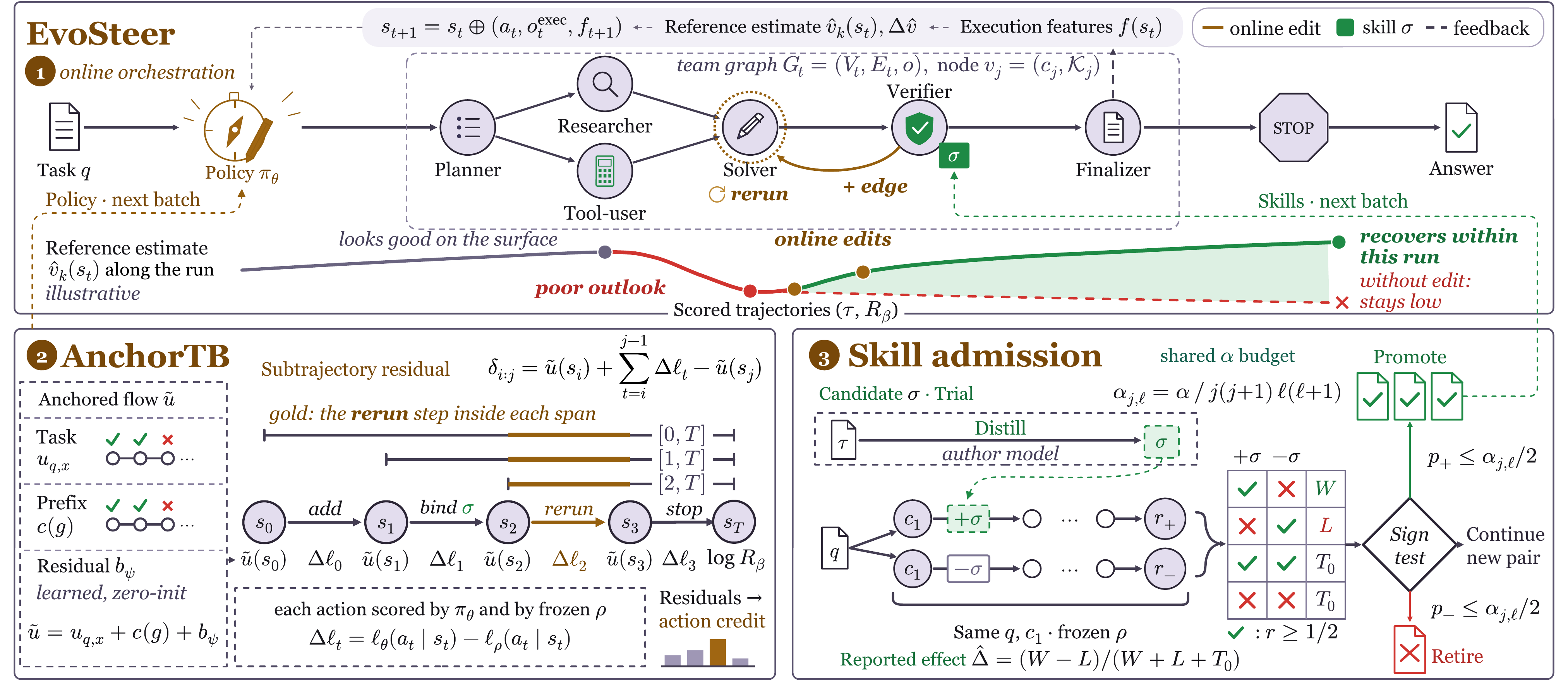}
\caption{EvoSteer architecture. \textbf{Top:} $\pi_\theta$ builds and runs the team
from execution features $f_t$ and reference estimate $\hat v_k$; rerun and new
edges are ordinary actions. \textbf{Bottom left:} AnchorTB balances subtrajectories
against $\rho$; $\tilde u,\delta$ abbreviate trajectory-specific flow estimates and
residuals. \textbf{Bottom right:} paired rollouts and a sign test under a shared
nominal $\alpha$ budget promote or retire candidate $\sigma$.}
\label{fig:arch}
\end{figure}

As illustrated in Figure~\ref{fig:arch}, this section introduces the EvoSteer framework,
including online self-evolving graph orchestration (Section~4.1), anchored
trajectory balance (Section~4.2), and validated skill admission (Section~4.3), all built on the frozen reference $\rho$.

\subsection{Online Self-Evolving Graph Orchestration}
As shown in Figure~\ref{fig:arch} (top), EvoSteer follows an Orchestrator--Executor paradigm
\citep{yao2022react,zhu2026recognize,zhang2026flowsteer}: a trainable orchestrator $\pi_\theta$
organizes tool-using agents, each on a frozen executor, inside a structured environment.

\noindent\textbf{Structured Environment.}
The environment $\mathcal{E}=(\mathcal{P},\Slib{k},\mathcal{R}_{\mathrm{exec}},\hat v_{k})$ maintains
the role catalogue $\mathcal{P}$, i.e., the roles with their tool and session permissions;
the active skill set $\Slib{k}$ visible in training step $k$ (slot counts in Appendix~C);
the executor runtime $\mathcal{R}_{\mathrm{exec}}$ with a shared resource budget; and the
reference value head $\hat v_k$ of Eq.~\eqref{eq:value}. The two snapshots $\Slib{k}$ and $\hat v_k$
are fixed for a whole rollout batch, while the legal mask $\mathcal A(s)$ remains state-dependent.

\noindent\textbf{Interleaved Execution.}
Each orchestrator action is validated, applied to the graph, and executed at
once; its feedback and a fixed-length execution feature vector are appended to the state:
\begin{equation}
  s_{t+1}=s_t\oplus\bigl(a_t,\ o_t^{\mathrm{exec}},\ f_{t+1}\bigr),
  \qquad
  f_{t+1}=f\bigl(s_t,a_t,o_t^{\mathrm{exec}}\bigr)\in\mathbb R^{30},
  \label{eq:interleave}
\end{equation}
where $o_t^{\mathrm{exec}}$ is the execution feedback of Definition~2 and $f$ records graph size,
action type, repair outcome, whether agents answered and agreed, communication revision,
environment phase, and budget use. These features need no extra model call and cover the
execution-side failure modes catalogued for multi-agent systems \citep{cemri2026multi}; $f(s)$
denotes the latest vector in $s$.

\noindent\textbf{Learnable Repair.}
Repair shares one action space and one policy $a_t\sim\pi_\theta(\cdot\mid s_t)$ with construction:
\begin{equation}
  \mathcal A(s_t)\subseteq\mathcal A_{\mathrm{con}}\cup\mathcal A_{\mathrm{rep}}\cup\{\textsc{stop}\},
  \qquad
  \mathcal A_{\mathrm{rep}}=\{\textsc{rerun\_agent},\ \textsc{drop\_agent}\},
  \label{eq:actions}
\end{equation}
where $\mathcal A_{\mathrm{con}}=\{\textsc{add\_agent},\textsc{add\_edge},\textsc{bind\_skill},\textsc{set\_output}\}$
holds the construction edits of Definition~2, and the mask $\mathcal{A}(s_t)$ \citep{li2026multi}
follows from the current graph, roles, skill slots, output modes, and budget, admitting
repairs only when valid. One objective trains all three kinds of decision: how many
agents a task needs, when a result is worth redoing, and when the answer in hand suffices.
Team size and repair count follow from the policy under the shared budget, and the
terminal reward $r(x)$ is computed once, after the stop, on the output node's answer.

\noindent\textbf{In-Loop Reference Value Estimate.}
A lightweight head uses execution features and task type to estimate
reference continuation reward by squared-error regression:
\begin{equation}
  \hat v_\omega(s)=\sigma\bigl(\mathrm{MLP}_\omega[f(s);\,\mathrm{onehot}(\mathrm{tasktype}(q))]\bigr),
  \qquad
  \mathcal{L}_{v}=\E_{(s,r)\sim\mathcal{D}_\rho}\bigl(\hat v_\omega(s)-r\bigr)^2,
  \label{eq:value}
\end{equation}
where $\mathrm{MLP}_\omega$ is a two-layer head on $30{+}6$ inputs (Appendix~B), $\mathcal{D}_\rho$ holds
states of legal reference continuations with their terminal rewards, and
$\hat v_k$ in $\mathcal{E}$ is the snapshot of $\hat v_\omega$ frozen for
batch $k$. The estimate and its one-step change are
shown to the orchestrator as feedback; both are deterministic readouts of the
retained history under the frozen head and carry cross-task experience
into in-loop decisions, including when to stop. Its population MSE optimum is the
feature-conditional mean of the terminal reward under the training data law, and this
optimum is calibrated (Proposition~\ref{prop:evs-value-projection}).

\noindent\textbf{Proposition 1.} \emph{Executing every action as it is
issued and placing rerun and drop in the action space lets the orchestrator
edit and retry a running team through the same action interface.}
\emph{Proof.} Appendix~\ref{sec:evs-C6}, via the history tree and legal
actions of Lemma~\ref{lem:evs-history-tree}.

\subsection{Anchored Trajectory Balance}
As shown in Figure~\ref{fig:arch} (bottom left) and step by step in Figure~\ref{fig:anchortb-steps}, we
train $\pi_\theta$ on the history tree of Section~3 toward the ideal reward-proportional
target of Eq.~\eqref{eq:target}, relative to the frozen reference.

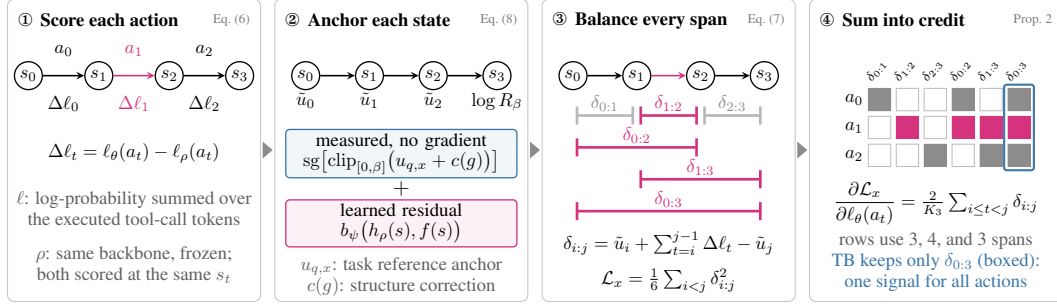
\begin{figure}[t]
\centering
\scalebox{0.7}{%
\begin{tikzpicture}[x=1.428571pt,y=1.428571pt,font=\fontsize{10}{11.43}\selectfont,
  st/.style={circle,draw,line width=0.714pt,inner sep=0pt,minimum size=15pt},
  act/.style={->,>=stealth,line width=0.857pt},
  lab/.style={inner sep=1.143pt},
  meas/.style={draw=evoblue,fill=evoblue!7,line width=0.714pt,rounded corners=2.143pt,inner sep=2.143pt,align=center,text width=120pt},
  lrn/.style={draw=evopink,fill=evopink!7,line width=0.714pt,rounded corners=2.143pt,inner sep=2.143pt,align=center,text width=120pt},
  sp/.style={line width=1.429pt,color=black!28,{|-|}},
  hs/.style={line width=1.429pt,color=evopink,{|-|}},
  ttl/.style={font=\bfseries,anchor=west,inner sep=0pt},
  tag/.style={font=\fontsize{7.14}{8.2}\selectfont,text=black!50,anchor=east,inner sep=0pt},
  note/.style={text=black!62,align=center},
  pan/.style={draw=black!22,line width=0.571pt,rounded corners=3.571pt}]
\pgfmathsetmacro{\G}{5}
\pgfmathsetmacro{\W}{(\linewidth/1pt-3*\G)/4}
\pgfmathsetmacro{\H}{114}
\pgfmathsetmacro{\LW}{\linewidth/1pt}%
\useasboundingbox (0,2) rectangle ({\LW},{-\H-0.2});
\foreach \k/\t/\e in {0/{\ding{172}\, Score each action}/{Eq.~(\ref*{eq:logratio})},
                      1/{\ding{173}\, Anchor each state}/{Eq.~(\ref*{eq:anchor})},
                      2/{\ding{174}\, Balance every span}/{Eq.~(\ref*{eq:subtb})},
                      3/{\ding{175}\, Sum into credit}/{Prop.~2}}{
  \draw[pan] ({\k*(\W+\G)+0.2},-0.2) rectangle ({\k*(\W+\G)+\W-0.2},{-\H});
  \node[ttl] at ({\k*(\W+\G)+4},-8) {\t};
  \node[tag] at ({\k*(\W+\G)+\W-4},-8) {\e};}
\foreach \k in {1,2,3}{\fill[black!45] ({\k*(\W+\G)-\G+0.6},{-\H/2+3.2})--++(3.8,-3.2)--++(-3.8,-3.2)--cycle;}
\pgfmathsetmacro{\cx}{\W/2}
\begin{scope}[shift={(\cx,-29)}]
\foreach \i in {0,1,2,3} \node[st] (a\i) at ({(\i-1.5)*26.5},0) {$s_{\i}$};
\draw[act] (a0)--(a1); \draw[act,evopink] (a1)--(a2); \draw[act] (a2)--(a3);
\foreach \i/\c in {0/black,1/evopink,2/black}{
  \node[lab,text=\c] at ({(\i-1)*26.5},9.5) {$a_{\i}$};
  \node[lab,text=\c] at ({(\i-1)*26.5},-10.5) {$\Delta\ell_{\i}$};}
\node at (0,-29) {$\Delta\ell_t=\ell_\theta(a_t)-\ell_\rho(a_t)$};
\node[note] at (0,-50) {$\ell$: log-probability summed over\\the executed tool-call tokens};
\node[note] at (0,-71) {$\rho$: same backbone, frozen;\\both scored at the same $s_t$};
\end{scope}
\pgfmathsetmacro{\cx}{(\W+\G)+\W/2}
\begin{scope}[shift={(\cx,-29)}]
\foreach \i in {0,1,2,3} \node[st] (b\i) at ({(\i-1.5)*24},0) {$s_{\i}$};
\foreach \i/\j in {0/1,1/2,2/3} \draw[act] (b\i)--(b\j);
\foreach \i in {0,1,2} \node[lab,anchor=north] at (b\i.south) {$\tilde u_{\i}$};
\node[lab,anchor=north] at (b3.south) {$\log R_\beta$};
\node[meas,anchor=north] (m) at (0,-20) {measured, no gradient\\[-0.714pt]$\operatorname{sg}\bigl[\mathrm{clip}_{[0,\beta]}\bigl(u_{q,x}+c(g)\bigr)\bigr]$};
\node[font=\fontsize{12.86}{14}\selectfont] at (0,-42.5) {$+$};
\node[lrn,anchor=north] (l) at (0,-46.5) {learned residual\\[-0.714pt]$b_\psi\bigl(h_\rho(s),f(s)\bigr)$};
\node[note] at (0,-75.5) {$u_{q,x}$: task reference anchor\\$c(g)$: structure correction};
\end{scope}
\pgfmathsetmacro{\cx}{2*(\W+\G)+\W/2}
\begin{scope}[shift={(\cx,-29)}]
\foreach \i in {0,1,2,3} \node[st] (c\i) at ({(\i-1.5)*24},0) {$s_{\i}$};
\draw[act] (c0)--(c1); \draw[act,evopink] (c1)--(c2); \draw[act] (c2)--(c3);
\def\xa{-36}\def\xb{-12}\def\xc{12}\def\xd{36}
\draw[sp] (\xa+1,-15)--node[lab,above=0.429pt,text=black!45]{$\delta_{0:1}$}(\xb-1,-15);
\draw[hs] (\xb+1,-15)--node[lab,above=0.429pt]{$\delta_{1:2}$}(\xc-1,-15);
\draw[sp] (\xc+1,-15)--node[lab,above=0.429pt,text=black!45]{$\delta_{2:3}$}(\xd-1,-15);
\draw[hs] (\xa+1,-27)--node[lab,above=0.429pt]{$\delta_{0:2}$}(\xc-1,-27);
\draw[hs] (\xb+1,-39)--node[lab,above=0.429pt]{$\delta_{1:3}$}(\xd-1,-39);
\draw[hs] (\xa+1,-51)--node[lab,above=0.429pt]{$\delta_{0:3}$}(\xd-1,-51);
\node at (0,-63) {$\delta_{i:j}=\tilde u_i+\sum_{t=i}^{j-1}\Delta\ell_t-\tilde u_j$};
\node at (0,-77) {$\mathcal L_x=\tfrac{1}{6}\sum_{i<j}\delta_{i:j}^2$};
\end{scope}
\pgfmathsetmacro{\cx}{3*(\W+\G)+\W/2}
\begin{scope}[shift={(\cx,-29)}]
\pgfmathsetmacro{\cw}{10.5}
\pgfmathsetmacro{\xz}{-\cw*2.5+5}
\foreach \j/\l in {0/{0:1},1/{1:2},2/{2:3},3/{0:2},4/{1:3},5/{0:3}}
  \node[font=\fontsize{7.14}{8.2}\selectfont,rotate=50,anchor=south west,inner sep=0pt] at ({\xz+\j*\cw-2.5},-3) {$\delta_{\l}$};
\foreach \r/\nm in {0/a_0,1/a_1,2/a_2}{
  \node[anchor=east,inner sep=1.429pt] at ({\xz-5},{-9-\r*10.5}) {$\nm$};
  \foreach \j in {0,...,5} \draw[black!25,line width=0.571pt] ({\xz+\j*\cw-4.2},{-9-\r*10.5-4.2}) rectangle ++(8.4,8.4);}
\foreach \j in {0,3,5} \fill[black!45] ({\xz+\j*\cw-4.2},{-9-4.2}) rectangle ++(8.4,8.4);
\foreach \j in {1,3,4,5} \fill[evopink] ({\xz+\j*\cw-4.2},{-19.5-4.2}) rectangle ++(8.4,8.4);
\foreach \j in {2,4,5} \fill[black!45] ({\xz+\j*\cw-4.2},{-30-4.2}) rectangle ++(8.4,8.4);
\draw[evoblue,line width=1.143pt,rounded corners=1.429pt] ({\xz+5*\cw-5.6},-3.4) rectangle ({\xz+5*\cw+5.6},-35.6);
\node at (0,-48) {$\dfrac{\partial\mathcal L_x}{\partial\ell_\theta(a_t)}=\tfrac{2}{K_3}\sum_{i\le t<j}\delta_{i:j}$};
\node[note] at (0,-62) {rows use 3, 4, and 3 spans};
\node[align=center,text=evoblue] at (0,-74) {TB keeps only $\delta_{0:3}$ (boxed):\\one signal for all actions};
\end{scope}
\end{tikzpicture}}%
\caption{AnchorTB on a three-action history: \ding{172}~score each action against the frozen reference;
\ding{173}~anchor each state with a measured, stop-gradient anchor (blue) and a learned residual (pink);
\ding{174}~form all $K_3=6$ span residuals; \ding{175}~sum each action's spans into its coefficient.}
\label{fig:anchortb-steps}
\end{figure}

\noindent\textbf{Action Log-Ratio.}
An orchestrator action is a tool call of $K_t$ tokens sampled under a grammar
mask; its log-probability and its log-ratio to the reference are computed on
the executed tokens:
\begin{equation}
  \ell_\theta(a_t\mid s_t)=\sum_{j=1}^{K_t}\log P_\theta(a_{t,j}\mid s_t,a_{t,<j}),
  \qquad
  \Delta\ell_t=\ell_\theta(a_t\mid s_t)-\ell_\rho(a_t\mid s_t),
  \label{eq:logratio}
\end{equation}
where $\rho$ is the same backbone without the trainable adapter, scored under
the same context, tokens, and mask. Positions with a single legal token
contribute zero, so only real choices enter the ratio.

\noindent\textbf{Subtrajectory Residual and Loss.}
Every non-root state has a unique parent, so the backward policy is identically
one. Motivated by reference-relative subtrajectory balance
\citep{madan2023learning,venkatraman2024amortizing}, we regress each residual toward zero:
\begin{equation}
  \delta^{(x)}_{i:j}=\tilde u_x(s_i)+\sum_{t=i}^{j-1}\Delta\ell_t-\tilde u_x(s_j),
  \qquad
  \tilde u_x(s_T)=\log R_\beta(x),
  \qquad
  \mathcal{L}_x=\frac{1}{K_T}\sum_{0\le i<j\le T}\bigl(\delta^{(x)}_{i:j}\bigr)^2,
  \label{eq:subtb}
\end{equation}
where nonterminal flows use Eq.~\eqref{eq:anchor} with the same trajectory index
$x$ throughout (batch index suppressed), the terminal value overrides the learned
parameterization, and $K_T=T(T+1)/2$ counts the equally weighted subtrajectories
for $T\ge1$, so each action's coefficient sums the residuals of the spans containing it
(Figure~\ref{fig:anchortb-steps}, \ding{175}).

\noindent\textbf{Measured State Flows.}
The ideal relative log flow $u^*(s)=\log\mathbb E_\rho[R_\beta(x)\mid s,C]$
satisfies $u^*(s_0)=\log Z(C)$. A shared flow with zero
residuals recovers $u^*$ and $P^*$ (Proposition~\ref{prop:consist}); this
motivates reference anchors with a learned residual:
\begin{equation}
  \tilde u_x(s)=\operatorname{sg}\!\Bigl[\clip_{[0,\beta]}\bigl(u_{q,x}+c(g(s))\bigr)\Bigr]+b_\psi\bigl(h_\rho(s),f(s)\bigr),
  \qquad
  u_{q,x}=\log\bigl[1+(e^{\beta}-1)\,\hat p_q^{(-x)}\bigr],
  \label{eq:anchor}
\end{equation}
For nonterminal $s$, $\hat p_q^{(-x)}$ is a task-level partial leave-one-out
shrinkage estimate of the reference mean reward on the same task, computed from
reference rollout outcomes (Appendix~B); $c(g)$ is a correction read one batch
behind from a bucket of canonical prefix structures $g$, $h_\rho(s)$ the frozen
reference encoding, and $b_\psi$ a zero-initialized MLP. Under leave-one-out, one
prefix can differ across scored trajectories. The correction
averages reward ratios over prefix visits before the logarithm,
\begin{equation}
  y_x=\frac{R_\beta(x)}{e^{u_{q,x}}}-1,
  \qquad
  c(g)=\clip_{[-c_{\max},\,c_{\max}]}\log\Bigl(1+\frac{\sum_{(x,t)\in\mathcal V_g}y_x}{n_g+n_0}\Bigr),
\end{equation}
where $\mathcal V_g$ is the multiset of eligible reference prefix visits
to structure $g$, paired prefixes included (Appendix~C), $n_g=|\mathcal V_g|$, and $n_0\ge0$ with $n_g+n_0>0$.
The pseudo-count shrinks rare structures toward zero correction and $c_{\max}$ bounds it
(values in Appendix~C), and the logarithm's argument stays positive
(Lemma~\ref{lem:evs-pooling-domain}). The measured terms are held fixed, so only $\pi_\theta$
and $b_\psi$ are updated, with $b_\psi$ carrying the remainder.

\noindent\textbf{Proposition 2.} \emph{AnchorTB assigns each
orchestration action a regression coefficient from all subtrajectories
containing it, and combines a task-level reference anchor with prefix-dependent
corrections.}
\emph{Proof.} Appendix~\ref{sec:evs-C6}, via the coefficient of
Proposition~\ref{prop:evs-credit} and the corrections of Appendix~B.

\subsection{Validated Skill Admission}
Prior work admits a skill from a single trajectory, by a judge, by counting
successes, or by a trust hierarchy over verifiers
\citep{yang2026skillforge,chen2026skillcat,wu2026bayesian,shang2026self}, leaving
underspecified \emph{whether} a candidate helps and \emph{when} it may be promoted.
EvoSteer answers both inside the loop: a candidate is an ordinary entry of the action set in
Eq.~\eqref{eq:actions} that the orchestrator may bind while its evidence accumulates, and
that evidence is measured under the frozen reference of Section~4.2 (Figure~\ref{fig:arch}, bottom right),
so credit and admission share one reference.

\noindent\textbf{Skill Proposal and Action-Set Integration.}
In an author window, an independent author model distils at most one candidate
$\sigma$, a structured procedure, from scored, side-effect-free trajectories
(proposal rules and skill format in Appendix~C). A candidate occupies a
candidate slot of $\Slib{k}$, so the action mask of
Eq.~\eqref{eq:actions} already exposes it to \textsc{add\_agent} and
\textsc{bind\_skill} before it is validated; the full skill text is rendered
only when the bound agent executes.

\noindent\textbf{Paired Counterfactual Rollouts.}
For every batch task in a family with a candidate, two extra trajectories are
launched from the same record and the same forced first role, one binding
$\sigma$ and one not, and both are continued by the reference policy:
\begin{equation}
  x^{+}\sim\rho\bigl(\cdot\mid C^+,\ \textsc{add\_agent}(c_1,\ \sigma)\bigr),
  \qquad
  x^{-}\sim\rho\bigl(\cdot\mid C^-,\ \textsc{add\_agent}(c_1,\ \varnothing)\bigr),
  \label{eq:pair}
\end{equation}
where $c_1$ is the forced first role. The arm contexts $C^\pm$ share
the initial task, runtime, and value-head snapshot; $C^-$ masks $\sigma$ throughout. Pairs are scored by the terminal
reward of Eq.~\eqref{eq:target}; with
$b^{\pm}=\indic[r(x^{\pm})\ge\tfrac12]$, the comparisons accumulate as wins
$W=\sum\indic[b^{+}>b^{-}]$, losses $L=\sum\indic[b^{+}<b^{-}]$, and ties
$T_0=\sum\indic[b^{+}=b^{-}]$ over complete, side-effect-free pairs, which also
enter AnchorTB as off-policy paths, re-scored under each arm's own
context and legal mask (Appendix~C).

\noindent\textbf{Sequential Validation.}
At validation rounds between batches, each candidate with new wins or losses
gets one binomial sign-test tail per direction,
exact under the fixed-look sign model (Appendix~C):
\begin{equation}
  p_{+}=\Prob\bigl\{\mathrm{Bin}(D,\tfrac12)\ge W\bigr\},
  \qquad
  p_{-}=\Prob\bigl\{\mathrm{Bin}(D,\tfrac12)\ge L\bigr\},
  \qquad
  D=W+L,
  \label{eq:signtest}
\end{equation}
with $p_+=p_-=1$ when $D=0$. A candidate whose $p_{+}$ clears its
boundary is promoted to validated, one
whose $p_{-}$ clears it is retired, and otherwise it keeps accumulating. One
budget covers the whole run on two levels, and the same counts give the
reported effect:
\begin{equation}
  \alpha_{j,\ell}=\frac{\alpha}{j(j+1)\,\ell(\ell+1)},
  \qquad
  \hat\Delta=\frac{W-L}{W+L+T_0},
  \label{eq:spend}
\end{equation}
where $\hat\Delta$ is reported only when $W+L+T_0>0$ and measures the
paired threshold-pass-rate difference. Here $j$ globally indexes the
registered comparison and $\ell$ its observation; each
direction receives $\alpha_{j,\ell}/2$, and the budget is spent only when new
evidence arrives, so the total nominal level is at most $\alpha$ (Appendix~C). By a union
bound, family-wise error is at most $\alpha$ whenever each directional p-value is valid
for the selection and sampling rules in use, even for dependent tests.

\noindent\textbf{Proposition 3.} \emph{Validated Skill Admission lets a
candidate be tried before it is promoted, and promotes or retires it only when
paired evidence passes a sequential test under a shared alpha-spending budget.}
\emph{Proof.} Appendix~\ref{sec:evs-C6}, via the candidate-slot and status-transition rules,
the nominal budget accounting of Proposition~\ref{prop:budget}, and the conditional FWER result of
Proposition~\ref{prop:evs-conditional-fwer}.

\section{Experiments}

We evaluate EvoSteer through the following research questions (RQs).
\textbf{RQ1}: How does EvoSteer compare with the baselines in distribution?
\textbf{RQ2}: Does it generalize to held-out benchmarks?
\textbf{RQ3}: Does it transfer across executor backbones?
\textbf{RQ4}: What does each component contribute, and how do fixed orchestration
paradigms compare?
\textbf{RQ5}: How does AnchorTB compare with other objectives, and do the measured
anchor, validated admission, and in-run edits each behave as designed?

\begin{table}[t]
\centering
\large
\tabcolsep=1pt
\renewcommand{\arraystretch}{1.28}
\begin{adjustbox}{max width=\linewidth}
\begin{tabular}{@{}l@{\hspace{2pt}}l|ccccccc|cc|c@{}}
\toprule
\textbf{} & \textbf{} & \multicolumn{2}{c}{\textbf{Baseline}} & \textbf{SFT} & \textbf{GRPO$^\dagger$} & \textbf{AFlow} & \multicolumn{2}{c|}{\textbf{Agent+RL}} & \multicolumn{2}{c|}{\textbf{Skill evolution}} & \textbf{Ours} \\
\cmidrule(lr){3-4}\cmidrule(lr){5-5}\cmidrule(lr){6-6}\cmidrule(lr){7-7}\cmidrule(lr){8-9}\cmidrule(lr){10-11}\cmidrule(lr){12-12}
\textbf{Dataset} & \textbf{Metric} & \textbf{Qwen3.5} & \textbf{v4-flash} & \textbf{Qwen3.5} & \textbf{Qwen3.5} & \textbf{Qwen3.5} & \textbf{AgentFlow} & \textbf{FlowSteer} & \textbf{SkillFlow} & \textbf{SkillOpt} & \textbf{EvoSteer ($\Delta\uparrow$)} \\
\midrule
\multicolumn{12}{@{}l}{\textit{(a) In-Distribution (IID) benchmarks}} \\
\midrule
\textbf{HotpotQA} & Ans EM & 57.66\,{\large$_{\pm1.28}$} & 70.94\,{\large$_{\pm0.35}$} & 63.59\,{\large$_{\pm1.18}$} & 69.22\,{\large$_{\pm0.43}$} & 87.34\,{\large$_{\pm0.86}$} & 86.09\,{\large$_{\pm0.86}$} & 89.22\,{\large$_{\pm0.65}$} & 89.06\,{\large$_{\pm0.00}$} & 88.28\,{\large$_{\pm0.00}$} & \textbf{92.34}\,{\large$_{\pm0.35}$}\,\textcolor{evopink}{(+34.7)} \\
\textbf{} & Ans F1 & 73.75\,{\large$_{\pm1.24}$} & 82.33\,{\large$_{\pm0.47}$} & 77.70\,{\large$_{\pm1.23}$} & 81.64\,{\large$_{\pm0.50}$} & 88.97\,{\large$_{\pm0.85}$} & 88.42\,{\large$_{\pm0.88}$} & 90.30\,{\large$_{\pm0.66}$} & 92.00\,{\large$_{\pm0.11}$} & 90.34\,{\large$_{\pm0.44}$} & \textbf{92.84}\,{\large$_{\pm0.35}$}\,\textcolor{evopink}{(+19.1)} \\
\textbf{NQ-Open} & Ans EM & 23.59\,{\large$_{\pm0.86}$} & 38.59\,{\large$_{\pm1.31}$} & 25.78\,{\large$_{\pm0.00}$} & 24.22\,{\large$_{\pm1.10}$} & 76.56\,{\large$_{\pm0.78}$} & 81.41\,{\large$_{\pm1.28}$} & 79.53\,{\large$_{\pm1.28}$} & 82.66\,{\large$_{\pm0.86}$} & 82.34\,{\large$_{\pm1.31}$} & \textbf{85.00}\,{\large$_{\pm0.65}$}\,\textcolor{evopink}{(+61.4)} \\
\textbf{} & Ans F1 & 33.04\,{\large$_{\pm1.22}$} & 49.60\,{\large$_{\pm1.78}$} & 37.69\,{\large$_{\pm0.14}$} & 32.59\,{\large$_{\pm1.45}$} & 80.77\,{\large$_{\pm0.82}$} & 85.46\,{\large$_{\pm1.75}$} & 84.31\,{\large$_{\pm1.36}$} & 86.56\,{\large$_{\pm0.80}$} & 86.28\,{\large$_{\pm1.27}$} & \textbf{88.92}\,{\large$_{\pm0.69}$}\,\textcolor{evopink}{(+55.9)} \\
\textbf{MedQA} & Acc. & 71.41\,{\large$_{\pm0.43}$} & 84.53\,{\large$_{\pm0.35}$} & 73.13\,{\large$_{\pm0.43}$} & 77.66\,{\large$_{\pm0.43}$} & 89.22\,{\large$_{\pm0.65}$} & 88.44\,{\large$_{\pm0.86}$} & 90.31\,{\large$_{\pm1.18}$} & 91.41\,{\large$_{\pm0.78}$} & 91.72\,{\large$_{\pm0.70}$} & \textbf{93.13}\,{\large$_{\pm0.86}$}\,\textcolor{evopink}{(+21.7)} \\
\textbf{AIME 2026} & Acc. & 48.67\,{\large$_{\pm2.98}$} & 50.67\,{\large$_{\pm1.49}$} & 32.67\,{\large$_{\pm2.79}$} & 32.00\,{\large$_{\pm1.83}$} & 53.33\,{\large$_{\pm0.00}$} & 60.67\,{\large$_{\pm2.79}$} & 64.67\,{\large$_{\pm2.98}$} & 63.33\,{\large$_{\pm2.36}$} & 66.00\,{\large$_{\pm1.49}$} & \textbf{74.67}\,{\large$_{\pm1.83}$}\,\textcolor{evopink}{(+26.0)} \\
\textbf{MBPP+} & Pass@1 & 78.91\,{\large$_{\pm1.10}$} & 84.53\,{\large$_{\pm0.86}$} & 78.44\,{\large$_{\pm0.70}$} & 81.88\,{\large$_{\pm0.35}$} & 85.00\,{\large$_{\pm0.65}$} & 86.09\,{\large$_{\pm0.65}$} & 87.34\,{\large$_{\pm0.86}$} & 88.59\,{\large$_{\pm0.70}$} & 90.16\,{\large$_{\pm0.43}$} & \textbf{92.19}\,{\large$_{\pm1.10}$}\,\textcolor{evopink}{(+13.3)} \\
\textbf{ALFWorld} & SR & 46.88\,{\large$_{\pm1.24}$} & 60.47\,{\large$_{\pm0.43}$} & 40.31\,{\large$_{\pm0.70}$} & 50.47\,{\large$_{\pm0.70}$} & 73.28\,{\large$_{\pm0.65}$} & 76.09\,{\large$_{\pm0.43}$} & 81.56\,{\large$_{\pm0.70}$} & 83.28\,{\large$_{\pm0.70}$} & 85.78\,{\large$_{\pm0.86}$} & \textbf{89.69}\,{\large$_{\pm1.02}$}\,\textcolor{evopink}{(+42.8)} \\
\cmidrule(lr){1-12}
\multirow{3}{*}{\textbf{Avg.\,(IID)}} & Ans EM & 40.63\,{\large$_{\pm0.77}$} & 54.77\,{\large$_{\pm0.68}$} & 44.69\,{\large$_{\pm0.59}$} & 46.72\,{\large$_{\pm0.59}$} & 81.95\,{\large$_{\pm0.58}$} & 83.75\,{\large$_{\pm0.77}$} & 84.38\,{\large$_{\pm0.72}$} & 85.86\,{\large$_{\pm0.43}$} & 85.31\,{\large$_{\pm0.65}$} & \textbf{88.67}\,{\large$_{\pm0.37}$}\,\textcolor{evopink}{(+48.0)} \\
 & Ans F1 & 53.39\,{\large$_{\pm0.87}$} & 65.97\,{\large$_{\pm0.92}$} & 57.70\,{\large$_{\pm0.62}$} & 57.12\,{\large$_{\pm0.77}$} & 84.87\,{\large$_{\pm0.59}$} & 86.94\,{\large$_{\pm0.98}$} & 87.31\,{\large$_{\pm0.76}$} & 89.28\,{\large$_{\pm0.40}$} & 88.31\,{\large$_{\pm0.67}$} & \textbf{90.88}\,{\large$_{\pm0.39}$}\,\textcolor{evopink}{(+37.5)} \\
 & Acc. & 61.46\,{\large$_{\pm0.86}$} & 70.05\,{\large$_{\pm0.45}$} & 56.14\,{\large$_{\pm0.75}$} & 60.50\,{\large$_{\pm0.51}$} & 75.21\,{\large$_{\pm0.28}$} & 77.82\,{\large$_{\pm0.76}$} & 80.97\,{\large$_{\pm0.85}$} & 81.65\,{\large$_{\pm0.67}$} & 83.41\,{\large$_{\pm0.48}$} & \textbf{87.42}\,{\large$_{\pm0.63}$}\,\textcolor{evopink}{(+26.0)} \\
\midrule
\multicolumn{12}{@{}l}{\textit{(b) Out-of-Distribution (OOD) benchmarks}} \\
\midrule
\textbf{TriviaQA} & Ans EM & 44.22\,{\large$_{\pm0.43}$} & 70.31\,{\large$_{\pm0.00}$} & 46.09\,{\large$_{\pm0.55}$} & 47.34\,{\large$_{\pm1.18}$} & 90.94\,{\large$_{\pm1.05}$} & 89.06\,{\large$_{\pm0.00}$} & 91.25\,{\large$_{\pm0.35}$} & 90.16\,{\large$_{\pm1.05}$} & 92.34\,{\large$_{\pm1.28}$} & \textbf{95.63}\,{\large$_{\pm1.18}$}\,\textcolor{evopink}{(+51.4)} \\
\textbf{} & Ans F1 & 53.46\,{\large$_{\pm0.52}$} & 80.81\,{\large$_{\pm0.21}$} & 58.77\,{\large$_{\pm0.70}$} & 58.50\,{\large$_{\pm1.45}$} & 92.34\,{\large$_{\pm1.06}$} & 90.44\,{\large$_{\pm0.04}$} & 92.71\,{\large$_{\pm0.35}$} & 91.79\,{\large$_{\pm1.07}$} & 93.41\,{\large$_{\pm1.30}$} & \textbf{96.81}\,{\large$_{\pm1.23}$}\,\textcolor{evopink}{(+43.4)} \\
\textbf{MuSiQue} & Ans EM & 39.22\,{\large$_{\pm0.86}$} & 50.00\,{\large$_{\pm0.78}$} & 39.06\,{\large$_{\pm0.55}$} & 40.63\,{\large$_{\pm0.55}$} & 77.66\,{\large$_{\pm0.70}$} & 79.06\,{\large$_{\pm1.16}$} & 80.00\,{\large$_{\pm1.18}$} & 80.94\,{\large$_{\pm0.43}$} & 81.09\,{\large$_{\pm0.65}$} & \textbf{84.53}\,{\large$_{\pm0.86}$}\,\textcolor{evopink}{(+45.3)} \\
\textbf{} & Ans F1 & 49.12\,{\large$_{\pm1.17}$} & 58.73\,{\large$_{\pm0.72}$} & 50.79\,{\large$_{\pm0.95}$} & 50.09\,{\large$_{\pm0.68}$} & 82.70\,{\large$_{\pm0.74}$} & 84.36\,{\large$_{\pm1.24}$} & 85.74\,{\large$_{\pm1.47}$} & 84.82\,{\large$_{\pm0.45}$} & 86.09\,{\large$_{\pm0.69}$} & \textbf{87.84}\,{\large$_{\pm0.89}$}\,\textcolor{evopink}{(+38.7)} \\
\textbf{GPQA} & Acc. & 61.72\,{\large$_{\pm1.10}$} & 73.91\,{\large$_{\pm0.43}$} & 68.91\,{\large$_{\pm0.86}$} & 67.50\,{\large$_{\pm0.89}$} & 75.63\,{\large$_{\pm1.02}$} & 79.69\,{\large$_{\pm0.96}$} & 83.75\,{\large$_{\pm0.65}$} & 82.19\,{\large$_{\pm0.86}$} & 81.88\,{\large$_{\pm0.35}$} & \textbf{85.63}\,{\large$_{\pm0.70}$}\,\textcolor{evopink}{(+23.9)} \\
\textbf{MATH-Hard} & Acc. & 89.06\,{\large$_{\pm0.78}$} & 92.19\,{\large$_{\pm0.55}$} & 88.59\,{\large$_{\pm1.31}$} & 87.50\,{\large$_{\pm0.78}$} & 91.09\,{\large$_{\pm1.05}$} & 92.66\,{\large$_{\pm0.70}$} & 93.44\,{\large$_{\pm0.70}$} & 92.34\,{\large$_{\pm1.02}$} & 94.69\,{\large$_{\pm0.65}$} & \textbf{95.47}\,{\large$_{\pm0.86}$}\,\textcolor{evopink}{(+6.4)} \\
\textbf{SWE-Bench} & Resolved & 15.94\,{\large$_{\pm0.43}$} & 36.25\,{\large$_{\pm0.70}$} & 15.00\,{\large$_{\pm0.86}$} & 16.88\,{\large$_{\pm1.05}$} & 28.28\,{\large$_{\pm0.65}$} & 38.75\,{\large$_{\pm0.70}$} & 40.31\,{\large$_{\pm0.70}$} & 39.84\,{\large$_{\pm0.78}$} & 41.56\,{\large$_{\pm0.65}$} & \textbf{42.50}\,{\large$_{\pm0.70}$}\,\textcolor{evopink}{(+26.6)} \\
\textbf{WebShop} & SR & 32.66\,{\large$_{\pm0.86}$} & 65.47\,{\large$_{\pm0.86}$} & 32.81\,{\large$_{\pm0.78}$} & 36.25\,{\large$_{\pm0.70}$} & 59.53\,{\large$_{\pm0.86}$} & 76.09\,{\large$_{\pm0.43}$} & 78.28\,{\large$_{\pm0.86}$} & 82.81\,{\large$_{\pm1.10}$} & 85.94\,{\large$_{\pm0.78}$} & \textbf{90.63}\,{\large$_{\pm0.00}$}\,\textcolor{evopink}{(+58.0)} \\
\cmidrule(lr){1-12}
\multirow{3}{*}{\textbf{Avg.\,(OOD)}} & Ans EM & 41.72\,{\large$_{\pm0.48}$} & 60.16\,{\large$_{\pm0.39}$} & 42.58\,{\large$_{\pm0.39}$} & 43.98\,{\large$_{\pm0.65}$} & 84.30\,{\large$_{\pm0.63}$} & 84.06\,{\large$_{\pm0.58}$} & 85.63\,{\large$_{\pm0.62}$} & 85.55\,{\large$_{\pm0.57}$} & 86.72\,{\large$_{\pm0.72}$} & \textbf{90.08}\,{\large$_{\pm0.73}$}\,\textcolor{evopink}{(+48.4)} \\
 & Ans F1 & 51.29\,{\large$_{\pm0.64}$} & 69.77\,{\large$_{\pm0.38}$} & 54.78\,{\large$_{\pm0.59}$} & 54.29\,{\large$_{\pm0.80}$} & 87.52\,{\large$_{\pm0.65}$} & 87.40\,{\large$_{\pm0.62}$} & 89.22\,{\large$_{\pm0.76}$} & 88.31\,{\large$_{\pm0.58}$} & 89.75\,{\large$_{\pm0.74}$} & \textbf{92.32}\,{\large$_{\pm0.76}$}\,\textcolor{evopink}{(+41.0)} \\
 & Acc. & 49.84\,{\large$_{\pm0.41}$} & 66.95\,{\large$_{\pm0.33}$} & 51.33\,{\large$_{\pm0.49}$} & 52.03\,{\large$_{\pm0.43}$} & 63.63\,{\large$_{\pm0.45}$} & 71.80\,{\large$_{\pm0.36}$} & 73.95\,{\large$_{\pm0.37}$} & 74.30\,{\large$_{\pm0.47}$} & 76.02\,{\large$_{\pm0.31}$} & \textbf{78.55}\,{\large$_{\pm0.33}$}\,\textcolor{evopink}{(+28.7)} \\
\bottomrule
\end{tabular}
\end{adjustbox}
\vspace{5pt}
\caption{Main results (five-run mean $\pm$ std). All methods run on the Qwen3.5-9B
executor except v4-flash (DeepSeek-V4-Flash); GRPO$^\dagger$ trains the backbone;
$\Delta\uparrow$: gain over Qwen3.5-9B.}
\label{tab:main}
\end{table}
\subsection{Experimental Setup}

\noindent\textbf{Datasets.}
Six in-distribution (IID) benchmarks supply the training tasks and IID tests:
HotpotQA \citep{yang2018hotpotqa}, NQ-Open \citep{kwiatkowski2019natural}, MedQA
\citep{jin2021disease}, AIME 2026, MBPP+ \citep{liu2023your}, and ALFWorld
\citep{shridhar2020alfworld}. Six out-of-distribution (OOD) benchmarks are held out:
TriviaQA \citep{joshi2017triviaqa}, MuSiQue \citep{trivedi2022musique}, GPQA
\citep{rein2023gpqa}, MATH-Hard \citep{hendrycks2021measuring}, SWE-Bench Verified
\citep{jimenez2024swe}, and WebShop \citep{yao2022webshop}. Each test set has 128
items (AIME 2026: 30) disjoint from training.

\noindent\textbf{Baselines.}
We compare with direct prompting (Qwen3.5-9B, DeepSeek-V4-Flash), backbone training
(SFT, GRPO$^\dagger$ \citep{shao2024deepseekmath}), workflow search (AFlow
\citep{zhang2025aflow}), RL-trained orchestration (AgentFlow \citep{li2026flow},
FlowSteer \citep{zhang2026flowsteer}), and skill evolution (SkillFlow
\citep{zhang2026skillflow}, SkillOpt \citep{yang2026skillopt}). All use Qwen3.5-9B as
executor and trained model, the same training tasks, and their authors' best
configurations; DeepSeek-V4-Flash only writes EvoSteer's candidate skills, and other
authors move the averages by under 2 points.
EvoSteer is tested with learned parts frozen, and every test score is
a five-run mean (Appendix~\ref{sec:exp-details}).

\begin{table}[p]
\centering
\large
\tabcolsep=1pt
\renewcommand{\arraystretch}{1.16}
\setlength{\aboverulesep}{0pt}\setlength{\belowrulesep}{0pt}%
\resizebox{\linewidth}{!}{
\begin{tabular}{@{}l|*{6}{C{47pt}}|*{6}{C{47pt}}@{}}
\toprule
\textbf{Variant} & \multicolumn{6}{c|}{\textbf{IID}} & \multicolumn{6}{c}{\textbf{OOD}} \\
\cmidrule(lr){2-7}\cmidrule(lr){8-13}
 & {\normalsize \textbf{HotpotQA}} & {\normalsize \textbf{NQ-Open}} & {\normalsize \textbf{MedQA}} & {\normalsize \textbf{AIME}} & {\normalsize \textbf{MBPP+}} & {\normalsize \textbf{ALFWorld}} & {\normalsize \textbf{TriviaQA}} & {\normalsize \textbf{MuSiQue}} & {\normalsize \textbf{GPQA}} & {\normalsize \textbf{MATH}} & {\normalsize \textbf{SWE}} & {\normalsize \textbf{WebShop}} \\
 & {\normalsize Ans EM} & {\normalsize Ans EM} & {\normalsize Acc.} & {\normalsize Acc.} & {\normalsize Pass@1} & {\normalsize SR} & {\normalsize Ans EM} & {\normalsize Ans EM} & {\normalsize Acc.} & {\normalsize Acc.} & {\normalsize Resolved} & {\normalsize SR} \\
\midrule
Qwen3.5-9B (frozen) & 57.66 & 23.59 & 71.41 & 48.67 & 78.91 & 46.88 & 44.22 & 39.22 & 61.72 & 89.06 & 15.94 & 32.66 \\
EvoSteer arch., untrained & 82.34 & 73.28 & 82.97 & 53.33 & 87.66 & 73.59 & 90.47 & 71.88 & 75.31 & 90.00 & 29.53 & 69.84 \\
\midrule
\multicolumn{1}{@{}l|}{\normalsize\textit{Fixed orchestration paradigms}} & & & & & & & & & & & & \\[-2pt]
Single agent with tools & 75.94 & 67.19 & 80.78 & 48.67 & 85.63 & 70.16 & 86.25 & 61.41 & 70.31 & 89.84 & 27.03 & 66.09 \\
Fixed multi-agent template & 82.03 & 74.84 & 85.47 & 56.67 & 83.28 & 66.72 & 89.84 & 70.31 & 73.75 & 90.47 & 24.06 & 54.84 \\
Plan-then-execute & 80.47 & 72.50 & 81.88 & 51.33 & 84.84 & 65.63 & 90.78 & 68.59 & 71.56 & 89.53 & 25.94 & 57.81 \\
Searched workflow & 87.19 & 76.41 & 88.91 & 54.67 & 85.16 & 73.13 & 91.56 & 77.50 & 75.63 & 90.78 & 28.28 & 59.06 \\
\midrule
\multicolumn{1}{@{}l|}{\normalsize\textit{Online orchestration (Sec.~4.1)}} & & & & & & & & & & & & \\[-2pt]
$-$ Interleaved execution & 89.38 & 83.75 & 89.53 & 58.67 & 90.63 & 80.94 & 95.00 & 79.84 & 81.41 & 93.28 & 31.25 & 75.47 \\
$-$ Execution features & 88.28 & 83.75 & 90.31 & 66.67 & 91.41 & 81.88 & 93.91 & 81.25 & 82.34 & 93.28 & 37.34 & 82.81 \\
$-$ Learnable repair & 90.31 & 84.06 & 89.84 & 67.33 & 90.47 & 83.59 & 95.31 & 80.16 & 82.81 & 92.50 & 31.88 & 87.03 \\
$-$ Reference value head & 83.28 & 84.22 & 91.09 & 68.00 & 89.84 & 80.16 & 94.22 & 79.22 & 81.09 & 92.81 & 39.22 & 82.66 \\
\midrule
\multicolumn{1}{@{}l|}{\normalsize\textit{AnchorTB (Sec.~4.2)}} & & & & & & & & & & & & \\[-2pt]
$-$ Measured flows & 90.16 & 83.59 & 89.22 & 56.67 & 90.00 & 77.81 & 94.38 & 76.25 & 81.88 & 91.09 & 33.75 & 80.63 \\
$-$ Flow corrections & 85.94 & 82.66 & 89.84 & 64.00 & 90.31 & 82.50 & 95.00 & 83.75 & 82.81 & 92.19 & 38.13 & 85.16 \\
\midrule
\multicolumn{1}{@{}l|}{\normalsize\textit{Skill admission (Sec.~4.3)}} & & & & & & & & & & & & \\[-2pt]
$-$ Skill evolution & 88.28 & 80.94 & 90.16 & 68.67 & 84.53 & 82.81 & 94.38 & 80.63 & 83.28 & 92.81 & 38.59 & 85.16 \\
$-$ Sequential validation & 90.47 & 81.72 & 89.22 & 66.00 & 90.00 & 78.59 & 94.84 & 81.25 & 83.28 & 94.53 & 35.00 & 86.09 \\
\midrule
\textbf{EvoSteer (Full)} & \textbf{92.34} & \textbf{85.00} & \textbf{93.13} & \textbf{74.67} & \textbf{92.19} & \textbf{89.69} & \textbf{95.63} & \textbf{84.53} & \textbf{85.63} & \textbf{95.47} & \textbf{42.50} & \textbf{90.63} \\
\bottomrule
\end{tabular}
}
\vspace{5pt}
\caption{Component ablation and paradigm comparison (untrained: initial
$\pi_\theta$). Paradigms (same executor and budget) fix the team before execution: a
ReAct-style agent with all tools, a hand-designed template, a planner that writes the
graph once, and a workflow searched offline on training tasks. Ablations: $-$ interleaved
execution builds the graph first; $-$ execution features drops $f$ from the state of
$\pi_\theta$; $-$ learnable repair masks \textsc{rerun} and
\textsc{drop}; $-$ reference value head withholds $\hat v_k$ and $\Delta\hat v$; $-$ measured
flows learns $u_q$ as a free scalar; $-$ flow corrections turns off $c(g)$ and $b_\psi$;
$-$ skill evolution runs without any skills; $-$ sequential validation admits after $k=5$
successes.}
\label{tab:ablation}
\end{table}
\begin{figure}[p]
\centering
\begin{minipage}[b]{0.677\linewidth}
\centering
\includegraphics[width=\linewidth,trim=3 7 4 1,clip]{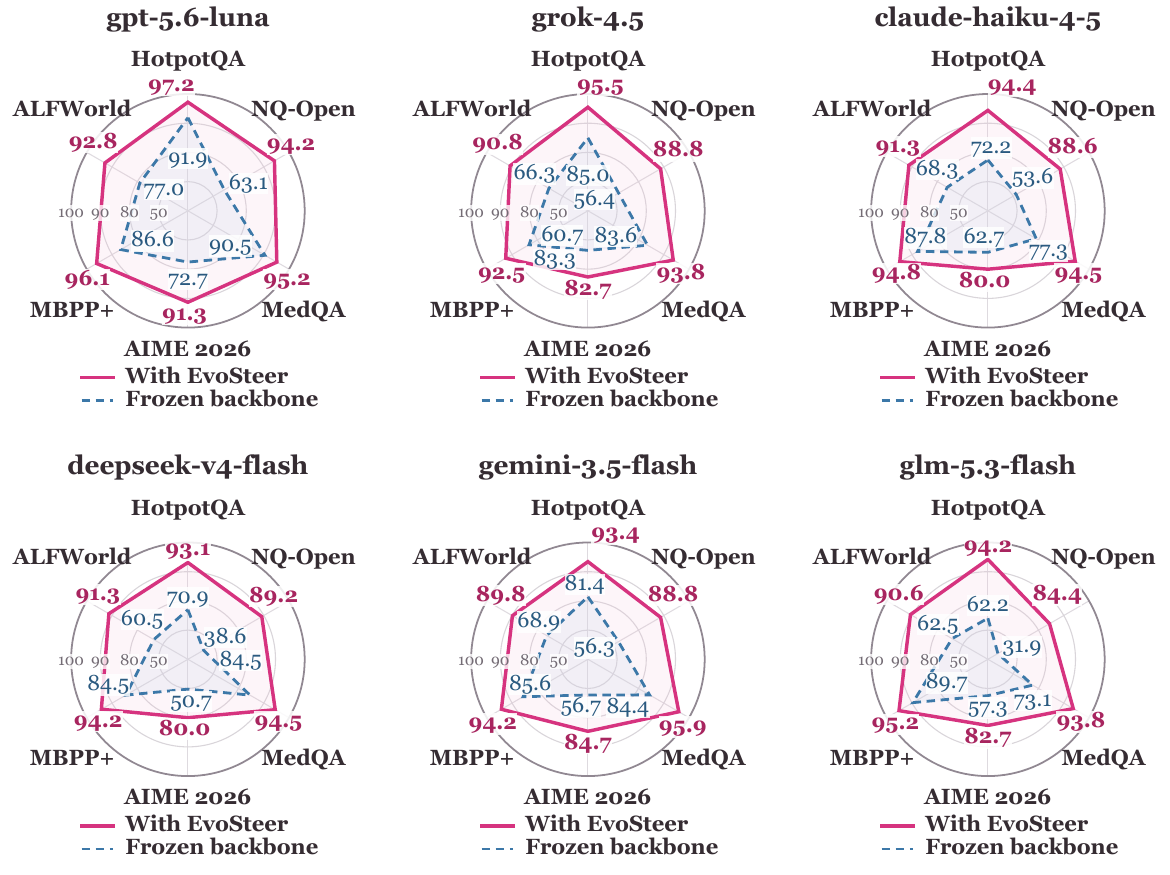}\\[1pt]
{\small (a) IID benchmark profile on each frozen backbone}
\end{minipage}\hspace{8pt}%
\begin{minipage}[b]{0.300\linewidth}
\centering
\includegraphics[width=\linewidth,trim=25 7 38 6,clip]{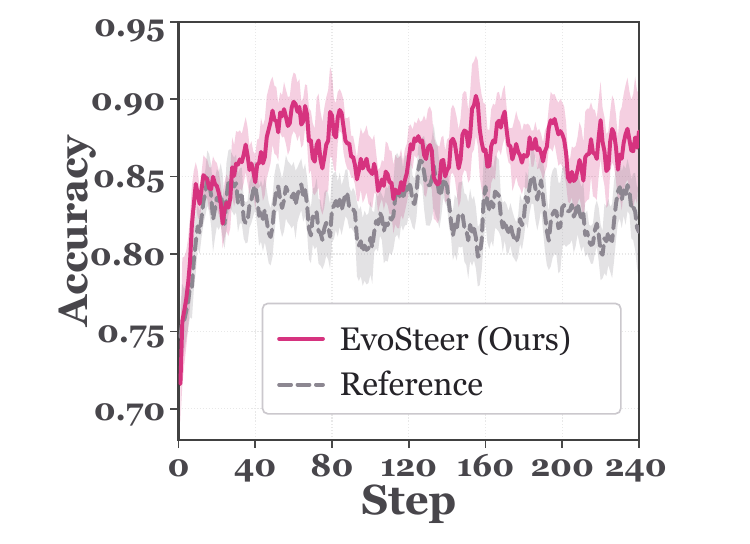}\\
\includegraphics[width=\linewidth,trim=26 7 38 8,clip]{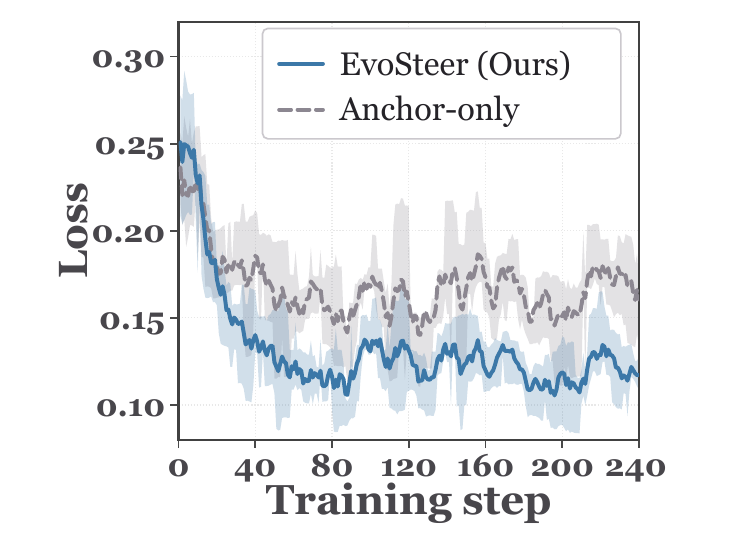}\\[1pt]
{\small (b) Training dynamics}
\end{minipage}\\[4pt]
\includegraphics[width=1.000\linewidth,trim=1 4 1 4,clip]{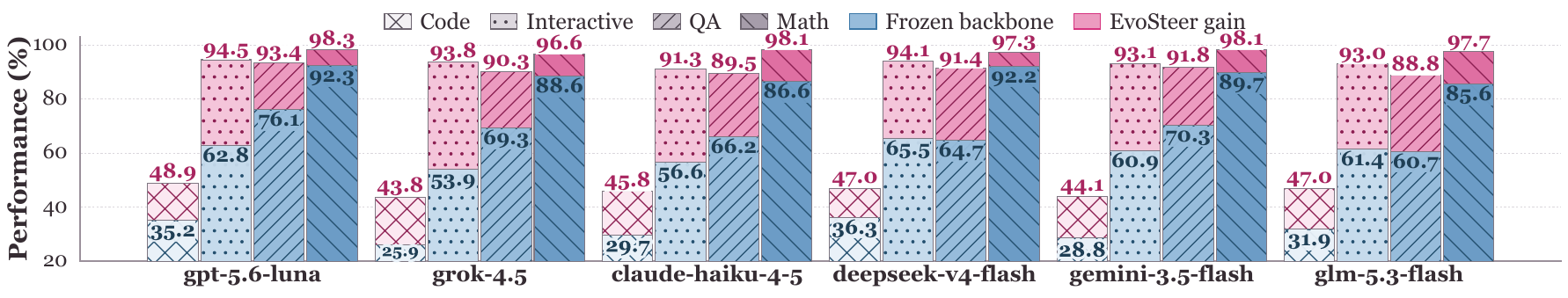}\\[1pt]
{\small (c) OOD scores aggregated by task domain}
\caption{Backbone transfer and training dynamics. \textbf{(a)} IID scores of six other
frozen executors (dashed) and with the same trained orchestrator (solid); the radius is
linear over 20--80 on the inner half and 80--100 on the outer half. \textbf{(b)} Accuracy of
$\pi_\theta$ and of the frozen $\rho$, and loss against the anchor-only level, over 240 steps.
\textbf{(c)} OOD scores per task domain.}
\label{fig:transfer}
\end{figure}

\begin{figure}[t]
\centering
\begin{minipage}[b]{0.690\linewidth}
\centering
\begin{tabular}{@{}c@{}}%
\resizebox*{\linewidth}{88pt}{%
\setlength{\tabcolsep}{4.43pt}\renewcommand{\arraystretch}{1.51}%
\begin{tabular}{@{}l|cccccc|ccc@{}}
\toprule
\textbf{Objective} & \textbf{Trivia} & \textbf{MuSiQue} & \textbf{GPQA} & \textbf{MATH} & \textbf{SWE} & \textbf{WebShop} & \textbf{GPU ms} & \textbf{GPU s} & \textbf{Tokens} \\
 & \textbf{EM} & \textbf{EM} & \textbf{Acc.} & \textbf{Acc.} & \textbf{Res.} & \textbf{SR} & \textbf{/ episode} & \textbf{/ step} & \textbf{/ problem} \\
\midrule
No training & 90.47 & 71.88 & 75.31 & 90.00 & 29.53 & 69.84 & -- & -- & -- \\
PPO & 93.59 & 79.22 & 80.63 & 92.81 & 35.94 & 82.97 & 2,002 & 384.4 & 17,834.3 \\
GRPO & 92.19 & 77.34 & 77.50 & 91.56 & 32.66 & 78.44 & 1,174 & 225.4 & 17,886.0 \\
Trajectory balance & 94.84 & 77.03 & 81.25 & 92.66 & 33.44 & 77.34 & 911 & 174.9 & 9,730.1 \\
Tempered TB & \textbf{95.63} & 81.09 & 82.50 & 93.75 & 40.00 & 87.81 & 1,737 & 333.5 & 19,373.8 \\
\textbf{AnchorTB (ours)} & \textbf{95.63} & \textbf{84.53} & \textbf{85.63} & \textbf{95.47} & \textbf{42.50} & \textbf{90.63} & 1,254 & 240.8 & 19,287.4 \\
\bottomrule
\end{tabular}}\\[1pt]
{\small (a) Objective comparison and training cost (RQ5)}\end{tabular}
\end{minipage}\hfill
\begin{minipage}[b]{\dimexpr\linewidth-0.690\linewidth-6pt\relax}
\centering
\begin{tabular}{@{}c@{}}\includegraphics[height=88.0pt,trim=3 13 12 3,clip]{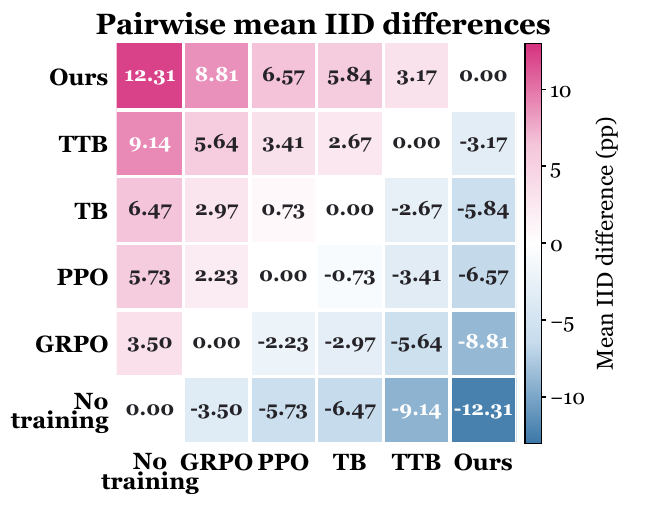}\\[1pt]{\small (b) Pairwise IID differences}\end{tabular}
\end{minipage}\\[6pt]
\begin{tabular}{@{}c@{}}\includegraphics[height=88.0pt,trim=9 60 15 17,clip]{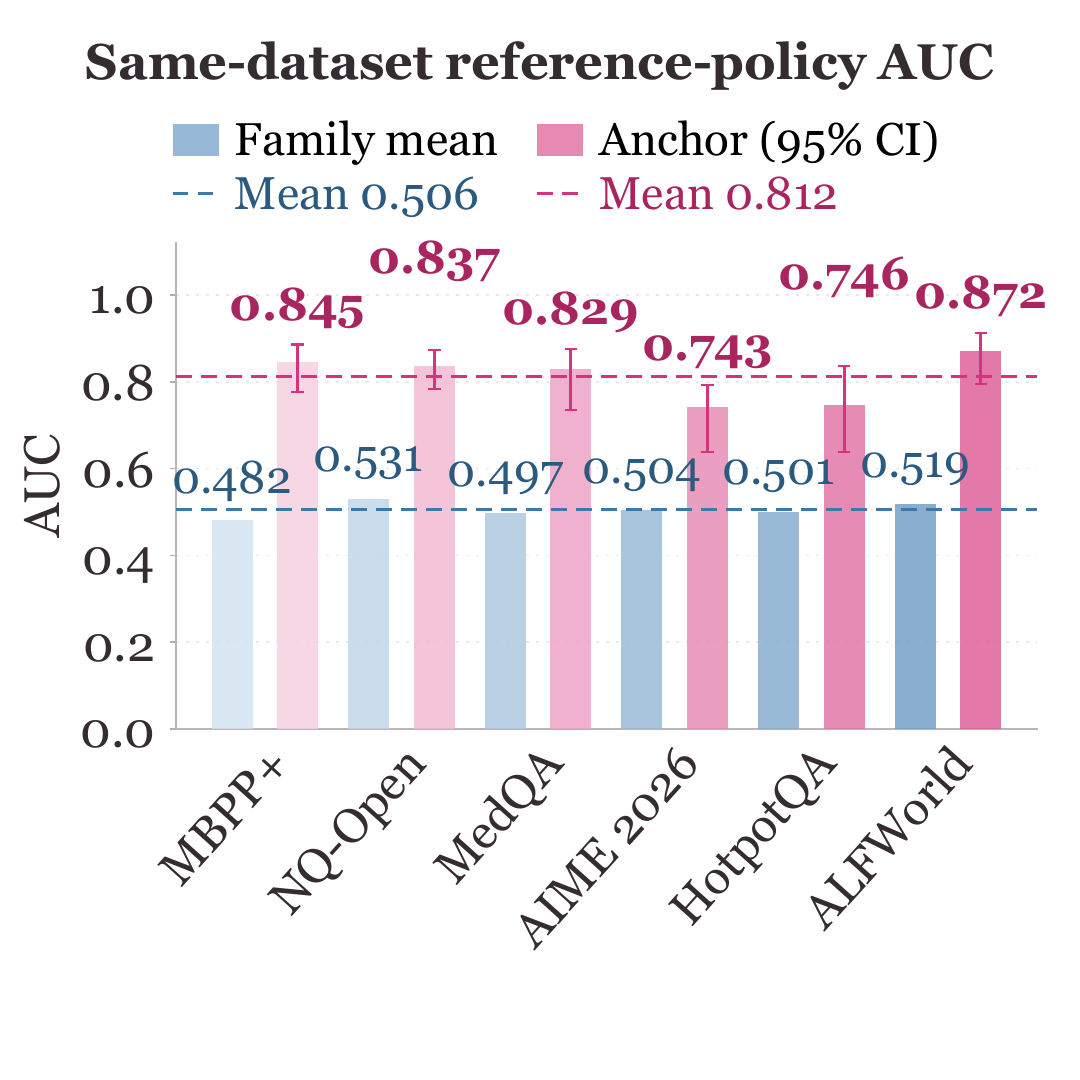}\\[1pt]{\small (c) Anchor AUC}\end{tabular}\hfill
\begin{tabular}{@{}c@{}}\includegraphics[height=88.0pt,trim=11 66 10 19,clip]{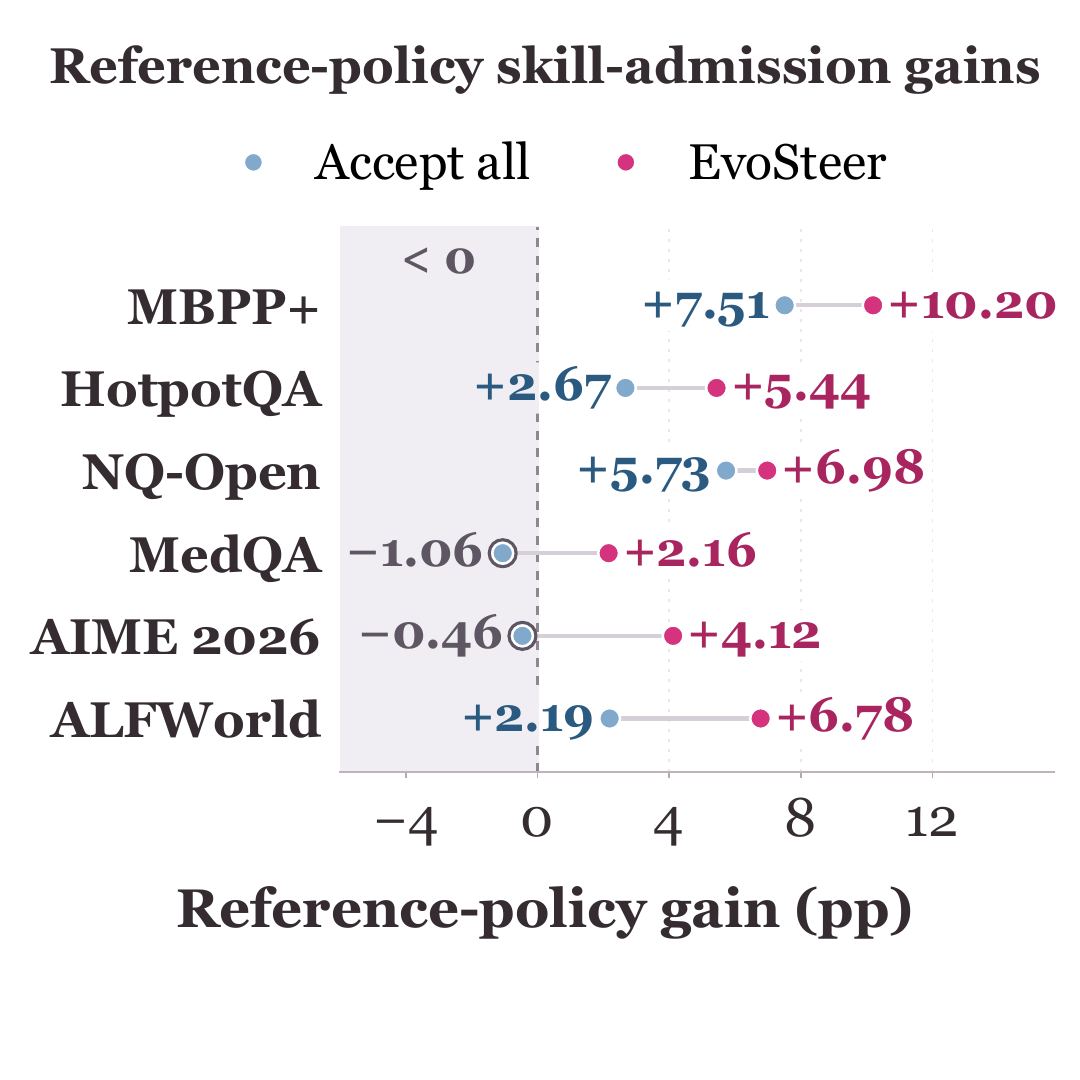}\\[1pt]{\small (d) Admission gain}\end{tabular}\hfill
\begin{tabular}{@{}c@{}}\includegraphics[height=88.0pt,trim=5.5 13.5 10.5 7,clip]{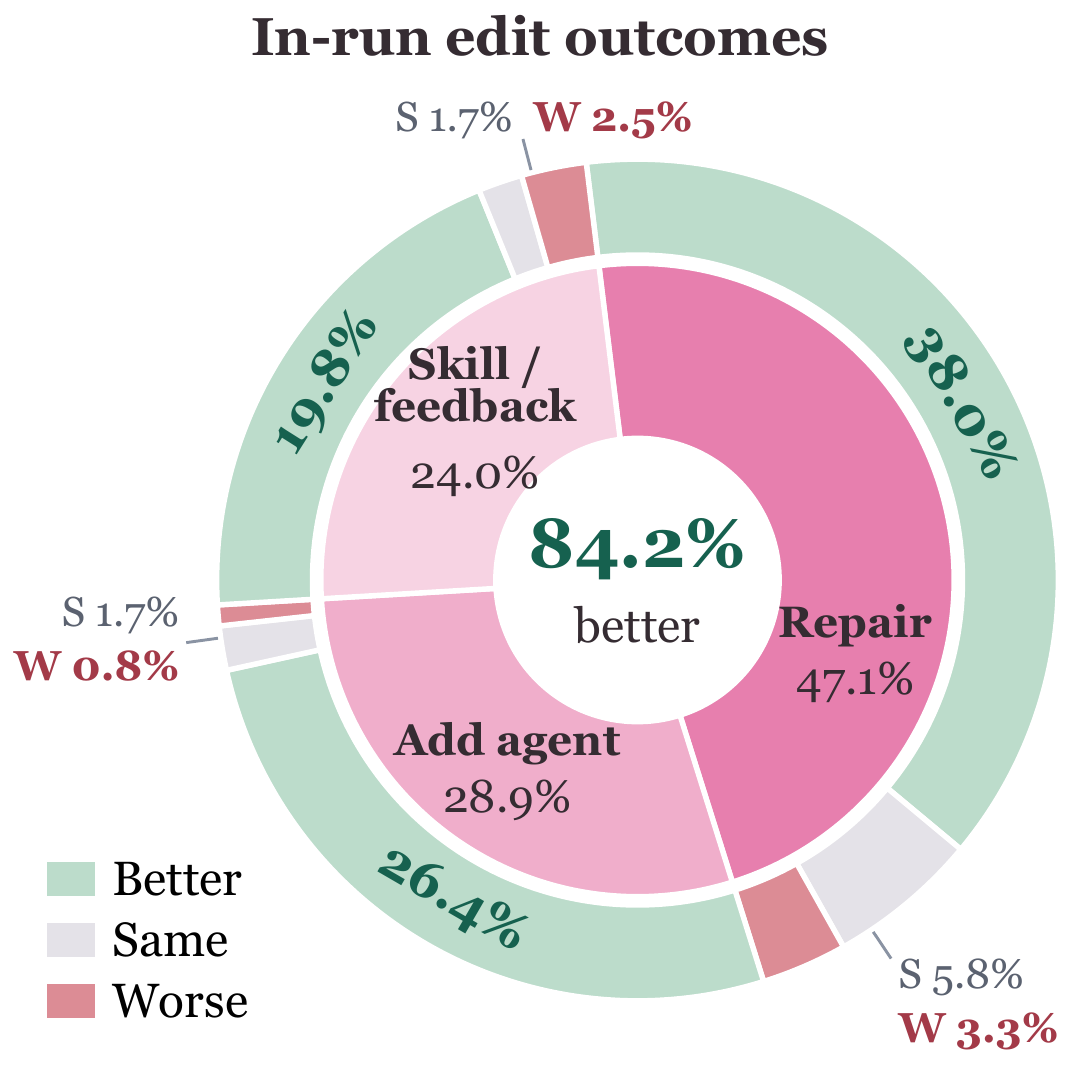}\\[1pt]{\small (e) Edit outcomes}\end{tabular}\hfill
\begin{tabular}{@{}c@{}}\includegraphics[height=88.0pt,trim=11 21 15 19,clip]{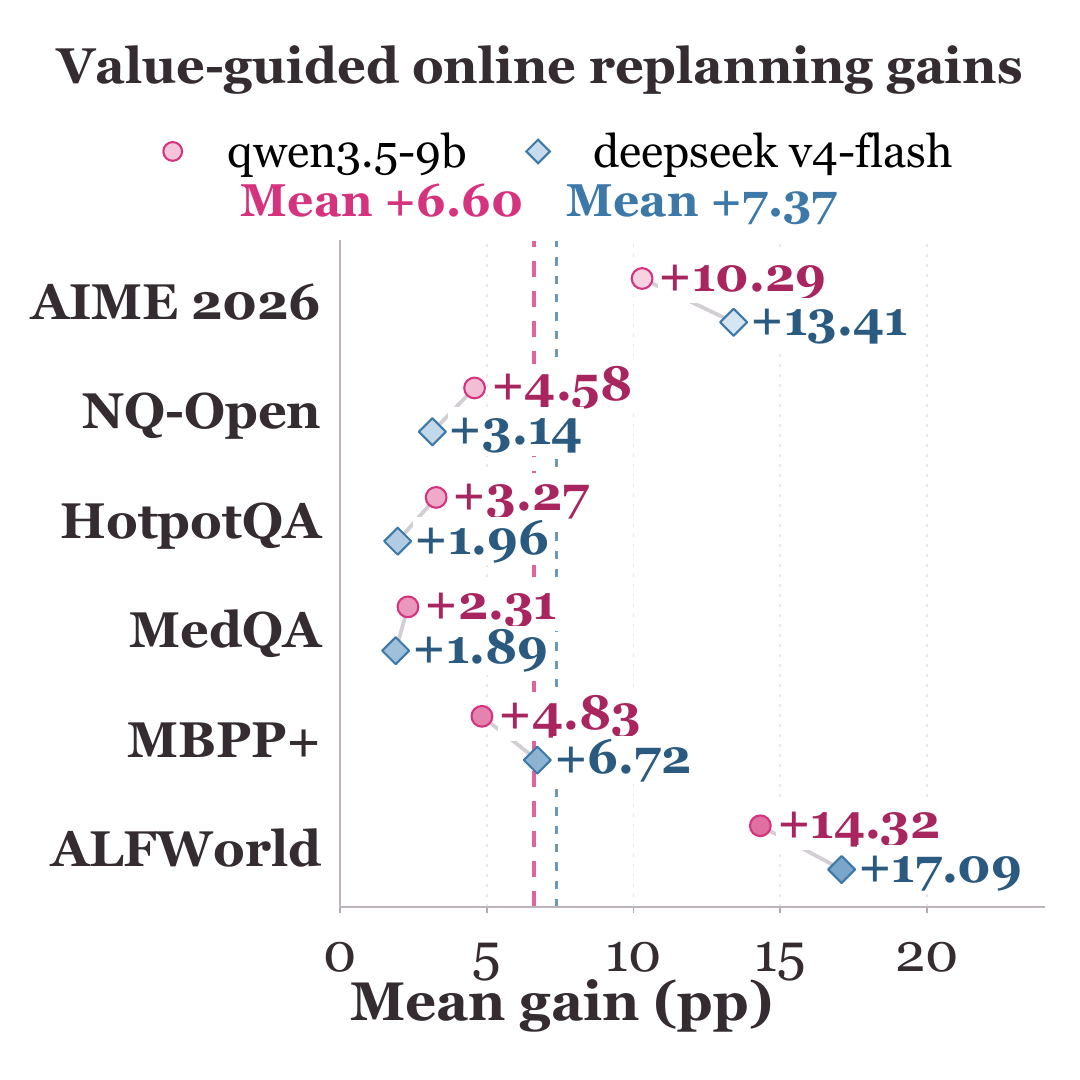}\\[1pt]{\small (f) Replanning gain}\end{tabular}
\caption{Objective comparison and mechanism analysis (definitions in
Appendix~\ref{sec:exp-details}). \textbf{(a)} OOD scores and cost per training
objective. \textbf{(b)} Row minus column, IID mean (pp); TTB: tempered TB.
\textbf{(c)} AUC for predicting reference-rollout correctness. \textbf{(d)} Paired skill gain,
admitting all or only validated candidates. \textbf{(e)} In-run edits by type and
outcome. \textbf{(f)} Gain of value-guided replanning in training.}
\label{fig:mechanism}
\end{figure}

\subsection{Main Results and Backbone Transfer (RQ1--RQ3)}

\noindent\textbf{Main results.}
EvoSteer is best on all twelve benchmarks (Table~\ref{tab:main}), ahead of the
strongest baseline by 2.81 EM and 4.01 accuracy points on IID averages and by 3.36 and
2.53 on OOD, with standard deviations at most 1.83. The lead is largest on AIME 2026
(8.67) and smallest on MATH-Hard (0.78), where the backbone is near its ceiling. Training the
backbone alone (SFT, GRPO$^\dagger$) adds at most 6.09 points to the IID EM average, so the gain comes from orchestration over the same frozen executor, and with a 9B executor
EvoSteer beats the larger DeepSeek-V4-Flash on every benchmark.

\noindent\textbf{Backbone transfer.}
The trained orchestrator transfers unchanged to six other executors and improves all 72
backbone--benchmark scores (Figure~\ref{fig:transfer}a,c), so one trained policy serves every executor. Weaker executors gain more, and the IID spread across backbones shrinks
from 17.5 to 4.3 points; gains are largest on interactive tasks and smallest on math,
where frozen scores are already near the ceiling.

\subsection{Component Analysis (RQ4, RQ5)}

\noindent\textbf{Online orchestration (Section~4.1).}
Each of its mechanisms contributes (Table~\ref{tab:ablation}): building the graph
before execution, dropping the execution features $f$, masking rerun and drop, and
withholding the value estimate cost 5.69, 4.12, 3.57, and 5.07 IID points. Keeping interleaving,
text feedback as in FlowSteer, and the value estimate but not $f$ costs 3.91 OOD
points, so $f$ carries information the text does not. Fixed paradigms help decomposable
questions but lag where results must be redone, and the best trails EvoSteer by
10.26 IID points; even untrained, the architecture leads all four on the OOD average.
In-run edits mostly help: 84.2\% raise the graded answer score and 6.6\% lower it
(Figure~\ref{fig:mechanism}e). Repairs, the most frequent edit, raise the score in 80.7\%
of cases, and value-guided replanning helps both executors on every IID dataset, most on
the interactive ALFWorld (Figure~\ref{fig:mechanism}f).

\noindent\textbf{AnchorTB (Section~4.2).}
Measured flows are the largest single contribution on IID: a free-scalar $u_q$ costs 6.59 IID
points, and removing $c(g)$ and $b_\psi$ costs 5.29. The anchor is informative on its own,
predicting reference outcomes before a rollout starts with AUC 0.812 against 0.506 for
the task-family mean (Figure~\ref{fig:mechanism}c), and the learned part lowers the loss
below the anchor-only level (Figure~\ref{fig:transfer}b). With only the loss changed,
AnchorTB beats tempered TB \citep{zhang2026skillflow}, TB \citep{malkin2022trajectory}, PPO
\citep{schulman2017proximal}, and GRPO by 3.17--8.81 IID points (Figure~\ref{fig:mechanism}a,b);
TB lags most on MuSiQue, SWE-Bench, and WebShop. AnchorTB needs 28\% less update compute than tempered TB, at a rollout
cost within 9\% of PPO and GRPO, so its gains come at comparable compute.

\noindent\textbf{Validated admission (Section~4.3).}
Removing skill evolution costs 5.27 IID and 3.26 OOD points, most on MBPP+ (7.66). A fixed
success count in place of the sequential test costs 5.17 IID points, nearly the same for any
count from 1 to 10 (Appendix~\ref{sec:exp-details}), and even trails the no-skill variant on
MedQA, AIME 2026, and ALFWorld. Accepting every candidate helps on average (+2.76) but
hurts MedQA and AIME 2026, whereas validated admission helps all six IID datasets (+5.95;
Figure~\ref{fig:mechanism}d).

\noindent\textbf{Architecture and training.}
The untrained architecture adds 21.01 IID and 24.04 OOD points to the frozen
executor, and AnchorTB training adds 12.31 and 11.22, so each supplies much of the gain.

\section{Conclusion}
EvoSteer unifies construction, repair, and skill growth in one learned loop: each action
executes as it is issued, AnchorTB credits it against a frozen reference, and paired
sequential tests decide which skills are kept. It beats all baselines on twelve
benchmarks and improves every executor backbone.

\section*{AI Use Statement}
In this work, generative AI tools were used to edit and polish the text for clarity and
readability, assist with \LaTeX{} typesetting and page layout, and adjust the placement and
sizing of existing figures and tables. They were also used for retrieval and discovery, to help
search for and identify related work; all bibliographic entries were taken from Google Scholar.
The authors manually reviewed and verified all AI-assisted text, figures, numerical values,
citations, and formatting against the underlying implementation and experimental records. The
authors take full responsibility for the final content of the paper and all artifacts produced
with the assistance of generative AI.

\section*{Ethics Statement}
This work follows the ICLR Code of Ethics. We aim to conduct and report our
research with scientific integrity, transparency, and reproducibility. Experimental results are
reported without fabrication, falsification, or intentional misrepresentation, and sufficient
implementation and evaluation details are provided to facilitate verification and reproduction.
We acknowledge prior work and use existing benchmarks, models, and associated resources in
accordance with their intended research purposes and applicable licenses. This study involves no
human participants or personal information.

\section*{Reproducibility Statement}
We support reproducibility by documenting the orchestration environment, action space, and value
estimate in Section~4.1, the AnchorTB objective in Section~4.2, and the validated skill admission
procedure in Section~4.3. Section~5.1 specifies the benchmarks, baselines, and evaluation setting.
The appendix provides the assumptions and complete proofs of all theoretical results
(Appendices~\mbox{\ref{app:evs-A}--\ref{app:evs-C}}), the training algorithm and all hyperparameters
(Appendix~\ref{app:evs-C}, Table~\ref{tab:config}), and the benchmark splits, metrics, baseline
settings, evaluation protocol, and compute (Appendix~\ref{sec:exp-details}). All test scores are
five-run means, and Table~\ref{tab:main} also reports their standard deviations.
\textbf{Code availability.} Our code, the exact training and test splits, and the full
configuration are publicly available at
\mbox{\href{https://github.com/beita6969/evosteer}{\nolinkurl{https://github.com/beita6969/evosteer}}}.

\setlength{\bibsep}{1.5pt plus 1pt minus 1pt}%
\bibliographystyle{iclr2027_conference}
\bibliography{refs}

\begin{thebibliography}{69}
\providecommand{\natexlab}[1]{#1}
\providecommand{\url}[1]{\texttt{#1}}
\expandafter\ifx\csname urlstyle\endcsname\relax
  \providecommand{\doi}[1]{doi: #1}\else
  \providecommand{\doi}{doi: \begingroup \urlstyle{rm}\Url}\fi

\bibitem[Bonagiri et~al.(2026)Bonagiri, Borkar, Anderias, Rafatirad, and
  Homayoun]{bonagiri2026causalflow}
Akash Bonagiri, Devang Borkar, Gerard~Janno Anderias, Setareh Rafatirad, and
  Houman Homayoun.
\newblock Causalflow: Causal attribution and counterfactual repair for llm
  agent failures.
\newblock \emph{arXiv preprint arXiv:2605.25338}, 2026.

\bibitem[Cemri et~al.(2026)Cemri, Pan, Yang, Agrawal, Chopra, Tiwari, Keutzer,
  Parameswaran, Klein, Ramchandran, et~al.]{cemri2026multi}
Mert Cemri, Melissa~Z Pan, Shuyi Yang, Lakshya~A Agrawal, Bhavya Chopra,
  Rishabh Tiwari, Kurt Keutzer, Aditya Parameswaran, Dan Klein, Kannan
  Ramchandran, et~al.
\newblock Why do multi-agent llm systems fail?
\newblock \emph{Advances in Neural Information Processing Systems}, 38, 2026.

\bibitem[Chen et~al.(2026{\natexlab{a}})Chen, Zhong, Liu, and
  Du]{chen2026skillcat}
Kunfeng Chen, Qihuang Zhong, Juhua Liu, and Bo~Du.
\newblock Skillcat: Contrastive assessment and topology-aware skill
  self-evolution for llm agents.
\newblock \emph{arXiv preprint arXiv:2606.13317}, 2026{\natexlab{a}}.

\bibitem[Chen et~al.(2026{\natexlab{b}})Chen, Liu, Wei, and
  Ding]{chen2026lemon}
Xudong Chen, Yixin Liu, Hua Wei, and Kaize Ding.
\newblock Lemon: Learning executable multi-agent orchestration via
  counterfactual reinforcement learning.
\newblock \emph{arXiv preprint arXiv:2605.14483}, 2026{\natexlab{b}}.

\bibitem[Chen et~al.(2026{\natexlab{c}})Chen, Sun, Wang, Wang, Zhang, Shen, Li,
  and Zhang]{chen2026exact}
Yanjun Chen, Yirong Sun, Hanlin Wang, Jinghan Wang, Xinming Zhang, Xiaoyu Shen,
  Wenjie Li, and Wei Zhang.
\newblock Exact is easier: Credit assignment for cooperative llm agents.
\newblock \emph{arXiv preprint arXiv:2603.06859}, 2026{\natexlab{c}}.

\bibitem[Dang et~al.(2026)Dang, Qian, Luo, Fan, Xie, Shi, Chen, Yang, Che,
  Tian, et~al.]{dang2026multi}
Yufan Dang, Chen Qian, Xueheng Luo, Jingru Fan, Zihao Xie, Ruijie Shi, Weize
  Chen, Cheng Yang, Xiaoyin Che, Ye~Tian, et~al.
\newblock Multi-agent collaboration via evolving orchestration.
\newblock \emph{Advances in neural information processing systems},
  38:\penalty0 165025--165059, 2026.

\bibitem[Deshmukh et~al.(2026)Deshmukh, Subramanian, Addanki, and
  Vlassis]{deshmukh2026cosac}
Shripad Deshmukh, Jayakumar Subramanian, Raghavendra Addanki, and Nikos
  Vlassis.
\newblock Cosac: Counterfactual credit assignment in sequential cooperative
  teams.
\newblock \emph{arXiv preprint arXiv:2604.17693}, 2026.

\bibitem[Fawkes \& Hartford(2026)Fawkes and Hartford]{fawkes2026f}
Jake Fawkes and Jason Hartford.
\newblock $ f $-trajectory balance: A loss family for tuning gflownets,
  generative models, and llms with off-and on-policy data.
\newblock \emph{arXiv preprint arXiv:2605.15417}, 2026.

\bibitem[Gao et~al.(2026)Gao, Chen, Wang, Guo, Pang, Liu, Shen, and
  Cheng]{gao2026skillaudit}
Haowen Gao, Haoran Chen, Can Wang, Shasha Guo, Liang Pang, Zhaoyang Liu, Huawei
  Shen, and Xueqi Cheng.
\newblock Skillaudit: Ground-truth-free skill evolution via paired trajectory
  auditing.
\newblock \emph{arXiv preprint arXiv:2606.14239}, 2026.

\bibitem[Hao et~al.(2026)Hao, Wang, Dong, Liu, Wang, Lin, Lin, Wang, Dong, and
  Chen]{hao2026evolve}
Zhezheng Hao, Tianfu Wang, Huanshuo Dong, Ziyan Liu, Hong Wang, Xiankun Lin,
  Qiang Lin, Can Wang, Hande Dong, and Jiawei Chen.
\newblock Evolve as a team: Collaborative self-evolution for llm-based
  multi-agent systems.
\newblock \emph{arXiv preprint arXiv:2605.29790}, 2026.

\bibitem[He \& Yang(2026)He and Yang]{he2026skillcommit}
Yu~He and Weikai Yang.
\newblock Skillcommit: Evolving agent skills through behaviorally validated
  scope expansion.
\newblock \emph{arXiv preprint arXiv:2608.15165}, 2026.

\bibitem[Hendrycks et~al.(2021)Hendrycks, Burns, Kadavath, Arora, Basart, Tang,
  Song, and Steinhardt]{hendrycks2021measuring}
Dan Hendrycks, Collin Burns, Saurav Kadavath, Akul Arora, Steven Basart, Eric
  Tang, Dawn Song, and Jacob Steinhardt.
\newblock Measuring mathematical problem solving with the math dataset.
\newblock \emph{arXiv preprint arXiv:2103.03874}, 2021.

\bibitem[Hong et~al.(2024)Hong, Zhuge, Chen, Zheng, Cheng, Wang, Zhang, Yau,
  Lin, Zhou, et~al.]{hong2024metagpt}
Sirui Hong, Mingchen Zhuge, Jonathan Chen, Xiawu Zheng, Yuheng Cheng, Jinlin
  Wang, Ceyao Zhang, Steven Yau, Zijuan Lin, Liyang Zhou, et~al.
\newblock Metagpt: Meta programming for a multi-agent collaborative framework.
\newblock In \emph{International Conference on Learning Representations},
  volume 2024, pp.\  23247--23275, 2024.

\bibitem[Hu et~al.(2026)Hu, Chu, Zhang, Wu, Jin, Zhao, Shao, Wang, and
  Wen]{hu2026skillbrew}
Wentao Hu, Zhendong Chu, Yiming Zhang, Junda Wu, Ming Jin, Xiangyu Zhao, Yilei
  Shao, Yanfeng Wang, and Qingsong Wen.
\newblock Skillbrew: Multi-objective curation of skill banks for llm agents.
\newblock \emph{arXiv preprint arXiv:2605.29440}, 2026.

\bibitem[Jiang et~al.(2026)Jiang, Li, Wang, Zhai, and Zhang]{jiang2026don}
Yanze Jiang, Mingxuan Li, Yuhao Wang, Shengfang Zhai, and Jiaheng Zhang.
\newblock Don't solve, just compare: Tiny advisors for runtime intervention in
  llm agents.
\newblock \emph{arXiv preprint arXiv:2608.21027}, 2026.

\bibitem[Jimenez et~al.(2024)Jimenez, Yang, Wettig, Yao, Pei, Press, and
  Narasimhan]{jimenez2024swe}
Carlos~E Jimenez, John Yang, Alexander Wettig, Shunyu Yao, Kexin Pei, Ofir
  Press, and Karthik Narasimhan.
\newblock Swe-bench: Can language models resolve real-world github issues?
\newblock In \emph{International Conference on Learning Representations},
  volume 2024, pp.\  54107--54157, 2024.

\bibitem[Jin et~al.(2021)Jin, Pan, Oufattole, Weng, Fang, and
  Szolovits]{jin2021disease}
Di~Jin, Eileen Pan, Nassim Oufattole, Wei-Hung Weng, Hanyi Fang, and Peter
  Szolovits.
\newblock What disease does this patient have? a large-scale open domain
  question answering dataset from medical exams.
\newblock \emph{Applied Sciences}, 11\penalty0 (14):\penalty0 6421, 2021.

\bibitem[Joshi et~al.(2017)Joshi, Choi, Weld, and
  Zettlemoyer]{joshi2017triviaqa}
Mandar Joshi, Eunsol Choi, Daniel~S Weld, and Luke Zettlemoyer.
\newblock Triviaqa: A large scale distantly supervised challenge dataset for
  reading comprehension.
\newblock In \emph{Proceedings of the 55th Annual Meeting of the Association
  for Computational Linguistics (Volume 1: Long Papers)}, pp.\  1601--1611,
  2017.

\bibitem[Kwiatkowski et~al.(2019)Kwiatkowski, Palomaki, Redfield, Collins,
  Parikh, Alberti, Epstein, Polosukhin, Devlin, Lee,
  et~al.]{kwiatkowski2019natural}
Tom Kwiatkowski, Jennimaria Palomaki, Olivia Redfield, Michael Collins, Ankur
  Parikh, Chris Alberti, Danielle Epstein, Illia Polosukhin, Jacob Devlin,
  Kenton Lee, et~al.
\newblock Natural questions: a benchmark for question answering research.
\newblock \emph{Transactions of the Association for Computational Linguistics},
  7:\penalty0 453--466, 2019.

\bibitem[Li et~al.(2026{\natexlab{a}})Li, Chen, Sun, and Wang]{li2026multi}
Haoran Li, Shulun Chen, Shaoyuan Sun, and Hanchen Wang.
\newblock Multi-agent coordination adaptation via structure-guided
  orchestration.
\newblock \emph{arXiv preprint arXiv:2605.25746}, 2026{\natexlab{a}}.

\bibitem[Li \& Ramakrishnan(2026)Li and Ramakrishnan]{li2026experience}
Sha Li and Naren Ramakrishnan.
\newblock Experience as a compass: Multi-agent rag with evolving orchestration
  and agent prompts.
\newblock \emph{arXiv preprint arXiv:2604.00901}, 2026.

\bibitem[Li et~al.(2026{\natexlab{b}})Li, Tian, Chen, Zhang, Liu, Ban, and
  Zhuang]{li2026counterfactual}
Zhongyi Li, Wan Tian, Jinju Chen, Huiming Zhang, Yang Liu, Yikun Ban, and
  Fuzhen Zhuang.
\newblock Counterfactual credit policy optimization for multi-agent
  collaboration.
\newblock \emph{arXiv preprint arXiv:2603.21563}, 2026{\natexlab{b}}.

\bibitem[Li et~al.(2026{\natexlab{c}})Li, Zhang, Han, Liu, Xie, Zhang, Choi,
  Zou, and Lu]{li2026flow}
Zhuofeng Li, Haoxiang Zhang, Seungju Han, Sheng Liu, Jianwen Xie, Yu~Zhang,
  Yejin Choi, James~Y Zou, and Pan Lu.
\newblock In-the-flow agentic system optimization for effective planning and
  tool use.
\newblock In \emph{International Conference on Learning Representations},
  volume 2026, pp.\  50524--50570, 2026{\natexlab{c}}.

\bibitem[Li et~al.(2026{\natexlab{d}})Li, Yu, Zang, Zhuang, Mo, and
  Gan]{li2026last}
Zongyue Li, Chengyue Yu, Lei Zang, Chenyi Zhuang, Linjian Mo, and Leilei Gan.
\newblock Last step matters: Early uncertainty cannot predict failure in
  long-horizon agents.
\newblock \emph{arXiv preprint arXiv:2608.29685}, 2026{\natexlab{d}}.

\bibitem[Liang et~al.(2026)Liang, Liu, Luo, Yang, Zhang, Min, Huang, Zhang, Qi,
  Xu, et~al.]{liang2026granularity}
Taoran Liang, Yang Liu, Shang Luo, Yingguang Yang, Rongrong Zhang, Yingzong
  Min, Yulin Huang, Jianshen Zhang, Yongzhi Qi, Kefu Xu, et~al.
\newblock Granularity-adaptive credit assignment for long-horizon llm agent
  reinforcement learning.
\newblock \emph{arXiv preprint arXiv:2609.12424}, 2026.

\bibitem[Liu et~al.(2023)Liu, Xia, Wang, and Zhang]{liu2023your}
Jiawei Liu, Chunqiu~Steven Xia, Yuyao Wang, and Lingming Zhang.
\newblock Is your code generated by chatgpt really correct? rigorous evaluation
  of large language models for code generation.
\newblock \emph{Advances in neural information processing systems},
  36:\penalty0 21558--21572, 2023.

\bibitem[Liu et~al.(2026)Liu, Xu, Stokes, Smolensky, Burger, and
  Gao]{liu2026gflowrl}
Xiaodong Liu, Michael Xu, Jack~W Stokes, Paul Smolensky, Doug Burger, and
  Jianfeng Gao.
\newblock Gflowrl: Scaling distribution-matching rl to large language models.
\newblock \emph{arXiv preprint arXiv:2607.13394}, 2026.

\bibitem[Lu \& Zhang(2026)Lu and Zhang]{lu2026autonomous}
Hanxiao Lu and Tianyi Zhang.
\newblock Autonomous repair for multi-agent systems via monte-carlo tree
  search.
\newblock \emph{arXiv preprint arXiv:2607.29055}, 2026.

\bibitem[Luan et~al.(2026)Luan, Zhang, Hu, Yang, Yu, and Chen]{luan2026repair}
Zhongwen Luan, Xiaoyu Zhang, Ming Hu, Yue Yang, Jiongchi Yu, and Xiaohong Chen.
\newblock Repair or resample? rethinking failure debugging in llm multi-agent
  systems.
\newblock \emph{arXiv preprint arXiv:2608.25920}, 2026.

\bibitem[Madan et~al.(2023)Madan, Rector-Brooks, Korablyov, Bengio, Jain, Nica,
  Bosc, Bengio, and Malkin]{madan2023learning}
Kanika Madan, Jarrid Rector-Brooks, Maksym Korablyov, Emmanuel Bengio, Moksh
  Jain, Andrei~Cristian Nica, Tom Bosc, Yoshua Bengio, and Nikolay Malkin.
\newblock Learning gflownets from partial episodes for improved convergence and
  stability.
\newblock In \emph{International Conference on Machine Learning}, pp.\
  23467--23483. PMLR, 2023.

\bibitem[Malkin et~al.(2022)Malkin, Jain, Bengio, Sun, and
  Bengio]{malkin2022trajectory}
Nikolay Malkin, Moksh Jain, Emmanuel Bengio, Chen Sun, and Yoshua Bengio.
\newblock Trajectory balance: Improved credit assignment in gflownets.
\newblock \emph{Advances in Neural Information Processing Systems},
  35:\penalty0 5955--5967, 2022.

\bibitem[Mishra et~al.(2026)Mishra, Chakraborty, and
  Kapusuzoglu]{mishra2026policy}
Amritansh Mishra, Supriyo Chakraborty, and Berkcan Kapusuzoglu.
\newblock On the policy gradient foundations of group relative policy
  optimization: Credit assignment, gradient sparsity, and rank collapse.
\newblock \emph{arXiv preprint arXiv:2606.29238}, 2026.

\bibitem[Pan et~al.(2026)Pan, Liu, Gao, Gao, Liu, Lin, Fu, Wang, Zhang, and
  Yu]{pan2026skillmas}
Shuai Pan, Yixiang Liu, Jiaye Gao, Te~Gao, Weiwen Liu, Jianghao Lin, Zhihui Fu,
  Jun Wang, Weinan Zhang, and Yong Yu.
\newblock Skillmas: Skill co-evolution with llm-based multi-agent system.
\newblock \emph{arXiv preprint arXiv:2605.09341}, 2026.

\bibitem[Rein et~al.(2023)Rein, Hou, Stickland, Petty, Pang, Dirani, Michael,
  and Bowman]{rein2023gpqa}
David Rein, Betty~Li Hou, Asa~Cooper Stickland, Jackson Petty, Richard~Yuanzhe
  Pang, Julien Dirani, Julian Michael, and Samuel~R Bowman.
\newblock Gpqa: A graduate-level google-proof q\&a benchmark.
\newblock \emph{arXiv preprint arXiv:2311.12022}, 2023.

\bibitem[Schulman et~al.(2017)Schulman, Wolski, Dhariwal, Radford, and
  Klimov]{schulman2017proximal}
John Schulman, Filip Wolski, Prafulla Dhariwal, Alec Radford, and Oleg Klimov.
\newblock Proximal policy optimization algorithms.
\newblock \emph{arXiv preprint arXiv:1707.06347}, 2017.

\bibitem[Sengupta(2026)]{sengupta2026self}
Biswa Sengupta.
\newblock Self-evolving agents with anytime-valid certificates.
\newblock \emph{arXiv preprint arXiv:2607.00871}, 2026.

\bibitem[Shah(2026)]{shah2026causal}
Jaineet Shah.
\newblock Causal agent replay: Counterfactual attribution for llm-agent
  failures.
\newblock \emph{arXiv preprint arXiv:2606.08275}, 2026.

\bibitem[Shang \& Yang(2026)Shang and Yang]{shang2026hypothesis}
Fangxin Shang and Yehui Yang.
\newblock Hypothesis-driven skill optimization for llm agents.
\newblock \emph{arXiv preprint arXiv:2606.22330}, 2026.

\bibitem[Shang et~al.(2026)Shang, Xu, Sun, Xia, Hu, Xu, and
  Zheng]{shang2026self}
Linfang Shang, Ming Xu, Yiding Sun, Tianle Xia, Lingxiang Hu, Lan Xu, and Ning
  Zheng.
\newblock When self-evolution backfires: Pre-commit gating against skill
  contamination in llm agents.
\newblock \emph{arXiv preprint arXiv:2608.05810}, 2026.

\bibitem[Shao et~al.(2024)Shao, Wang, Zhu, Xu, Song, Bi, Zhang, Zhang, Li, Wu,
  et~al.]{shao2024deepseekmath}
Zhihong Shao, Peiyi Wang, Qihao Zhu, Runxin Xu, Junxiao Song, Xiao Bi, Haowei
  Zhang, Mingchuan Zhang, YK~Li, Yang Wu, et~al.
\newblock Deepseekmath: Pushing the limits of mathematical reasoning in open
  language models.
\newblock \emph{arXiv preprint arXiv:2402.03300}, 2024.

\bibitem[Shawn(2026)]{shawn2026pace}
Zayx Shawn.
\newblock Pace: Anytime-valid acceptance tests for self-evolving agents.
\newblock \emph{arXiv preprint arXiv:2606.08106}, 2026.

\bibitem[Shen et~al.(2026)Shen, Zhang, Zhao, and Cheng]{shen2026dynamic}
Junhao Shen, Teng Zhang, Xiaoyan Zhao, and Hong Cheng.
\newblock Dynamic skill lifecycle management for agentic reinforcement
  learning.
\newblock \emph{arXiv preprint arXiv:2605.10923}, 2026.

\bibitem[Shridhar et~al.(2020)Shridhar, Yuan, C{\^o}t{\'e}, Bisk, Trischler,
  and Hausknecht]{shridhar2020alfworld}
Mohit Shridhar, Xingdi Yuan, Marc-Alexandre C{\^o}t{\'e}, Yonatan Bisk, Adam
  Trischler, and Matthew Hausknecht.
\newblock Alfworld: Aligning text and embodied environments for interactive
  learning.
\newblock \emph{arXiv preprint arXiv:2010.03768}, 2020.

\bibitem[Trivedi et~al.(2022)Trivedi, Balasubramanian, Khot, and
  Sabharwal]{trivedi2022musique}
Harsh Trivedi, Niranjan Balasubramanian, Tushar Khot, and Ashish Sabharwal.
\newblock ♫ musique: Multihop questions via single-hop question composition.
\newblock \emph{Transactions of the Association for Computational Linguistics},
  10:\penalty0 539--554, 2022.

\bibitem[Venkatraman et~al.(2024)Venkatraman, Jain, Scimeca, Kim, Sendera,
  Hasan, Rowe, Mittal, Lemos, Bengio, et~al.]{venkatraman2024amortizing}
Siddarth Venkatraman, Moksh Jain, Luca Scimeca, Minsu Kim, Marcin Sendera,
  Mohsin Hasan, Luke Rowe, Sarthak Mittal, Pablo Lemos, Emmanuel Bengio, et~al.
\newblock Amortizing intractable inference in diffusion models for vision,
  language, and control.
\newblock \emph{Advances in neural information processing systems},
  37:\penalty0 76080--76114, 2024.

\bibitem[Wang et~al.(2026{\natexlab{a}})Wang, Lyu, He, Yang, Zhong, Harel, and
  Lo]{wang2026fail}
Chenyu Wang, Yunbo Lyu, Junda He, Zhou Yang, Chenxing Zhong, Yaniv Harel, and
  David Lo.
\newblock Fail-fast, restart-smart: Early failure prediction and restart for
  swe agentic tasks.
\newblock \emph{arXiv preprint arXiv:2608.03222}, 2026{\natexlab{a}}.

\bibitem[Wang et~al.(2023)Wang, Xie, Jiang, Mandlekar, Xiao, Zhu, Fan, and
  Anandkumar]{wang2023voyager}
Guanzhi Wang, Yuqi Xie, Yunfan Jiang, Ajay Mandlekar, Chaowei Xiao, Yuke Zhu,
  Linxi Fan, and Anima Anandkumar.
\newblock Voyager: An open-ended embodied agent with large language models.
\newblock \emph{arXiv preprint arXiv:2305.16291}, 2023.

\bibitem[Wang et~al.(2026{\natexlab{b}})Wang, Zheng, and Xu]{wang2026rtmc}
Tao Wang, Suhang Zheng, and Xiaoxiao Xu.
\newblock Rtmc: Step-level credit assignment via rollout trees.
\newblock \emph{arXiv preprint arXiv:2604.11037}, 2026{\natexlab{b}}.

\bibitem[Wang et~al.(2026{\natexlab{c}})Wang, Zhou, Liang, Zhang, Liu, Zhou,
  and Yao]{wang2026not}
Yixuan Wang, Yiyang Zhou, Yiming Liang, Congyu Zhang, Fuxiao Liu, Jiawei Zhou,
  and Huaxiu Yao.
\newblock Not all skills help: Measuring and repairing agent knowledge.
\newblock \emph{arXiv preprint arXiv:2606.15390}, 2026{\natexlab{c}}.

\bibitem[Wu et~al.(2026)Wu, Yang, Liu, Lin, Zhang, Shi, Jiang, Xu, Li, and
  Guo]{wu2026bayesian}
Xiaojun Wu, Cehao Yang, Honghao Liu, Xueyuan Lin, Wenjie Zhang, Zhichao Shi,
  Xuhui Jiang, Chengjin Xu, Jia Li, and Jian Guo.
\newblock Bayesian-agent: Posterior-guided skill evolution for llm agent
  harnesses.
\newblock \emph{arXiv preprint arXiv:2606.08348}, 2026.

\bibitem[Xu et~al.(2026)Xu, Hu, Wang, Lin, Wang, Liu, and Feng]{xu2026tacomas}
Chen Xu, Yicheng Hu, Ruizi Wang, Xinyu Lin, Wenjie Wang, Dongrui Liu, and Fuli
  Feng.
\newblock Tacomas: Test-time co-evolution of topology and capability in
  llm-based multi-agent systems.
\newblock \emph{arXiv preprint arXiv:2605.09539}, 2026.

\bibitem[Yang et~al.(2026{\natexlab{a}})Yang, Ma, Huang, Wang, Li, Hu, Wang,
  and Chu]{yang2026skillforge}
Shidong Yang, Ziyu Ma, Tongwen Huang, Xucong Wang, Renda Li, Yiming Hu, Yong
  Wang, and Xiangxiang Chu.
\newblock Skillforge: Evolving verifiable skills for reinforcement learning
  agents.
\newblock \emph{arXiv preprint arXiv:2608.24747}, 2026{\natexlab{a}}.

\bibitem[Yang et~al.(2026{\natexlab{b}})Yang, Gong, Huang, Yang, Zhou, Huang,
  Li, Gao, Dai, Liu, et~al.]{yang2026skillopt}
Yifan Yang, Ziyang Gong, Weiquan Huang, Qihao Yang, Ziwei Zhou, Zisu Huang, Yan
  Li, Xuemei Gao, Qi~Dai, Bei Liu, et~al.
\newblock Skillopt: Executive strategy for self-evolving agent skills.
\newblock \emph{arXiv preprint arXiv:2605.23904}, 2026{\natexlab{b}}.

\bibitem[Yang et~al.(2018)Yang, Qi, Zhang, Bengio, Cohen, Salakhutdinov, and
  Manning]{yang2018hotpotqa}
Zhilin Yang, Peng Qi, Saizheng Zhang, Yoshua Bengio, William Cohen, Ruslan
  Salakhutdinov, and Christopher~D Manning.
\newblock Hotpotqa: A dataset for diverse, explainable multi-hop question
  answering.
\newblock In \emph{Proceedings of the 2018 conference on empirical methods in
  natural language processing}, pp.\  2369--2380, 2018.

\bibitem[Yao et~al.(2026)Yao, Liu, Fleming, Chen, and Wei]{yao2026maskills}
Huaiyuan Yao, Xiaoou Liu, Charles Fleming, Tianlong Chen, and Hua Wei.
\newblock Maskills: Continual skills optimization for multi-agent llm systems.
\newblock \emph{arXiv preprint arXiv:2609.02094}, 2026.

\bibitem[Yao et~al.(2022{\natexlab{a}})Yao, Chen, Yang, and
  Narasimhan]{yao2022webshop}
Shunyu Yao, Howard Chen, John Yang, and Karthik Narasimhan.
\newblock Webshop: Towards scalable real-world web interaction with grounded
  language agents.
\newblock \emph{Advances in Neural Information Processing Systems},
  35:\penalty0 20744--20757, 2022{\natexlab{a}}.

\bibitem[Yao et~al.(2022{\natexlab{b}})Yao, Zhao, Yu, Du, Shafran, Narasimhan,
  and Cao]{yao2022react}
Shunyu Yao, Jeffrey Zhao, Dian Yu, Nan Du, Izhak Shafran, Karthik Narasimhan,
  and Yuan Cao.
\newblock React: Synergizing reasoning and acting in language models.
\newblock \emph{arXiv preprint arXiv:2210.03629}, 2022{\natexlab{b}}.

\bibitem[Zhang et~al.(2026{\natexlab{a}})Zhang, Zhu, Shi, Liu, and
  Tang]{zhang2026agentforesight}
Boxuan Zhang, Jianing Zhu, Zeru Shi, Dongfang Liu, and Ruixiang Tang.
\newblock Agentforesight: Online auditing for early failure prediction in
  multi-agent systems.
\newblock \emph{arXiv preprint arXiv:2605.08715}, 2026{\natexlab{a}}.

\bibitem[Zhang(2026)]{zhang2026reinforcement}
Chenchen Zhang.
\newblock Reinforcement learning for llm-based multi-agent systems through
  orchestration traces.
\newblock \emph{arXiv preprint arXiv:2605.02801}, 2026.

\bibitem[Zhang et~al.(2024)Zhang, Yue, Sun, Wan, Yu, Fang, Wang, Chen, and
  Cheng]{zhang2024g}
Guibin Zhang, Yanwei Yue, Xiangguo Sun, Guancheng Wan, Miao Yu, Junfeng Fang,
  Kun Wang, Tianlong Chen, and Dawei Cheng.
\newblock G-designer: Architecting multi-agent communication topologies via
  graph neural networks.
\newblock \emph{arXiv preprint arXiv:2410.11782}, 2024.

\bibitem[Zhang et~al.(2025{\natexlab{a}})Zhang, Yue, Li, Yun, Wan, Wang, Cheng,
  Yu, and Chen]{zhang2025cut}
Guibin Zhang, Yanwei Yue, Zhixun Li, Sukwon Yun, Guancheng Wan, Kun Wang, Dawei
  Cheng, Jeffrey Yu, and Tianlong Chen.
\newblock Cut the crap: An economical communication pipeline for llm-based
  multi-agent systems.
\newblock In \emph{International Conference on Learning Representations},
  volume 2025, pp.\  75389--75428, 2025{\natexlab{a}}.

\bibitem[Zhang et~al.(2025{\natexlab{b}})Zhang, Xiang, Yu, Teng, Chen, Chen,
  Zhuge, Cheng, Hong, Wang, et~al.]{zhang2025aflow}
Jiayi Zhang, Jinyu Xiang, Zhaoyang Yu, Fengwei Teng, Xionghui Chen, Jiaqi Chen,
  Mingchen Zhuge, Xin Cheng, Sirui Hong, Jinlin Wang, et~al.
\newblock Aflow: Automating agentic workflow generation.
\newblock In \emph{International Conference on Learning Representations},
  volume 2025, pp.\  34040--34077, 2025{\natexlab{b}}.

\bibitem[Zhang et~al.(2026{\natexlab{b}})Zhang, Liu, Shen, Lin, Mao, Cambria,
  Tang, and Luo]{zhang2026flowsteer}
Mingda Zhang, Wenjin Liu, Tiesunlong Shen, Qika Lin, Rui Mao, Erik Cambria,
  Xiaoying Tang, and Haoran Luo.
\newblock Flowsteer: Towards agents designing agentic workflows via reinforced
  progressive canvas editing.
\newblock \emph{arXiv preprint arXiv:2602.01664}, 2026{\natexlab{b}}.

\bibitem[Zhang et~al.(2026{\natexlab{c}})Zhang, Shen, Luo, Liu, Xiao, Cambria,
  and Tang]{zhang2026skillflow}
Mingda Zhang, Tiesunlong Shen, Haoran Luo, Wenjin Liu, Zikai Xiao, Erik
  Cambria, and Xiaoying Tang.
\newblock Skillflow: Flow-driven recursive skill evolution for agentic
  orchestration.
\newblock \emph{arXiv preprint arXiv:2605.14089}, 2026{\natexlab{c}}.

\bibitem[Zhang \& Li(2026)Zhang and Li]{zhang2026consistencygate}
Yan Zhang and Shibo Li.
\newblock Consistencygate: Preventing memory contamination in llm agents via
  self-consistency admission control.
\newblock \emph{arXiv preprint arXiv:2607.22962}, 2026.

\bibitem[Zhang et~al.(2026{\natexlab{d}})Zhang, Xu, Dai, Shao, Wu, and
  Wang]{zhang2026evochamber}
Yaolun Zhang, Tianyi Xu, Shengyu Dai, Zhenwen Shao, Qingyun Wu, and Huazheng
  Wang.
\newblock Evochamber: Test-time co-evolution of multi-agent system at
  individual, team, and population scales.
\newblock \emph{arXiv preprint arXiv:2605.11136}, 2026{\natexlab{d}}.

\bibitem[Zhao et~al.(2026)Zhao, Zhang, Gu, Sun, Pei, Bansal, Rajmohan, and
  Ma]{zhao2026agenttether}
Chenyu Zhao, Shenglin Zhang, Wenwei Gu, Yongqian Sun, Dan Pei, Chetan Bansal,
  Saravan Rajmohan, and Minghua Ma.
\newblock Agenttether: Graph-guided diagnosis and runtime intervention for
  reliable llm agent operation.
\newblock \emph{arXiv preprint arXiv:2607.06273}, 2026.

\bibitem[Zhu et~al.(2026)Zhu, Chen, Zhang, Wu, and Dai]{zhu2026recognize}
Junze Zhu, Weihao Chen, Xuanwang Zhang, Zhen Wu, and Xinyu Dai.
\newblock Recognize your orchestrator: An entropy dynamics perspective for llm
  multi-agent systems.
\newblock \emph{arXiv preprint arXiv:2606.01351}, 2026.

\bibitem[Zhuge et~al.(2024)Zhuge, Wang, Kirsch, Faccio, Khizbullin, and
  Schmidhuber]{zhuge2024language}
Mingchen Zhuge, Wenyi Wang, Louis Kirsch, Francesco Faccio, Dmitrii Khizbullin,
  and J{\"u}rgen Schmidhuber.
\newblock Language agents as optimizable graphs.
\newblock \emph{arXiv preprint arXiv:2402.16823}, 2024.

\end{thebibliography}

\clearpage
\appendix
\setlength{\intextsep}{10pt plus 2pt minus 2pt}
\setlength{\textfloatsep}{12pt plus 2pt minus 2pt}
\setlength{\abovecaptionskip}{6pt}
\makeatletter
\def\thm@space@setup{\thm@preskip=5pt plus 2pt minus 1pt \thm@postskip=\thm@preskip}
\renewenvironment{proof}[1][\proofname]{\par\pushQED{\qed}\normalfont
  \topsep2\p@\@plus2\p@\relax\trivlist\item[\hskip\labelsep\itshape#1\@addpunct{.}]\ignorespaces}
  {\popQED\endtrivlist\@endpefalse}
\renewcommand\paragraph{\@startsection{paragraph}{4}{\z@}{0.9ex plus 0.3ex minus 0.2ex}{-1em}{\normalsize\bf}}
\makeatother
\relpenalty=10000
\binoppenalty=10000
\numberwithin{proposition}{section}
\numberwithin{definition}{section}
\numberwithin{assumption}{section}
\numberwithin{remark}{section}
\numberwithin{equation}{section}
\setcounter{proposition}{0}
\setcounter{definition}{0}
\setcounter{assumption}{0}
\setcounter{remark}{0}

\section{Anchored Trajectory Balance}
\label{app:evs-A}

This appendix treats three objects in turn: the reward-tilted history law, the
regression loss evaluated on recorded trajectories, and the statistical procedure
for skill validation. We first establish
structural and distributional properties conditional on a frozen environment
context, then analyze the implemented trajectory-indexed anchors and paired
tests. Algebraic identities hold on every recorded trajectory, and each
distributional or statistical result holds under the assumptions it states.

\subsection{Formal Setting, Histories, and Legal Actions}
\label{sec:evs-A1}

\begin{definition}[Context and complete-history representation]
\label{def:evs-context}
Fix a task and a rollout-batch context $C$. The context includes the task,
tools, role catalogue, executor configuration, skill-menu and value-head
snapshots, and rules that generate legal masks. A state $s_t$ retains the
complete ordered history, including the current communication graph, execution
records, features, and budget information. With $s_0$ the initial state,
\begin{equation}
 s_{t+1}=s_t\oplus(a_t,o_t^{\mathrm{exec}},f_{t+1}),\qquad
 a_t:s_t\longrightarrow s_{t+1},\qquad 0\le t<T,
 \label{eq:evs-history-update}
\end{equation}
where $f_{t+1}=f(s_t,a_t,o_t^{\mathrm{exec}})$ is the feature record produced by
action $a_t$. Thus $f(s_{t+1})=f_{t+1}$ denotes the latest feature record; the
initial state has its interface-specified initial features. The displayed
value estimate and its change are deterministic functions of this retained
history and the frozen head in $C$; they add no transition
randomness. A complete
history is $x=(s_0,a_0,s_1,\ldots,a_{T-1},s_T)$, identified with its
terminal state $s_T$, which retains the whole record. The final
\textsc{stop} transition is included in $T\ge1$. Write $x\succeq s$ when
$x$ extends prefix $s$, and let $|s|$ count appended action records.
\end{definition}

\begin{definition}[Legal actions and fixed executor]
\label{def:evs-kernel}
Let $\mathcal A_C(s)$ be the nonempty legal action set at a nonterminal
history. It is a state-dependent subset of the batch-frozen action
vocabulary. Let $\mathsf K_C(s'\mid s,a)$ be the normalized kernel of
execution, tool observations, and history updates after a legal action.
For a normalized history-dependent action policy $p$, its joint successor
kernel is $p(a\mid s,C)\mathsf K_C(s'\mid s,a)$. The executor is fixed:
this conditional kernel is the same for every $p$, and its outcomes may be
stochastic and history-dependent.
\end{definition}

\begin{lemma}[Strict growth and unique parentage]
\label{lem:evs-history-tree}
Under Definition~\ref{def:evs-context}, the history graph is a rooted tree:
each edge increases $|s|$ by one and every non-root state has a unique
parent. Its backward policy on parents is therefore identically one, so no backward model needs to be learned.
\end{lemma}
\begin{proof}
Equation~\eqref{eq:evs-history-update} appends exactly one action record,
so $|s_{t+1}|=|s_t|+1$. A directed cycle would strictly increase this
integer and return to its original value, a contradiction. Deleting the
last action, observation, and feature record from a non-root history
recovers its preceding prefix. Any parent must produce precisely this
last appended record and the same preceding ordered history, so no second
parent exists. A normalized distribution on this singleton parent set has
probability one. Returning to an earlier topology keeps the intervening records, so the history still grows.
\end{proof}

\begin{assumption}[Support, normalization, and termination]
\label{ass:evs-laws}
The serialized history space is finite or countable. At every
reference-reachable nonterminal prefix, $\rho(\cdot\mid s,C)$ is normalized
and positive on $\mathcal A_C(s)$. A policy compared to it in a log-ratio
has the same positive action support and uses the same $\mathsf K_C$.
From every such prefix, both continuation processes reach a terminal
history almost surely, and every terminal record has a reward
$r(x)\in[0,1]$, with $0\le\beta<\infty$.
\end{assumption}

\begin{proposition}[Normalized terminal and prefix laws]
\label{prop:evs-normalization}
Under Assumption~\ref{ass:evs-laws}, terminal histories form a probability-one
partition for $p\in\{\rho,\pi_\theta\}$. In particular,
\begin{equation}
 \sum_{x\in\mathcal X_C}P_p(x\mid C)=1,\qquad
 P_p(s\mid C)=\sum_{x\succeq s}P_p(x\mid C),\qquad P_p(s_0\mid C)=1,
 \label{eq:evs-prefix-mass}
\end{equation}
where $P_p$ denotes the history law induced by $p$ and the executor.
\end{proposition}
\begin{proof}
Normalized local kernels define the successive action and observation
probabilities. Different terminal histories describe disjoint events:
a terminal record cannot be a proper prefix of a later executed history.
Almost-sure termination makes their union a probability-one event. The
same reasoning, restricted to the event of reaching $s$, partitions that
event by its terminal descendants. Countable additivity proves both sums;
the root is reached with probability one.
\end{proof}

\begin{remark}[Legal actions and horizon]
\label{rem:evs-structural-scope}
A frozen snapshot fixes the interface and the rules that generate masks, and
the legal set $\mathcal A_C(s)$ at each state follows from them; for example, a
rerun or drop action requires an applicable agent and sufficient resources.
The communication graph may contain cycles. Finite-horizon results below
assume a bound $T\le H$.
\end{remark}

\subsection{Path Laws and Positive Affine Tilting}
\label{sec:evs-A2}

\begin{lemma}[Path factorization and kernel cancellation]
\label{lem:evs-path-ratio}
Under Assumption~\ref{ass:evs-laws},
\begin{equation}
 P_p(x\mid C)=\prod_{t=0}^{T-1}
 p(a_t\mid s_t,C)\mathsf K_C(s_{t+1}\mid s_t,a_t).
 \label{eq:evs-path-law}
\end{equation}
For a supported continuation of $s=s_m$, define
$\Delta\ell_t=\log\pi_\theta(a_t\mid s_t,C)-\log\rho(a_t\mid s_t,C)$.
Then
\begin{equation}
 \frac{P_\theta(x\mid s,C)}{P_\rho(x\mid s,C)}
 =\prod_{t=m}^{T-1}\frac{\pi_\theta(a_t\mid s_t,C)}{\rho(a_t\mid s_t,C)}
 =\exp\!\left(\sum_{t=m}^{T-1}\Delta\ell_t\right).
 \label{eq:evs-suffix-ratio}
\end{equation}
Here $P_\theta=P_{\pi_\theta}$ and $P_\rho(x\mid C)$ is the quantity
written $\rho(x\mid C)$ in Eq.~\eqref{eq:target}.
\end{lemma}
\begin{proof}
Apply the probability chain rule, conditioning each action on the history
before it and each observation on that history and action. This gives
Eq.~\eqref{eq:evs-path-law}; starting at $s_m$ gives the corresponding
suffix product. On the common support every executor factor in numerator
and denominator is equal and positive, so it cancels. The remaining
product consists of the action ratios, whose logarithm is the displayed
sum. This cancellation compares two policies on the same recorded
history.
\end{proof}

\begin{lemma}[Masked token likelihoods]
\label{lem:evs-token-mask}
Suppose legal actions have unique finite token serializations, including
any required action-ending symbol, and masked token generation terminates
almost surely. Use the same legal token sets for both policies, with
each policy normalized on each such set. The sum of executed token
log-probabilities is the log-probability of that action. Singleton legal
sets contribute zero to its log-ratio, since both policies give them probability one.
\end{lemma}
\begin{proof}
For action $a=(a_1,\ldots,a_K)$, the autoregressive chain rule gives
\[
 p(a\mid s,C)=\prod_{j=1}^K p(a_j\mid s,a_{<j},C),
\]
with any termination probability included in the serialization. Taking logarithms gives the
sum, and taking the difference for the two policies gives the action
log-ratio. If a legal token set is a singleton, both normalized
probabilities equal one. Its two log-probabilities are zero. Each
model's normalization constant on the common mask remains part of its
normalized probability.
\end{proof}

\begin{proposition}[Affine tilt, normalizer, and ideal reward gain]
\label{prop:evs-affine-tilt}
Let $\mu=\mathbb E_\rho[r(X)\mid C]$ and $a_\beta=e^\beta-1$.
The target in Eq.~\eqref{eq:target} is normalized, has the same terminal
support as $P_\rho$, and satisfies
\begin{align}
 &1\le R_\beta(x)\le e^\beta,\qquad
 Z(C)=1+a_\beta\mu\in[1,e^\beta],
 \label{eq:evs-affine-normalizer}\\
 &\mathbb E_{P^*}[r(X)\mid C]-\mu
   =\frac{a_\beta\operatorname{Var}_\rho(r(X)\mid C)}{1+a_\beta\mu}\ge0.
 \label{eq:evs-ideal-gain}
\end{align}
At $\beta=0$, $P^*=P_\rho$. For $\beta>0$, strict gain in
Eq.~\eqref{eq:evs-ideal-gain} occurs exactly when the reference reward has
positive variance, which holds whenever the reference both succeeds and fails.
\end{proposition}
\begin{proof}
The reward range follows from $r\in[0,1]$ and $a_\beta\ge0$.
Summing $P_\rho(x)(1+a_\beta r(x))$ gives $Z=1+a_\beta\mu$;
it is finite and positive, so dividing by $Z$ normalizes the
target and keeps its support. Moreover,
\[
 \mathbb E_{P^*}r
 =\frac{\mathbb E_\rho r+a_\beta\mathbb E_\rho r^2}{1+a_\beta\mu}.
\]
Subtracting $\mu$ gives
$a_\beta(\mathbb E_\rho r^2-\mu^2)/(1+a_\beta\mu)$, and the equality and
strictness statements follow from $a_0=0$ and $a_\beta>0$ for $\beta>0$.
\end{proof}

\begin{remark}[Meaning of the target]
\label{rem:evs-target-scope}
The tilt is affine in $r$ and equals $\exp(\beta r)$ on binary rewards,
where $\mu$ is the success probability; for continuous scores $\mu$ is a mean
reward. Equation~\eqref{eq:evs-ideal-gain} describes the training target $P^*$,
and skill validation uses the threshold-pass probability of Appendix~\ref{sec:evs-C1}.
\end{remark}

\subsection{Conditional Relative Flows and Uniqueness}
\label{sec:evs-A3}

\begin{definition}[Unnormalized target flow and relative flow]
\label{def:evs-relative-flow}
For $P_\rho(s\mid C)>0$, define
\begin{equation}
   F^*(s\mid C)=\sum_{x\succeq s}P_\rho(x\mid C)R_\beta(x),\qquad
 U^*(s\mid C)=\frac{F^*(s\mid C)}{P_\rho(s\mid C)},
 \label{eq:evs-relative-flow}
\end{equation}
and $u^*(s\mid C)=\log U^*(s\mid C)$. Thus $F^*$ is unnormalized. The normalized target prefix probability is
$P^*(s\mid C)=F^*(s\mid C)/Z(C)$, so
$P^*(s\mid C)/P_\rho(s\mid C)=U^*(s\mid C)/Z(C)$.
\end{definition}

\begin{lemma}[Prefix partition and conditional recursion]
\label{lem:evs-flow-recursion}
Under Assumption~\ref{ass:evs-laws}, at every supported prefix,
\begin{equation}
 U^*(s\mid C)=\mathbb E_\rho[R_\beta(X)\mid s,C]
 =1+a_\beta\mathbb E_\rho[r(X)\mid s,C]\in[1,e^\beta].
 \label{eq:evs-conditional-flow}
\end{equation}
For nonterminal $s$, writing $\operatorname{Ch}(s)$ for its supported children,
\begin{align}
 F^*(s)&=\sum_{s'\in\operatorname{Ch}(s)}F^*(s'),
 \label{eq:evs-flow-partition}\\
 U^*(s)&=\sum_{a\in\mathcal A_C(s)}\rho(a\mid s,C)
                 \sum_{s'}\mathsf K_C(s'\mid s,a)U^*(s').
 \label{eq:evs-flow-recursion}
\end{align}
The boundaries are $U^*(x)=R_\beta(x)$ and $U^*(s_0)=Z(C)$.
\end{lemma}
\begin{proof}
Divide the descendant sum defining $F^*(s)$ by the positive prefix
probability. By Eq.~\eqref{eq:evs-prefix-mass}, the resulting weights are
the normalized reference continuation law, yielding
Eq.~\eqref{eq:evs-conditional-flow}. Descendants of a nonterminal $s$
partition by their unique first child, proving
Eq.~\eqref{eq:evs-flow-partition}. Each child retains its incoming action,
so $P_\rho(s')=P_\rho(s)\rho(a\mid s,C)\mathsf K_C(s'\mid s,a)$.
Substitute $F^*(s')=P_\rho(s')U^*(s')$ in the child sum and divide by
$P_\rho(s)$ to obtain Eq.~\eqref{eq:evs-flow-recursion}. No additional
transition weight belongs in the sum of $F^*(s')$, since it already
includes the probability of reaching the child. A terminal has only
itself as a terminal descendant; the root descendant sum is $Z$.
\end{proof}

\begin{proposition}[Uniqueness with an appropriate horizon condition]
\label{prop:evs-flow-uniqueness}
On a history tree with a uniform finite action horizon $H$, $U^*$ is the
unique finite-valued solution of Eq.~\eqref{eq:evs-flow-recursion} with
terminal boundary $V(x)=R_\beta(x)$. On a possibly infinite-depth tree
with almost-sure reference termination from every supported prefix,
$U^*$ is the unique \emph{bounded} solution with this boundary.
\end{proposition}
\begin{proof}
For the finite-horizon statement, terminal values coincide. Suppose
$V=U^*$ at all supported successors of a nonterminal $s$. Applying the
same recursion to both functions gives equality at $s$. Induct backward
through the at most $H$ remaining transitions. This proves equality at
every supported history, also for countably many children since their
already identified bounded values have well-defined weighted sums, so the backward induction goes through unchanged.

For the infinite-depth statement, start a reference continuation at $s$,
and let $\tau$ be its remaining termination time. Extend the state process
mathematically by retaining its terminal state after $\tau$. The recursion
and iterated conditional expectation give
$V(s)=\mathbb E_\rho[V(S_{n\wedge\tau})\mid s,C]$ for every integer $n$.
Almost-sure termination implies $V(S_{n\wedge\tau})\to R_\beta(X)$.
Boundedness permits passage to the expectation, so
$V(s)=\mathbb E_\rho[R_\beta(X)\mid s,C]=U^*(s)$. The same argument
holds at each supported prefix. Equation~\eqref{eq:evs-conditional-flow}
already ensures that $U^*$ itself is bounded.
\end{proof}

\begin{remark}[Why infinite-tree boundedness matters]
\label{rem:evs-unbounded-recursion}
Consider a single action that, at nonterminal depth $n$, terminates with
probability $1/2$ and boundary value one, or advances to depth $n+1$.
Termination is almost sure. Nevertheless, for any $c>0$, the unbounded
finite-valued function $V_n=1+c2^n$ obeys
$V_n=\tfrac12+\tfrac12V_{n+1}$ and the same terminal boundary. It differs
from $U^*=1$. Boundedness in Proposition~\ref{prop:evs-flow-uniqueness} is therefore
what makes the solution unique.
\end{remark}

\subsection{Ideal Zero-Residual Consistency}
\label{sec:evs-A4}

\begin{definition}[Shared-flow residual]
\label{def:evs-shared-residual}
A shared flow is a real-valued function $u(s)$ assigning one value to a
prefix independently of which later continuation is scored. Its terminal
boundary is $u(s_T)=\log R_\beta(x)$. For $0\le i<j\le T$, set
\begin{equation}
 \delta^u_{i:j}(x)=u(s_i)+\sum_{t=i}^{j-1}\Delta\ell_t-u(s_j),\qquad
 \mathcal L_x^u=\frac1{K_T}\sum_{i<j}(\delta^u_{i:j}(x))^2,
 \quad K_T=\frac{T(T+1)}2.
 \label{eq:evs-shared-subtb}
\end{equation}
This is Eq.~\eqref{eq:subtb} with a common state function;
Appendix~\ref{app:evs-B} treats the trajectory-indexed estimate.
\end{definition}

\begin{proposition}[Ideal reference-relative consistency]
\label{prop:consist}
Under Assumption~\ref{ass:evs-laws}, suppose a shared $u$ satisfies the
terminal boundary and $\delta^u_{t:t+1}(x)=0$ on every supported complete
history and every action in that history. Then
\begin{equation}
 e^{u(s)}=U^*(s\mid C),\qquad u(s_0)=\log Z(C),\qquad
 P_\theta(x\mid C)=P^*(x\mid C)
 \label{eq:evs-consistency-conclusion}
\end{equation}
on the reference support. Requiring zero residual for all subtrajectories
is equivalent here to requiring it for all singleton intervals, so checking one-step residuals suffices.
\end{proposition}
\begin{proof}
Fix a supported prefix $s=s_m$. Adding the singleton equalities along
any supported complete continuation gives
\[
 \sum_{t=m}^{T-1}\Delta\ell_t
 =\sum_{t=m}^{T-1}(u(s_{t+1})-u(s_t))
 =\log R_\beta(x)-u(s).
\]
By Lemma~\ref{lem:evs-path-ratio},
$e^{u(s)}P_\theta(x\mid s,C)=P_\rho(x\mid s,C)R_\beta(x)$.
Sum over all terminal descendants. The shared value $e^{u(s)}$ can be
taken outside the sum, and the policy continuation probabilities sum to
one by termination. Hence
\[
 e^{u(s)}=\sum_{x\succeq s}P_\rho(x\mid s,C)R_\beta(x)=U^*(s\mid C).
\]
At terminal $s$, this equality follows directly from the boundary; at the
root, it yields the normalizer $u(s_0)=\log Z(C)$, and substituting this into
the root-to-terminal identity gives $P_\theta=P^*$. Finally, the same finite
sum of singleton equalities telescopes on any interval $i:j$. The reverse
implication follows because singleton intervals are included among all
subtrajectories.
\end{proof}

\begin{corollary}[Population zero loss under full coverage]
\label{cor:evs-population-zero}
Suppose the preceding shared-flow and history-law assumptions hold and a
fixed sampling law $\nu$ assigns positive mass to every reference-supported
terminal history. If $\mathbb E_\nu[\mathcal L_X^u]=0$, then
Proposition~\ref{prop:consist} applies.
\end{corollary}
\begin{proof}
Every loss is nonnegative. A positive loss at a history of positive
$\nu$-mass would give a positive expectation, so each supported history
has zero loss. Its finite sum of squared residuals then has every term
zero, in particular each singleton term. The consistency proposition
therefore applies.
\end{proof}

\subsection{Fixed-Executor Realizability and Its Obstruction}
\label{sec:evs-A5}

\begin{proposition}[Action-only realizability criterion]
\label{prop:evs-realizability}
Under the reference-law and reward premises of Assumption~\ref{ass:evs-laws},
define, on its support,
\begin{align}
 M^*(s,a)&=\sum_{s'}\mathsf K_C(s'\mid s,a)U^*(s'),
 \label{eq:evs-action-value}\\
 Q^*(a,s'\mid s)&=\rho(a\mid s,C)\mathsf K_C(s'\mid s,a)
                      \frac{U^*(s')}{U^*(s)},
 \label{eq:evs-target-joint}\\
  \pi^*(a\mid s)&=\rho(a\mid s,C)\frac{M^*(s,a)}{U^*(s)},\nonumber\\
 \mathsf K^*(s'\mid s,a)&=\mathsf K_C(s'\mid s,a)\frac{U^*(s')}{M^*(s,a)}.
 \label{eq:evs-target-disintegration}
\end{align}
Here $Q^*$ is the target joint successor law and
$Q^*=\pi^*\mathsf K^*$. Some arbitrary complete-history action policy
using the \emph{fixed} executor realizes $P^*$ if and only if
\begin{equation}
 U^*(s')=M^*(s,a)\quad
 \mathsf K_C(\cdot\mid s,a)\text{-almost surely}
 \label{eq:evs-realizability-condition}
\end{equation}
for every supported $(s,a)$. If realizable, that action policy is uniquely
$\pi^*$ on the reference support. A deterministic executor satisfies
Eq.~\eqref{eq:evs-realizability-condition} automatically.
\end{proposition}
\begin{proof}
The target probability of a supported child divided by the target
probability of its parent is $F^*(s')/F^*(s)$. Substitute the reference
prefix factorization from Lemma~\ref{lem:evs-flow-recursion}; since the
child records its action, this gives Eq.~\eqref{eq:evs-target-joint}.
Summing over children for a fixed action gives $\pi^*$, and dividing by
this positive action marginal gives $\mathsf K^*$. Both normalize by $M^*$ and the flow recursion.

For necessity, a policy $\pi$ realizing the complete target law must also
realize its prefix and joint successor probabilities. Its joint kernel
is $\pi(a\mid s)\mathsf K_C(s'\mid s,a)$. Equating it to $Q^*$ and
canceling a positive executor factor yields
$\pi(a\mid s)=\rho(a\mid s)U^*(s')/U^*(s)$. The left side does not
vary with $s'$, so $U^*(s')$ is constant on this successor support.
Averaging that constant gives $M^*(s,a)$, proving
Eq.~\eqref{eq:evs-realizability-condition} and $\pi=\pi^*$.

For sufficiency, under this condition $\mathsf K^*=\mathsf K_C$ and
$\pi^*\mathsf K_C=Q^*$. Multiplying these joint transitions on a complete
history telescopes the $U^*$ ratios, yielding
$P_\rho(x)U^*(x)/U^*(s_0)=P^*(x)$. These terminal probabilities sum to
one, so the constructed process terminates almost surely and realizes
the target. Uniqueness follows from its already determined action
marginals. A deterministic kernel has only one supported successor,
which proves the last statement.
\end{proof}

\begin{proposition}[Irreducible executor contribution to relative entropy]
\label{prop:evs-executor-kl}
Under the reference-law and reward premises of
Assumption~\ref{ass:evs-laws}, suppose the supported tree has a uniform
finite horizon $H$. Let $\Pi_C$ contain all normalized complete-history action
policies supported on the reference legal actions, including policies
with some zero action probabilities. For $\pi\in\Pi_C$, where $D_{\mathrm{KL}}(p\Vert q)=\sum_z p(z)\log[p(z)/q(z)]$,
zero-$p$ terms contribute zero, and a positive-$p$, zero-$q$ term
gives infinite divergence,
\begin{align}
 D_{\mathrm{KL}}(P^*\Vert P_\pi)
 &=\mathbb E_{P^*}\!\left[\sum_{t=0}^{T-1}
 D_{\mathrm{KL}}(\pi^*(\cdot\mid S_t)\Vert\pi(\cdot\mid S_t))\right]
 +\mathcal I_C,\label{eq:evs-kl-decomposition}\\
 \mathcal I_C
 &=\mathbb E_{P^*}\!\left[\sum_{t=0}^{T-1}
 D_{\mathrm{KL}}(\mathsf K^*(\cdot\mid S_t,A_t)
                 \Vert\mathsf K_C(\cdot\mid S_t,A_t))\right],
 \label{eq:evs-executor-gap}\\
 \inf_{\pi\in\Pi_C}D_{\mathrm{KL}}(P^*\Vert P_\pi)&=\mathcal I_C.
 \label{eq:evs-executor-infimum}
\end{align}
The infimum is attained by using the action marginal $\pi^*$ with the
fixed executor. Moreover, $\mathcal I_C=0$ exactly when
Eq.~\eqref{eq:evs-realizability-condition} holds on the reference support.
\end{proposition}
\begin{proof}
Factor the target with $\pi^*\mathsf K^*$ and the comparison process with
$\pi\mathsf K_C$. Their log density ratio is
\[
 \sum_{t=0}^{T-1}\log\frac{\pi^*(A_t\mid S_t)}{\pi(A_t\mid S_t)}
 +\sum_{t=0}^{T-1}\log
       \frac{\mathsf K^*(S_{t+1}\mid S_t,A_t)}
            {\mathsf K_C(S_{t+1}\mid S_t,A_t)}.
\]
For the first sum, condition on each nonterminal $S_t$ and average over
its target action law; the result is the action KL term in
Eq.~\eqref{eq:evs-kl-decomposition}. For the second, condition on
$(S_t,A_t)$ and average over $\mathsf K^*$; the result is
Eq.~\eqref{eq:evs-executor-gap}. Variable lengths cause no additional
term: for this calculation only, histories may be padded to $H$ with
deterministic terminal coordinates of log-ratio zero.

These steps remain valid as extended expectations. For probability
vectors $p,q$, the negative part of $\sum p\log(p/q)$ is bounded by
$\sum_{p<q}q(p/q)\log(q/p)\le1/e$, since
$z\log(1/z)\le1/e$ on $[0,1]$. There are at most $2H$ conditional
terms. Also $U^*,M^*\in[1,e^\beta]$, so the executor log-ratio is
bounded in absolute value by $\beta$. No subtraction of two infinite
positive expectations is required.

Each conditional KL is nonnegative. Taking $\pi=\pi^*$ makes every
action term zero and proves the infimum formula; finite horizon makes
this fixed-executor policy a terminating member of $\Pi_C$. The remaining
sum is zero precisely when $\mathsf K^*=\mathsf K_C$ at every target-supported
state-action pair. The target and reference have the same support, and
Eq.~\eqref{eq:evs-target-disintegration} makes this equality equivalent
to Eq.~\eqref{eq:evs-realizability-condition}.
\end{proof}

\begin{remark}[A stochastic-executor counterexample]
\label{rem:evs-stochastic-counterexample}
Let there be one legal action, after which the executor terminates with
$r=0$ or $r=1$, each with probability $1/2$. Every action policy has
success probability $1/2$, whereas for $\beta>0$ the target has success
probability $e^\beta/(1+e^\beta)>1/2$. Target realizability therefore
needs the condition of Proposition~\ref{prop:evs-realizability} in addition to
kernel cancellation.
\end{remark}

\subsection{Subtrajectory Algebra and Credit Coefficients}
\label{sec:evs-A6}

Fix a recorded trajectory $x$ of length $T\ge1$. All action log-ratios
and endpoint values are finite. Write $v_i$ for the value used at $s_i$
on this record, allowing $v_i=\widetilde u_x(s_i)$; reuse that same value
in every interval of the record. Put $e_t=v_t+\Delta\ell_t-v_{t+1}$,
$0\le t<T$, and $\delta^{(x)}_{i:j}=v_i+\sum_{t=i}^{j-1}\Delta\ell_t-v_j$.

\begin{lemma}[Within-trajectory telescoping]
\label{lem:evs-telescoping}
For every interval $0\le i<j\le T$,
$\delta^{(x)}_{i:j}=\sum_{t=i}^{j-1}e_t$.
\end{lemma}
\begin{proof}
Expand the sum of $e_t$. Each internal value $v_{i+1},\ldots,v_{j-1}$
appears once with each sign and cancels; only $v_i-v_j$ and the action
log-ratios remain. No other trajectory is used.
\end{proof}

\begin{proposition}[Matrix form, positive definiteness, and zero loss]
\label{prop:evs-subtb-matrix}
For all $0\le i<j\le T$ and $0\le t<T$, let $A_T[(i,j),t]=\mathbf1\{i\le t<j\}$.
With $G_T=A_T^{\mathsf T}A_T$,
\begin{equation}
 \delta^{(x)}=A_Te,\qquad
 \mathcal L_x=\frac1{K_T}e^{\mathsf T}G_Te,\qquad
 (G_T)_{rs}=(\min(r,s)+1)(T-\max(r,s)),
 \label{eq:evs-subtb-matrix}
\end{equation}
where $0\le r,s<T$. The matrix $G_T$ is positive definite and
\begin{equation}
 \mathcal L_x\ge\frac1{K_T}\sum_{t=0}^{T-1}e_t^2,\qquad
 \mathcal L_x=0\ \Longleftrightarrow\ e=0
 \ \Longleftrightarrow\ \delta^{(x)}_{i:j}=0\text{ for every }i<j.
 \label{eq:evs-zero-loss}
\end{equation}
\end{proposition}
\begin{proof}
Lemma~\ref{lem:evs-telescoping} gives the matrix multiplication and hence
the quadratic form. Entry $(r,s)$ counts intervals containing both
actions: there are $\min(r,s)+1$ possible left endpoints and
$T-\max(r,s)$ possible right endpoints. This proves the entry formula.
The rows indexed by $(t,t+1)$ form an identity matrix, so $A_T$ has
full column rank. For any nonzero $e$, $\|A_Te\|_2^2>0$, proving
positive definiteness. Retaining only the singleton squares proves the
lower bound and forces $e=0$ if the loss is zero. Conversely, $e=0$
gives every interval residual zero and therefore zero loss.
\end{proof}

\begin{proposition}[Action coefficients and cumulative-sum identity]
\label{prop:evs-credit}
Holding the endpoint values fixed when differentiating with respect to
action log-ratios, the coefficient for action $t$ is
\begin{equation}
 \Gamma_t(x):=\frac{\partial\mathcal L_x}{\partial\Delta\ell_t}
 =\frac2{K_T}\sum_{i\le t<j}\delta^{(x)}_{i:j}
 =\frac2{K_T}(G_Te)_t.
 \label{eq:evs-credit}
\end{equation}
There are $(t+1)(T-t)$ intervals in this sum. If $h_0=0$ and
$h_j=\sum_{t=0}^{j-1}e_t$, then
\begin{equation}
 K_T\mathcal L_x
 =(T+1)\sum_{j=0}^T h_j^2-\left(\sum_{j=0}^T h_j\right)^2.
 \label{eq:evs-cumulative-loss}
\end{equation}
\end{proposition}
\begin{proof}
The partial derivative of $\delta^{(x)}_{i:j}$ with respect to
$\Delta\ell_t$ is $\mathbf1\{i\le t<j\}$. Differentiate the finite
sum of squared residuals to obtain the first coefficient formula;
$A_T^{\mathsf T}\delta^{(x)}=G_Te$ gives the second. An interval
containing $t$ has $t+1$ possible starts and $T-t$ possible ends.
For the cumulative identity, telescoping gives
$\delta^{(x)}_{i:j}=h_j-h_i$. Therefore
\[
 \sum_{i<j}(h_j-h_i)^2
 =\frac12\sum_{i=0}^T\sum_{j=0}^T(h_j-h_i)^2
 =(T+1)\sum_{j=0}^Th_j^2-\left(\sum_{j=0}^Th_j\right)^2,
\]
where the last equality expands the two square terms and the cross term.
\end{proof}

\begin{remark}[Local coverage]
\label{rem:evs-credit-scope}
For $T=2$, $e_0=c$ and $e_1=-c$ give a zero full-trajectory residual
but $\mathcal L_x=2c^2/3>0$ when $c\ne0$. Singleton intervals expose
such local errors. Equation~\eqref{eq:evs-cumulative-loss} also allows an
algebraic linear-time evaluation once $e$ is known.
\end{remark}

\subsection{Approximate Consistency Under Uniform Residual Control}
\label{sec:evs-A7}

\begin{proposition}[Root, log-density, and total-variation bounds]
\label{prop:evs-uniform-residual}
Under Assumption~\ref{ass:evs-laws}, let $u$ be shared with the exact
terminal boundary. If $|\delta^u_{0:T(x)}(x)|\le\varepsilon$ for every
reference-supported complete history, where $\varepsilon\ge0$, then
\begin{align}
 |u(s_0)-\log Z(C)|&\le\varepsilon,
 \label{eq:evs-root-bound}\\
 \left|\log\frac{P_\theta(x\mid C)}{P^*(x\mid C)}\right|&\le2\varepsilon,
 \label{eq:evs-density-bound}\\
 \operatorname{TV}(P_\theta,P^*)&\le\tanh(\varepsilon),
 \label{eq:evs-tv-bound}
\end{align}
where $\operatorname{TV}(P,Q)=\tfrac12\sum_x|P(x)-Q(x)|$.
\end{proposition}
\begin{proof}
Write $d(x)=\delta^u_{0:T(x)}(x)$. Equation~\eqref{eq:evs-suffix-ratio}
and the terminal boundary give
\[
 P_\theta(x)/P^*(x)=\exp\bigl(d(x)+\log Z-u(s_0)\bigr).
\]
Summing against $P^*$ and using normalization shows
\[
 u(s_0)-\log Z=\log\mathbb E_{P^*}e^{d(X)}.
\]
Because $e^{-\varepsilon}\le e^{d(X)}\le e^\varepsilon$, this logarithm
lies in $[-\varepsilon,\varepsilon]$, proving the root bound. Subtract
it from $d(x)$ to obtain the log-density bound. In particular, for
$g(x)=P_\theta(x)/P^*(x)$, $e^{-2\varepsilon}\le g(x)\le e^{2\varepsilon}$.
For any positive $z$ in this interval,
\[
 \frac{|z-1|}{z+1}\le
 \frac{e^{2\varepsilon}-1}{e^{2\varepsilon}+1}=\tanh(\varepsilon).
\]
The fraction increases for $z\ge1$ and decreases for $z\le1$,
so its maximum is at an interval endpoint. Integrating
$|g-1|\le\tanh(\varepsilon)(g+1)$ under $P^*$ with
$\mathbb E_{P^*}g=1$ gives Eq.~\eqref{eq:evs-tv-bound}.
\end{proof}

\begin{corollary}[A uniformly controlled all-subtrajectory loss]
\label{cor:evs-uniform-loss}
If additionally $\eta\ge0$, $T(x)\le H$, and $\mathcal L_x^u\le\eta^2$ for every
supported complete history, the preceding bounds hold with
$\varepsilon=\sqrt{H(H+1)/2}\,\eta$.
\end{corollary}
\begin{proof}
The full-path square is one nonnegative summand of $K_T\mathcal L_x^u$, so
Proposition~\ref{prop:evs-uniform-residual} applies with
\[
 |\delta^u_{0:T}|\le\sqrt{K_T}\,\eta\le\sqrt{H(H+1)/2}\,\eta=\varepsilon .\qedhere
\]
\end{proof}

\subsection{Mixed-Behavior Regression and Gradient Bounds}
\label{sec:evs-A8}

\begin{remark}[Gradient coefficients and data regime]
\label{rem:gradient}
The action coefficient in Proposition~\ref{prop:evs-credit} is a partial
derivative on a recorded trajectory. In the implemented objective,
$\theta$ receives gradients through recomputed action log-probabilities
and $\psi$ through nonterminal residual-head values. The reference
encodings and measured anchor statistics are detached. Trajectories are
collected asynchronously, so the behavior adapter may lag the trained
one by one update; paired rollouts have a forced first action.
\end{remark}

\begin{lemma}[Record-conditioned gradient routing]
\label{lem:evs-gradient-routing}
Condition on a recorded batch, its contexts, and detached statistics.
Assume $b_\psi(h_\rho(s),f(s))$ has no $\theta$-dependence. The AnchorTB
sample gradients are
\begin{align}
 \nabla_\theta\mathcal L_x
 &=\sum_{t=0}^{T-1}\Gamma_t(x)\nabla_\theta
                   \log\pi_\theta(a_t\mid s_t,C),
 \label{eq:evs-theta-gradient}\\
 \nabla_\psi\mathcal L_x
 &=\frac2{K_T}\sum_{i<j}\delta^{(x)}_{i:j}
       \bigl(\nabla_\psi\widetilde u_x(s_i)
             -\nabla_\psi\widetilde u_x(s_j)\bigr),
 \label{eq:evs-psi-gradient}
\end{align}
where $\nabla_\psi\widetilde u_x(s_T)=0$ because the terminal reward
boundary overrides the learned head.
\end{lemma}
\begin{proof}
Differentiate each squared residual using the chain rule. For $\theta$,
the endpoint values and reference log-probabilities are fixed, leaving
only the current action log-probabilities; collecting terms for the
same action gives Eq.~\eqref{eq:evs-theta-gradient}. For $\psi$, the
action log-ratios are fixed and the two endpoints have opposite signs,
giving Eq.~\eqref{eq:evs-psi-gradient}. Stop-gradient sets the derivative
of the measured part to zero, and the calculation conditions on the sampled history and its terminal reward.
\end{proof}

\begin{proposition}[A bounded-Jacobian gradient second-moment bound]
\label{prop:evs-gradient-moment}
Let $\xi$ be any chosen trainable parameter vector, with the record and
detached statistics held fixed. At differentiability points suppose
\[
 \frac1{K_T}\sum_{i<j}\|\nabla_\xi\delta^{(x)}_{i:j}\|_2^2\le G^2
\]
for a constant $G<\infty$. Then
\begin{equation}
 \|\nabla_\xi\mathcal L_x\|_2^2\le4G^2\mathcal L_x.
 \label{eq:evs-gradient-pointwise}
\end{equation}
For a fixed distribution of recorded inputs satisfying this bound and
$\mathbb E\mathcal L_X<\infty$,
\begin{equation}
 \operatorname{tr}\operatorname{Cov}(\nabla_\xi\mathcal L_X)
 \le\mathbb E\|\nabla_\xi\mathcal L_X\|_2^2
 \le4G^2\mathbb E\mathcal L_X.
 \label{eq:evs-gradient-second-moment}
\end{equation}
\end{proposition}
\begin{proof}
Apply vector Cauchy--Schwarz to the sample gradient:
\[
 \left\|\frac2{K_T}\sum_{i<j}\delta^{(x)}_{i:j}
                      \nabla_\xi\delta^{(x)}_{i:j}\right\|_2^2
 \le4\left(\frac1{K_T}\sum_{i<j}(\delta^{(x)}_{i:j})^2\right)
       \left(\frac1{K_T}\sum_{i<j}\|\nabla_\xi\delta^{(x)}_{i:j}\|_2^2\right).
\]
The first factor is the loss and the second is at most $G^2$.
Taking expectation proves the second-moment bound. The identity
$\operatorname{tr}\operatorname{Cov}(V)=\mathbb E\|V\|_2^2-\|\mathbb EV\|_2^2$
for square-integrable $V$ gives the rest.
\end{proof}

\paragraph{Sampling law.}
With a fixed data law $\nu$, AnchorTB is a regression of recorded residuals
under $\nu$. Under full coverage its exact shared zero is characterized by
Corollary~\ref{cor:evs-population-zero}, and nonzero optima depend on the
sampling weights. These sample-gradient identities concern the AnchorTB
loss; the value head and the diagnostic heads are fitted separately.

\section{Shrinkage Baselines, Structural Pooling, and the Value Head}
\label{app:evs-B}

The ideal conditional mean in Eq.~\eqref{eq:evs-conditional-flow} is a
property of a specified reference continuation law. Measured task
statistics, lagged structural buckets, and a feature-based value head
are different estimators with different data sources. We give their algebraic properties and population conditions.

\subsection{Implemented Trajectory-Indexed Flows}
\label{sec:evs-B1}

\begin{definition}[Implemented flow and terminal override]
\label{def:evs-implemented-flow}
For batch $k$ and scored trajectory $x$, put
$u_{q,k,x}=\log[1+a_\beta\hat p_{q,k}^{(-x)}]$. The flow estimate in
Eq.~\eqref{eq:anchor}, with its terminal boundary made explicit, is
\begin{equation}
 \widetilde u_{k,x}(s_i)=
 \begin{cases}
 \operatorname{sg}\!\left[\operatorname{clip}_{[0,\beta]}
       (u_{q,k,x}+c_{k-1}(g(s_i)))\right]
       +b_\psi(h_\rho(s_i),f(s_i)),&0\le i<T,\\
 \log R_\beta(x),&i=T.
 \end{cases}
 \label{eq:evs-implemented-flow}
\end{equation}
The second branch overrides the entire learned parameterization at the
terminal, including $b_\psi$. We suppress the fixed batch index and write
$\widetilde u_x,u_{q,x},\hat p_q^{(-x)},c$ when unambiguous. The symbol
$g(s)$ denotes the canonical structure bucket of the history.
\end{definition}

\begin{proposition}[When a trajectory-indexed family is shared]
\label{prop:evs-shared-descent}
On a given collection of scored prefixes, the family
$\{\widetilde u_x(s)\}$ defines a single state function if and only if
\begin{equation}
 \widetilde u_x(s)=\widetilde u_{x'}(s)
 \quad\text{whenever }x\succeq s\text{ and }x'\succeq s
 \text{ belong to that collection}.
 \label{eq:evs-shared-descent}
\end{equation}
Within every record, Lemma~\ref{lem:evs-telescoping} and
Propositions~\ref{prop:evs-subtb-matrix}--\ref{prop:evs-credit} apply
using Eq.~\eqref{eq:evs-implemented-flow}.
\end{proposition}
\begin{proof}
If a single function exists, evaluating it at $s$ gives the same value
for every continuation, which proves necessity. Conversely, define its
value at each represented prefix by selecting any one continuation.
Equation~\eqref{eq:evs-shared-descent} makes the definition independent
of that selection. This proves sufficiency on the specified collection.
The within-record identities only require that a fixed endpoint value
be reused in each interval of that record, which the definition provides.
\end{proof}

\begin{remark}[Partial LOO is numerically trajectory dependent]
\label{rem:evs-loo-counterexample}
For two same-task reference records with rewards one and zero, take
$m_0=1/2$ and the configured $\kappa_q=0.5$. Removing the successful
record from its task-level statistic gives $\hat p_q^{(-x_1)}=1/6$;
removing the failure gives $\hat p_q^{(-x_2)}=5/6$. With $c=0$,
$b_\psi=0$, and $\beta>0$, the common root receives two distinct values
$\log(1+a_\beta/6)$ and $\log(1+5a_\beta/6)$, both inside $[0,\beta]$.
Stop-gradient leaves these numerical values unchanged.
\end{remark}

\subsection{Hierarchical Shrinkage and Its Range}
\label{sec:evs-B2}

\begin{definition}[Hierarchical task baseline]
\label{def:evs-shrinkage}
Let $S$ and $N$ be reward sums and trajectory counts from the healthy,
non-paired reference pool, with subscripts $g,c,q$ for global, category,
and task statistics. Here $g$ in $S_g,N_g,m_g$ means the global level,
not a structure bucket. Define
\begin{alignat}{2}
 m_g&=\begin{cases}S_g/N_g,&N_g>0,\\1/2,&N_g=0,\end{cases}
 &\qquad m_c&=\frac{S_c+\kappa_gm_g}{N_c+\kappa_g},
 \label{eq:evs-shrinkage-parent}\\
 m_0&=\frac{n'_dp_d+\kappa_cm_c}{n'_d+\kappa_c},
 &\qquad \hat p_q^{(-x)}&=\frac{S_q^{(-x)}+\kappa_qm_0}{N_q^{(-x)}+\kappa_q}.
 \label{eq:evs-shrinkage-task}
\end{alignat}
The optional record prior is $(p_d,n_d)$, with $n'_d=\min(n_d,20)$;
absence of that prior gives $n'_d=0$. The configured strengths are
$(\kappa_g,\kappa_c,\kappa_q)=(2,4,0.5)$. If the scored record is in
the eligible task pool, it is removed once from that task sum and count;
otherwise no subtraction is made. Leave-one-out applies at the task level. The estimate is clipped to $[0,1]$ before the
reward transform. For bounded nonbinary rewards, $\hat p$ estimates a mean score, and for binary rewards a success probability.
\end{definition}

\begin{lemma}[Range of the hierarchical estimate]
\label{lem:evs-shrinkage-range}
Assume nonnegative counts, $0\le S\le N$ at every used level, a valid
task-level deletion, $p_d\in[0,1]$, $n'_d\ge0$, and positive shrinkage
strengths. Then $m_g,m_c,m_0,\hat p_q^{(-x)}\in[0,1]$ and
$0\le u_{q,x}\le\beta$.
\end{lemma}
\begin{proof}
A nonempty empirical mean $S/N$ lies in $[0,1]$, and the empty global
convention does too. For $N_c>0$, $m_c$ is the convex combination of
$S_c/N_c$ and $m_g$ with weights $N_c/(N_c+\kappa_g)$ and
$\kappa_g/(N_c+\kappa_g)$; if $N_c=0$ it equals $m_g$. The same
argument applies to $m_0$ and the task estimate, including zero-count
cases because their shrinkage denominators stay positive. Finally
$p\mapsto\log(1+a_\beta p)$ is nondecreasing on $[0,1]$, with endpoint
values $0$ and $\beta$. This holds for any reward in $[0,1]$.
\end{proof}

\begin{proposition}[Fixed-prior shrinkage moments and sensitivity]
\label{prop:evs-shrinkage-moments}
For fixed $n\ge0$, $m\in[0,1]$, $\kappa>0$, and rewards
$Y_i\in[0,1]$, let $\hat p=(\sum_{i=1}^nY_i+\kappa m)/(n+\kappa)$. If the $Y_i$ are
i.i.d. with mean $\mu$ and variance $\sigma^2$, then
\begin{alignat}{2}
 \mathbb E\hat p-\mu&=\frac{\kappa(m-\mu)}{n+\kappa},&\qquad
 \operatorname{Var}(\hat p)&=\frac{n\sigma^2}{(n+\kappa)^2},
 \label{eq:evs-shrinkage-moments}\\
 \mathbb E(\hat p-\mu)^2
 &=\frac{n\sigma^2+\kappa^2(m-\mu)^2}{(n+\kappa)^2}.&&
 \label{eq:evs-shrinkage-mse}
\end{alignat}
For a fixed record set, changing one bounded reward changes $\hat p$ by
at most $1/(n+\kappa)$. Changing only $m$ to $m'$ changes it by exactly
$\kappa|m-m'|/(n+\kappa)$. When $n\ge1$, deleting record $i$ while
holding the prior fixed gives
$\hat p-\hat p_{-i}=(Y_i-\hat p_{-i})/(n+\kappa)$, again of magnitude at most $1/(n+\kappa)$.
\end{proposition}
\begin{proof}
Linearity gives $\mathbb E\hat p=(n\mu+\kappa m)/(n+\kappa)$.
Independence makes the variance of the sum $n\sigma^2$; division by the
squared denominator gives its variance. Squared bias plus variance proves
the MSE expression, also at $n=0$. The replacement and prior sensitivities
follow by subtracting the two numerators over their common denominator.
For deletion, write
$\sum_{h\ne i}Y_h+\kappa m=(n-1+\kappa)\hat p_{-i}$ and substitute in
$\hat p$; subtracting $\hat p_{-i}$ gives the identity, bounded as $Y_i,\hat p_{-i}\in[0,1]$.
\end{proof}

\paragraph{Pooled estimand.}
The moment formulas above condition on a fixed prior and sample size. If
records come from contexts $C_h$ and pass a health event $A_h$, their
conditional means are $\mathbb E[r\mid C_h,A_h]$, and for fixed mixture weights
$w_h$ the pooled population mean is $\sum_h w_h\mathbb E[r\mid C_h,A_h]$.

\subsection{Anchor Sensitivity and Residual Stability}
\label{sec:evs-B3}

\begin{lemma}[Reward-transform sensitivity and clipping]
\label{lem:evs-anchor-lipschitz}
For $p,p'\in[0,1]$, let $f_\beta(p)=\log(1+a_\beta p)$ and
$A(p,c)=\operatorname{clip}_{[0,\beta]}(f_\beta(p)+c)$. Then
\begin{equation}
 |f_\beta(p)-f_\beta(p')|\le a_\beta|p-p'|,\qquad
 |A(p,c)-A(p',c')|\le a_\beta|p-p'|+|c-c'|.
 \label{eq:evs-anchor-lipschitz}
\end{equation}
\end{lemma}
\begin{proof}
For $\beta>0$, $f'_\beta(p)=a_\beta/(1+a_\beta p)\le a_\beta$,
so integration between $p$ and $p'$ proves the first inequality.
At $\beta=0$, both sides of that inequality are zero. Projection onto
a closed interval is nonexpansive: ordering the two inputs shows that
clipping can only shorten their distance. This holds for $[0,0]$ too. This and the triangle inequality
give the second bound.
\end{proof}

\begin{proposition}[Endpoint perturbations, loss, and credit stability]
\label{prop:evs-anchor-stability}
Compare two endpoint arrays $v_i,\bar v_i$ on one fixed trajectory with
the same action log-ratios, and suppose
$\max_{0\le i\le T}|\bar v_i-v_i|\le\varepsilon$. Then
\begin{alignat}{2}
 |\bar\delta_{i:j}-\delta_{i:j}|&\le2\varepsilon,&\qquad
 |\sqrt{\bar{\mathcal L}_x}-\sqrt{\mathcal L_x}|&\le2\varepsilon,
 \label{eq:evs-loss-stability}\\
 |\bar\Gamma_t-\Gamma_t|
 &\le\frac{4\varepsilon}{K_T}(t+1)(T-t),
 &\qquad &0\le t<T.
 \label{eq:evs-credit-stability}
\end{alignat}
If only the measured anchor changes by inputs satisfying
$|p-p'|\le\varepsilon_p$, $|c-c'|\le\varepsilon_c$, while $b_\psi$
and terminal rewards are fixed, these bounds hold with
$\varepsilon=a_\beta\varepsilon_p+\varepsilon_c$.
\end{proposition}
\begin{proof}
The difference of interval residuals is
$(\bar v_i-v_i)-(\bar v_j-v_j)$, bounded by $2\varepsilon$.
There are $K_T$ residuals, so the Euclidean distance between their
vectors is at most $2\varepsilon\sqrt{K_T}$. Divide the reverse triangle
inequality for their norms by $\sqrt{K_T}$ to obtain the loss bound.
By Eq.~\eqref{eq:evs-credit}, the coefficient difference is the sum of
$(t+1)(T-t)$ residual differences times $2/K_T$; bounding each
summand gives Eq.~\eqref{eq:evs-credit-stability}. Lemma~\ref{lem:evs-anchor-lipschitz} gives the rest, the terminal being fixed.
\end{proof}

\begin{remark}[The residual head lies outside clipping]
\label{rem:evs-head-outside-clip}
The measured part of Eq.~\eqref{eq:evs-implemented-flow} lies in
$[0,\beta]$ and $b_\psi$ is added outside it; a changing head adds its
perturbation to $\varepsilon$ in Proposition~\ref{prop:evs-anchor-stability}.

\end{remark}

\subsection{Ratio-First Structural Pooling}
\label{sec:evs-B4}

\begin{definition}[Visit-multiplicity pooling]
\label{def:evs-visit-pooling}
For a structural bucket $g$, let $\mathcal V_g$ be the multiset of
eligible prefix visits $(x,t)$ assigned to that bucket, and set
$n_g=|\mathcal V_g|$. Each included occurrence is counted in both the
sum and the denominator. Put $B_x=e^{u_{q,x}}$ and
$z_{x,t}=R_\beta(x)/B_x$. For $n_0\ge0$, $c_{\max}\ge0$, and $n_g+n_0>0$, define
\begin{equation}
 A_g=1+\frac{\sum_{(x,t)\in\mathcal V_g}(z_{x,t}-1)}{n_g+n_0},\qquad
 c(g)=\operatorname{clip}_{[-c_{\max},c_{\max}]}(\log A_g).
 \label{eq:evs-ratio-pooling}
\end{equation}
These are the exact-arithmetic quantities before the positive-domain
numerical safeguard in the implementation description. Buckets use a
canonical graph and an open/stop marker, fall back to coarse node and
edge counts below eight visits, and are read one batch behind. The
configured values are $n_0=8$ and $c_{\max}=0.25$.
\end{definition}

\begin{lemma}[Positive denominator algebra and bounds]
\label{lem:evs-pooling-domain}
Under Definition~\ref{def:evs-visit-pooling}, if each $z_{x,t}>0$ is
finite, then
\begin{equation}
 A_g=\frac{n_0+\sum_{(x,t)\in\mathcal V_g}z_{x,t}}{n_g+n_0}>0.
 \label{eq:evs-ratio-domain}
\end{equation}
If $R_\beta(x),B_x\in[1,e^\beta]$, then
$A_g\in[e^{-\beta},e^\beta]$ and
$|c(g)|\le\min(\beta,c_{\max})$. If $n_g=0<n_0$, then $A_g=1$ and
$c(g)=0$, and the formula defines no value when $n_g=n_0=0$.
\end{lemma}
\begin{proof}
The number of subtracted ones is exactly $n_g$, so combining terms over
the common denominator gives Eq.~\eqref{eq:evs-ratio-domain}.
If $n_g>0$, its numerator contains a positive term; if $n_g=0$, the
positive denominator forces $n_0>0$. This proves strict positivity.
The ratio is a weighted average of the values $z_{x,t}$ and the
pseudo-observation one. Each lies in $[e^{-\beta},e^\beta]$ under the
additional bounds. The same interval contains their weighted average;
logarithm and clipping give the stated bound. Substituting $n_g=0$ into the
ratio gives $A_g=1$ and $c(g)=0$, which proves the empty-bucket case.
\end{proof}

\begin{proposition}[A visit-weighted population interpretation]
\label{prop:evs-visit-estimand}
For this proposition only, suppose the tuples $(X_h,B_h,V_h)$ are i.i.d.
under a fixed law, where $z_h=R_\beta(X_h)/B_h>0$ and $V_h$ counts the
eligible visits to a fixed bucket in record $h$. Assume
$0<\mathbb EV_h<\infty$ and $\mathbb E[V_hz_h]<\infty$.
With a fixed finite $n_0$, pooling the first $m$ records satisfies
\begin{equation}
 \frac{n_0+\sum_{h=1}^mV_hz_h}{n_0+\sum_{h=1}^mV_h}
 \longrightarrow \frac{\mathbb E[V_hz_h]}{\mathbb EV_h}
 \quad\text{almost surely}.
 \label{eq:evs-visit-limit}
\end{equation}
The log and clipped-log quantities converge to the corresponding
transforms of this positive limit.
\end{proposition}
\begin{proof}
The strong law applied to $V_hz_h$ and $V_h$ makes their averages
converge to the two expectations. Divide numerator and denominator
by $m$; $n_0/m\to0$, and the limiting denominator is strictly positive,
so the quotient converges. Its limiting numerator is positive because
$z_h>0$ on the positive-probability event $V_h>0$. Continuity of $\log$ at a positive limit and of clipping gives the rest.
\end{proof}

\paragraph{Three sources of statistical discrepancy.}
First, ratio-first pooling targets an arithmetic ratio average, and the
finite-sample Jensen gap remains: for a positive random average $\bar z$,
$\mathbb E\log\bar z\le\log\mathbb E\bar z$.
Second, random denominators matter even before taking a logarithm:
\begin{equation}
 \mathbb E[R/B]=\mathbb E[R]\mathbb E[B^{-1}]
                  +\operatorname{Cov}(R,B^{-1}),
 \label{eq:evs-random-denominator}
\end{equation}
when these moments exist. This identity is the definition of covariance; with $R=1$ and
$B$ equally likely to be one or two, $\mathbb E[R/B]=3/4$ while
$\mathbb E[R]/\mathbb E[B]=2/3$.
For $\beta\ge\log2$ these denominators obey the preceding bounded
range; shrinkage and clipping add further differences.

Third, visit pooling weights a record by $V_h$, so its limiting ratio in
Eq.~\eqref{eq:evs-visit-limit} is a visit-weighted average. Repeated visits
share the terminal reward and can be strongly correlated. For a fixed number $n>0$ of random
visit ratios with finite second moments,
\[
 \operatorname{Var}\!\left(\frac1n\sum_{v=1}^nz_v\right)
 =\frac1{n^2}\sum_{v=1}^n\sum_{w=1}^n\operatorname{Cov}(z_v,z_w),
\]
by bilinearity of covariance, so a visit count measures visits rather
than independent tasks. Structural reads lag one batch, and buckets also include paired
reference prefixes (Section~\ref{sec:evs-C5}).

\subsection{Value-Head Projection and Calibration}
\label{sec:evs-B5}

\begin{definition}[Feature-level regression distribution]
\label{def:evs-value-data}
Let $\mathcal D$ be the actual distribution of scored training pairs
$(S,Y)$ for the value head, where $X$ is the terminal continuation record
from $S$ and $Y=r(X)\in[0,1]$ its reward. For Eq.~\eqref{eq:value},
$\mathcal D$ denotes the actual scored-pair law represented by
$\mathcal D_\rho$, including its stated data filters.
The input is
$z(S)=[f(S);\operatorname{onehot}(\operatorname{tasktype}(q))]\in\mathbb R^{30+6}$. Write $Z_f=z(S)$ for the random feature
vector; it is unrelated to the normalizer $Z(C)$. Define
$m_{\mathcal D}(Z_f)=\mathbb E_{\mathcal D}[Y\mid Z_f]$.
\end{definition}

\begin{proposition}[Feature-conditional MSE projection and calibration]
\label{prop:evs-value-projection}
Over measurable square-integrable functions of $Z_f$, the population
risk in Eq.~\eqref{eq:value} decomposes as
\begin{equation}
 \mathbb E_{\mathcal D}(v(Z_f)-Y)^2
 =\mathbb E_{\mathcal D}(v(Z_f)-m_{\mathcal D}(Z_f))^2
   +\mathbb E_{\mathcal D}\operatorname{Var}(Y\mid Z_f).
 \label{eq:evs-mse-projection}
\end{equation}
Its unique minimizer up to $\mathcal D$-null sets is $m_{\mathcal D}$,
and
\begin{equation}
 \mathbb E_{\mathcal D}[Y\mid m_{\mathcal D}(Z_f)]=m_{\mathcal D}(Z_f).
 \label{eq:evs-mean-calibration}
\end{equation}
For any such predictor $v(Z_f)$, its squared mean-calibration error
is at most its excess risk above this unrestricted optimum:
\begin{equation}
 \mathbb E_{\mathcal D}\!\bigl[
       (\mathbb E_{\mathcal D}[Y\mid v(Z_f)]-v(Z_f))^2\bigr]
 \le\mathbb E_{\mathcal D}(m_{\mathcal D}(Z_f)-v(Z_f))^2.
 \label{eq:evs-approx-calibration}
\end{equation}
\end{proposition}
\begin{proof}
Write $v-Y=(v-m_{\mathcal D})+(m_{\mathcal D}-Y)$ and expand its square.
The cross term has conditional expectation zero given $Z_f$ because
$\mathbb E[Y-m_{\mathcal D}\mid Z_f]=0$. The remaining conditional
squared error is $\operatorname{Var}(Y\mid Z_f)$, proving the risk
identity. The first term is nonnegative and vanishes exactly when
$v=m_{\mathcal D}$ almost surely, which proves the minimizer claim.
Since $m_{\mathcal D}$ is measurable with respect to $Z_f$, the tower
property gives Eq.~\eqref{eq:evs-mean-calibration}. For the last bound,
$v$ is also a function of $Z_f$, and hence
\[
 \mathbb E[Y\mid v]-v=\mathbb E[m_{\mathcal D}-v\mid v].
\]
Conditional Jensen's inequality for the square followed by expectation
gives Eq.~\eqref{eq:evs-approx-calibration}.
\end{proof}

\paragraph{Complete-history values and a finite sigmoid network.}
Equality of $m_{\mathcal D}(z(s))$ with the ideal reference value
$\mu_\rho(s,C)=\mathbb E_\rho[r(X)\mid s,C]$ holds when the continuation-data
law is compatible and the features are sufficient for that conditional mean;
two histories that share features but have values $1/4$ and $3/4$ show why
sufficiency is needed. Calibration above is with respect to $\mathcal D$.

The implemented head is the sigmoid MLP of Eq.~\eqref{eq:value} with $30+6$ inputs and
hidden width 32: the 30 execution features and a six-slot task-type one-hot, one slot per
IID benchmark. Finite
logits produce outputs strictly between zero and one; an unrestricted sigmoid
predictor approaches a conditional mean of zero or one by clipping it to
$[\epsilon,1-\epsilon]$ and taking its logit, with squared excess risk at most
$\epsilon^2$. The head estimates a success probability
for binary rewards and an expected score otherwise. Its weights are
exported per batch and evaluated on CPU with no extra encoder call. The two larger outcome heads on $[h_\rho(s);f(s)]$
serve as diagnostics.

\subsection{The Implemented Optimum and the Ideal Flow}
\label{sec:evs-B6}

\begin{proposition}[Squared-log regression and the arithmetic normalizer]
\label{prop:evs-log-counterexample}
In the one-action executor of Remark~\ref{rem:evs-stochastic-counterexample},
take a shared scalar root value $u_0$ and the exact terminal boundary.
Taking expectation under that reference executor law, for $\beta>0$
the population AnchorTB loss is minimized at
$u_0=\beta/2$, not at the ideal root
$\log[(1+e^\beta)/2]$, and its minimum is $\beta^2/4>0$.
\end{proposition}
\begin{proof}
There is one action and both policies assign it probability one, so
$\Delta\ell_0=0$, $T=K_T=1$. The terminal log reward is zero or
$\beta$, each with probability $1/2$. Consequently
\[
 \mathcal L(u_0)=\tfrac12u_0^2+\tfrac12(u_0-\beta)^2
 =(u_0-\beta/2)^2+\beta^2/4.
\]
This proves the minimizer and minimum; also
$Z=(1+e^\beta)/2$. Strict concavity of $\log$ on the two distinct
rewards gives
$\mathbb E\log R_\beta=\beta/2<\log\mathbb ER_\beta$ for $\beta>0$, and both vanish at $\beta=0$.
\end{proof}

\begin{remark}[Why sharedness is needed]
\label{rem:evs-trajectory-zero}
For any fixed pair of supported policies, define the artificial family
$v_x(s_i)=\log R_\beta(x)-\sum_{t=i}^{T-1}\Delta\ell_t$ on each complete
record. The empty sum at $i=T$ gives the exact terminal boundary, and
subtraction shows $v_x(s_i)+\Delta\ell_i-v_x(s_{i+1})=0$ at every step.
Thus all its within-record losses are zero whether or not
$P_\theta=P^*$, and the family is generally not shared across continuations;
distributional results therefore rely on sharedness.
\end{remark}

The ideal $u^*(s\mid C)$, the measured task anchor $u_{q,x}$, the
implemented $\widetilde u_x(s)$, and a population regression minimizer
are distinct objects. The first is identified by a conditional
expectation and a recursion; the next two are specified estimators;
the last on its function class and training law.

\section{Sequential Validation and the Training Loop}
\label{app:evs-C}

Paired validation compares two specified rollout regimes. We identify
that effect, prove fixed-look directional tests and the error accounting
of an adaptive run, and close with the training loop.

\subsection{Paired Intervention Regimes and Effect Estimands}
\label{sec:evs-C1}

\begin{definition}[Paired regimes, counts, and empty samples]
\label{def:evs-pairs}
For a registered candidate, the two arms start from the empty team with
the same task record, runtime configuration, value-head snapshot, and
forced first role. The candidate arm binds the candidate at its first
\textsc{add\_agent}; the control arm does not, and keeps the candidate
unavailable later. Both arms subsequently use the reference policy under
their respective histories. They share the other validated skills, and their
later decisions and observations may differ. For complete pair $i$, set
$b_i^\pm=\mathbf1\{r(x_i^\pm)\ge1/2\}$ and define
\begin{alignat}{2}
 W&=\sum_{i=1}^N\mathbf1\{b_i^+=1,b_i^-=0\},&\qquad
 L&=\sum_{i=1}^N\mathbf1\{b_i^+=0,b_i^-=1\},\nonumber\\
 T_0&=\sum_{i=1}^N\mathbf1\{b_i^+=b_i^-\},&\qquad
 D&=W+L,\qquad N=W+L+T_0.
 \label{eq:evs-paired-counts}
\end{alignat}
When $N>0$, $\hat\Delta=(W-L)/N$. For $N=0$ the empirical effect is
undefined, rather than an estimate of zero. With $D=0$, both directional
p-values below are defined to be one, and there is no rejection.
\end{definition}

\begin{lemma}[Paired effect identity]
\label{lem:evs-paired-effect}
For $N>0$,
\begin{equation}
 \hat\Delta=\frac1N\sum_{i=1}^N(b_i^+-b_i^-),\qquad |\hat\Delta|\le1.
 \label{eq:evs-paired-effect}
\end{equation}
For a specified pair distribution, let $p_W$ and $p_L$ be its win and
loss probabilities. Its threshold-pass effect is
\begin{equation}
 \Delta=\mathbb P(b^+=1)-\mathbb P(b^-=1)=p_W-p_L.
 \label{eq:evs-population-effect}
\end{equation}
Neither identity requires independence of the two arms within a pair.
\end{lemma}
\begin{proof}
A win contributes one to $b_i^+-b_i^-$, a loss contributes minus one,
and a tie contributes zero. Summing proves the sample identity, and
boundedness of each summand proves its range. Taking expectations gives the population identity without factoring
joint arm probabilities.
\end{proof}

\begin{remark}[The intervention and the reward threshold]
\label{rem:evs-paired-regime}
The intervention begins at the first action of an empty team, and the
control keeps the candidate unavailable throughout, so $\Delta$ is the
full-regime threshold-pass effect, later decisions included. For continuous scores $\Delta$ counts threshold crossings on $r$:
a reward moving from $0.49$ to $0.51$ gives $\Delta=1$ but changes the mean by $0.02$.
\end{remark}

\subsection{Exact Fixed-Look Directional Tests}
\label{sec:evs-C2}

\begin{assumption}[Registered fixed-look i.i.d. pair model]
\label{ass:evs-iid-pairs}
Let $\mathcal G_j$ denote information available at registration of
comparison $j$. Conditional on it, the candidate, arm protocols, target
pair distribution, and a finite nonnegative integer $N$ of complete
pairs for this look are fixed. The future pair outcomes are conditionally i.i.d. under that
distribution. Within-pair dependence is allowed. No data used to choose
the candidate are reused as its future validation outcomes under this
assumption. Denote the conditional win and loss probabilities by
$p_W,p_L$, and the corresponding mean threshold effect by
$\Delta_j=p_W-p_L$.
\end{assumption}

\begin{proposition}[Both composite directional nulls at a fixed look]
\label{prop:evs-fixed-sign}
Under Assumption~\ref{ass:evs-iid-pairs}, define
\begin{equation}
 p_+=\mathbb P\{B_D\ge W\},\qquad
 p_-=\mathbb P\{B_D\ge L\},\qquad B_D\sim\operatorname{Bin}(D,1/2),
 \label{eq:evs-sign-pvalues}
\end{equation}
Here probabilities in the displayed tails are over the auxiliary binomial
variable with the observed $D,W,L$ held fixed; use the $D=0$ convention
of Definition~\ref{def:evs-pairs}. For every $u\in[0,1]$,
\begin{align}
 H_j^+:\Delta_j\le0&\quad\Longrightarrow\quad
 \mathbb P(p_+\le u\mid\mathcal G_j)\le u,
 \label{eq:evs-positive-sign-valid}\\
 H_j^-:\Delta_j\ge0&\quad\Longrightarrow\quad
 \mathbb P(p_-\le u\mid\mathcal G_j)\le u.
 \label{eq:evs-negative-sign-valid}
\end{align}
Each inequality holds on its null event, so each tail is valid for its
whole composite directional null.
\end{proposition}
\begin{proof}
All probabilities in the proof condition on $\mathcal G_j$. Put
$p_D=p_W+p_L$. If $p_D=0$, then $D=0$ almost surely and $p_+=p_-=1$,
which satisfies both inequalities; the same conclusion holds if $N=0$.
Otherwise, for $d$ of positive probability and $0\le w\le d$, the
multinomial law of wins, losses, and ties gives
\begin{align*}
 \mathbb P(W=w,L=d-w,D=d)
 &=\frac{N!}{w!(d-w)!(N-d)!}
       p_W^w p_L^{d-w}(1-p_D)^{N-d},\\
 \mathbb P(D=d)&=\binom Nd p_D^d(1-p_D)^{N-d}.
\end{align*}
Dividing, with the usual limiting conventions when a cell probability
is zero, proves
\[
 W\mid D=d\sim\operatorname{Bin}(d,\vartheta),\qquad
 \vartheta=p_W/p_D,\qquad
 L\mid D=d\sim\operatorname{Bin}(d,1-\vartheta).
\]

Under $H_j^+$ we have $p_W\le p_L$, so $\vartheta\le1/2$. With independent
uniform variables $U_1,\ldots,U_d$, the sum
$\sum_i\mathbf1\{U_i\le\vartheta\}$ is pointwise at most
$\sum_i\mathbf1\{U_i\le1/2\}$. Hence the former binomial is
stochastically dominated by the latter. Define
\[
 k_u(d)=\min\{k\in\{0,\ldots,d\}:\mathbb P(B_d\ge k)\le u\},
\]
using $k_u(d)=d+1$ if the set is empty. The tail is nonincreasing in its
observed count, so $\{p_+\le u\}=\{W\ge k_u(d)\}$ conditional on $D=d$,
and stochastic domination then gives
\[
 \mathbb P(p_+\le u\mid D=d)\le\mathbb P(B_d\ge k_u(d))\le u .
\]
For $d=0$, the assigned value one has the same validity property.
Averaging over $D$ proves Eq.~\eqref{eq:evs-positive-sign-valid}.

Under $H_j^-$, $\vartheta\ge1/2$, so $1-\vartheta\le1/2$.
Apply the same uniform coupling to the conditional law of $L$.
The event $\{p_-\le u\}$ is $\{L\ge k_u(d)\}$, whose probability
is at most the corresponding fair-binomial upper tail and thus at most
$u$. Averaging over $D$, including $D=0$, proves
Eq.~\eqref{eq:evs-negative-sign-valid}. 
\end{proof}

\begin{lemma}[A stronger fixed-context conditional-sign model]
\label{lem:evs-context-sign}
Fix a candidate and a prespecified finite collection of complete pairs.
Condition on the full vector of contexts $\mathbf C$, the full discordance
vector $\mathbf D=(\mathbf1\{b_i^+\ne b_i^-\})_i$, and registration
information. Suppose the retained win signs $Y_i=\mathbf1\{b_i^+=1,b_i^-=0\}$
are conditionally independent. If every retained sign has conditional
probability $p_i\le1/2$, then $p_+$ is conditionally super-uniform. If
every $p_i\ge1/2$, then $p_-$ is conditionally super-uniform. If all
$p_i=1/2$, then $W\mid(\mathbf C,\mathbf D,\mathcal G_j)$ is exactly
$\operatorname{Bin}(D,1/2)$.
\end{lemma}
\begin{proof}
After conditioning, the retained index set and its size $D=d$ are fixed.
For the positive direction construct independent uniforms and represent
each independent sign as $Y_i=\mathbf1\{U_i\le p_i\}$. When $p_i\le1/2$,
each sign is bounded above by $\mathbf1\{U_i\le1/2\}$, so their sum is
stochastically dominated by $\operatorname{Bin}(d,1/2)$. Using the cutoff
$k_u(d)$ from the preceding proof gives the super-uniform upper tail.
When $p_i\ge1/2$, the independent loss signs $1-Y_i$ instead have
probabilities at most $1/2$; the same argument applies to $p_-$.
When every $p_i=1/2$, independence identifies the sum as precisely the
fair binomial. The zero-discordance convention handles an empty set.
\end{proof}

\begin{corollary}[Two-direction accounting at one look]
\label{cor:evs-two-directions}
If each valid direction receives level $b/2$, the probability of any
false directional rejection at that look is at most $b$. Also, for the
p-values in Eq.~\eqref{eq:evs-sign-pvalues}, both directions cannot
simultaneously satisfy $p_\pm\le a$ when $a<1/2$.
\end{corollary}
\begin{proof}
Sum the error probabilities only over directions whose nulls are true.
There are at most two, each bounded by $b/2$, so the union bound gives
$b$, without independence. For the second claim, fair-binomial symmetry
and $L=D-W$ imply $p_-=\mathbb P(B_D\le W)$, so
\[
 p_++p_-=1+\mathbb P(B_D=W)\ge1,
\]
which no two numbers at most $a<1/2$ satisfy. If $D=0$, both p-values equal one.
\end{proof}

\paragraph{Two null models.}
The i.i.d. pair proposition tests an average threshold effect under the
registered pair distribution, possibly averaging over randomly sampled
tasks. The fixed-context lemma assumes the stronger sign inequalities for
every retained pair after conditioning on \emph{all} contexts and discordance
indicators; equally weighted contexts with certain wins in one and certain
losses in the other show that an average zero effect can coexist with
opposite deterministic conditional signs.

\subsection{Adaptive Selection, Repeated Problems, and Repeated Looks}
\label{sec:evs-C3}

\begin{remark}[Correlated-sign counterexample]
\label{rem:evs-dependent-signs}
Draw one fair Bernoulli variable $Y$, and for $i=1,\ldots,8$ set
$(b_i^+,b_i^-)=(Y,1-Y)$. Every pair has zero marginal
effect, but all eight signs agree. The observed direction always has
p-value $2^{-8}=0.00390625$. At the configured $\alpha=0.05$ and the
first allocation $j=\ell=1$, its threshold is $\alpha/8=0.00625$.
Thus one false directional rejection occurs with probability one,
although the nominal spending formula is obeyed. Correlated
duplicates therefore carry the evidence of one pair rather than eight, and
validity depends on the conditional outcome and dependence model.
\end{remark}

\begin{remark}[Candidate-selection counterexample]
\label{rem:evs-selected-candidate}
For integers $M\ge2$ and $d\ge1$, consider $M$ true zero-effect
candidates, each with $d$ independent fair
validation signs, independent across candidates. Use these same signs
to select a candidate with all wins whenever one exists, and then report
its ordinary fixed-candidate p-value. The event of at least one all-win
candidate has probability $1-(1-2^{-d})^M$, while the selected p-value
on this event is $2^{-d}$. For $M>1$ this exceeds its nominal tail
probability and can approach one. The invalidity arises from
selection and reuse of screening outcomes. The protocol below
selects from registration information and validates on future data.
\end{remark}

\paragraph{Planned looks and cumulative data.}
Suppose a fixed candidate has valid p-values at prespecified cumulative
sample sizes $n_1,n_2,\ldots$, with deterministic levels $a_1,a_2,\ldots$.
Although the looks reuse data and their p-values are correlated,
\[
 \mathbb P\{\text{some true-null look rejects}\}
 \le\sum_\ell\mathbb P(p_\ell\le a_\ell)
 \le\sum_\ell a_\ell.
\]
Independence between looks is unnecessary, and early stopping only removes
potential later rejections. What matters is that the \emph{actual} look
indexed by $\ell$ retains a valid null law; renumbering selected looks,
choosing the sample size by favorable signs, reusing candidate-selection
outcomes, or filtering complete pairs on outcomes can change that law. The
registration-based protocol below conditions on registration information,
which is all the union bound requires.

\subsection{Nominal Spending and Conditional Countable FWER}
\label{sec:evs-C4}

\begin{proposition}[Alpha-spending accounting]
\label{prop:budget}
Let $0<\alpha<1$, and assign each registered comparison $j\ge1$ and
look $\ell\ge1$ two directional allocations
\begin{equation}
 a_{j,\ell,d}=\frac{\alpha}{2j(j+1)\ell(\ell+1)},\qquad d\in\{+,-\}.
 \label{eq:evs-alpha-allocation}
\end{equation}
If each directional index is globally unique and used at most once,
any realized run spends at most $\alpha$.
\end{proposition}
\begin{proof}
For every integer $M\ge1$,
$\sum_{n=1}^M1/[n(n+1)]=\sum_{n=1}^M(1/n-1/(n+1))=1-1/(M+1)$.
Therefore, for finite $J,L$,
\[
 \sum_{j=1}^J\sum_{\ell=1}^L\sum_{d\in\{+,-\}}a_{j,\ell,d}
 =\alpha\left(1-\frac1{J+1}\right)\left(1-\frac1{L+1}\right).
\]
Monotone limits give total allocation $\alpha$, and the consumed indices,
a subset, sum to at most $\alpha$.
\end{proof}

\begin{proposition}[Conditional-validity family-wise error bound]
\label{prop:evs-conditional-fwer}
For each registered $j$, let $H_{j,d}$ be the event that its directional
null is true, measurable with respect to registration information
$\mathcal G_j$. Let $E_{j,\ell,d}$ be the event that the directional test
is actually performed and rejects a true null. Unregistered or unperformed
tests have empty rejection events. Suppose the actual selection, sampling,
and observation rules ensure
\begin{equation}
 \mathbb P(E_{j,\ell,d}\mid\mathcal G_j)
 \le a_{j,\ell,d}\mathbf1_{H_{j,d}}
 \quad\text{almost surely for every }j,\ell,d.
 \label{eq:evs-conditional-error-premise}
\end{equation}
Then the probability of any false directional rejection is at most
$\alpha$. More generally, if a pre-run information field $\mathcal F_0$
is contained in every $\mathcal G_j$, the same bound holds conditional
on $\mathcal F_0$.
\end{proposition}
\begin{proof}
By the tower property, Eq.~\eqref{eq:evs-conditional-error-premise} gives
$\mathbb P(E_{j,\ell,d})\le a_{j,\ell,d}$. For a finite rectangle of
indices, the probability of the union is at most their summed probabilities.
Increasing these rectangles to all positive indices and using continuity
of probability from below gives
\[
 \mathbb P\!\left(\bigcup_{j,\ell,d}E_{j,\ell,d}\right)
 \le\sum_{j,\ell,d}\mathbb P(E_{j,\ell,d})
 \le\sum_{j,\ell,d}a_{j,\ell,d}=\alpha.
\]
For the conditional version, the tower property with
$\mathcal F_0\subseteq\mathcal G_j$ first gives
$\mathbb P(E_{j,\ell,d}\mid\mathcal F_0)\le a_{j,\ell,d}$.
Apply the finite conditional union bound and conditional monotone
convergence to the same increasing sequence of unions. This proves
the bound almost surely given $\mathcal F_0$. The argument uses no
independence between comparisons or looks, so it covers dependent tests.
\end{proof}

\begin{assumption}[A sufficient registration-and-future-data protocol]
\label{ass:evs-sufficient-protocol}
For each comparison $j$, the protocol consists of the following conditions. The candidate may be chosen
adaptively from $\mathcal G_j$, but its content, arm protocols, response
law, and target pair distribution are then fixed for that comparison.
Conditional on $\mathcal G_j$, future complete pairs form the i.i.d.
stream in Assumption~\ref{ass:evs-iid-pairs}; past screening outcomes are
not reused, and pair inclusion does not selectively retain favorable
outcomes. A nondecreasing countable sequence of finite nonnegative integer
cumulative sample sizes $n_{j,\ell}$ is $\mathcal G_j$-measurable and fixed before those future outcomes are
seen. Each potential p-value is computed from the first $n_{j,\ell}$
pairs, and keeps its original planned index and allocation. Whether a
planned test is ultimately performed may depend on available information,
but cannot change that potential test or its data law. Unperformed
planned tests consume no level.
\end{assumption}

\begin{corollary}[FWER under the sufficient protocol]
\label{cor:evs-protocol-fwer}
Assumption~\ref{ass:evs-sufficient-protocol}, the unique indices
of Proposition~\ref{prop:budget}, and rejection at the allocations in
Eq.~\eqref{eq:evs-alpha-allocation} imply the bound of
Proposition~\ref{prop:evs-conditional-fwer}.
\end{corollary}
\begin{proof}
Conditional on $\mathcal G_j$, each $n_{j,\ell}$ is fixed. Apply
Proposition~\ref{prop:evs-fixed-sign} to the corresponding future-data
prefix to obtain, on $H_{j,d}$,
$\mathbb P(p_{j,\ell,d}\le a_{j,\ell,d}\mid\mathcal G_j)\le a_{j,\ell,d}$.
The actual false rejection event is a subset of
$H_{j,d}\cap\{p_{j,\ell,d}\le a_{j,\ell,d}\}$, even when the decision
to execute a planned look uses earlier outcomes. Since $H_{j,d}$ is
registration-measurable, this subset relation proves
Eq.~\eqref{eq:evs-conditional-error-premise}. The preceding proposition
then gives the claimed bound, also when cumulative prefixes overlap.
\end{proof}

\begin{remark}[Error accounting for the online procedure]
\label{rem:evs-fwer-scope}
The implemented observation index advances when new wins or losses
exist. For this procedure we establish nominal accounting
(Proposition~\ref{prop:budget}) and whole-run FWER under the premise of
Eq.~\eqref{eq:evs-conditional-error-premise}. Dependence \emph{between}
valid planned tests leaves the union bound intact. These results bound
false directional rejections.
\end{remark}

\subsection{Training Algorithm and Data Routing}
\label{sec:evs-C5}

\paragraph{Skill proposal rules.}
The author model runs only in an author window and distils at most one
candidate per window from scored, side-effect-free trajectories, preferring
task families with few skills and contrasting a high-scoring run with a
same-task failure when one exists. A candidate is a structured procedure with
name, description, trigger, plan, pitfall, and constraint fields, and is
filtered for duplicates and answer leakage. Candidate slots are ordered by
fewest paired observations. Paired rollouts enter AnchorTB and structural buckets as off-policy paths
but not the task, category, and global baseline statistics, and the forced first
action is scored at its true probability under $\theta$ and $\rho$.

The following algorithm describes the frozen implementation.
It is distinct from the sufficient testing protocol in
Assumption~\ref{ass:evs-sufficient-protocol}. The same trajectory-indexed
endpoint values are reused across all intervals of each scored trajectory.
For paired records, ``the same context'' in the re-scoring step means
that $\theta$ and $\rho$ use that record's actual arm context, $C^+$ or
$C^-$, and its legal mask.

\begin{algorithm}[h]
\caption{One training step of EvoSteer}
\begin{enumerate}[leftmargin=2em,itemsep=1pt]
\item Sample $32$ task records balanced over sources.
\item Roll out $2\theta+2\rho$ natural trajectories per record, plus one
      candidate and one control rollout per record in a family with a candidate,
      under fixed skill-menu and value-head snapshots; execute every action on
      issue and append feedback and features.
\item Gate the batch on completeness and risk; publish the current adapter to
      the sampling service and prefetch the next batch.
\item Re-score every action under $\theta$ and $\rho$ with the same context,
      tokens, and grammar mask.
\item Fit the value head and auxiliary heads on legal reference states.
\item Build same-task shrinkage baselines; read structural buckets from earlier
      batches; accumulate this batch's ratios.
\item Compute all subtrajectory residuals of Eq.~\eqref{eq:subtb} and update
      $\theta$ and $b_\psi$.
\item Accumulate paired wins, losses, and ties; at each validation round, run the
      author window and the sequential tests of Eq.~\eqref{eq:signtest}.
\item Save a checkpoint with adapter, heads, optimiser, statistics, skills,
      and test state.
\end{enumerate}
\end{algorithm}

\begin{table}[h]
\centering\small
\begin{tabular}{lp{10.6cm}}
\toprule
Group & Setting \\
\midrule
Backbones & orchestrator, reference, executor: one shared base model (\mbox{Qwen3.5-9B}); skill author: DeepSeek-V4-Flash, called only in author windows \\
Adapter & LoRA rank 64, $\alpha=128$, on q/k/v/o\_proj, out\_proj, and in\_proj\_qkv; AdamW ($\beta_1=0.95$), lr $5\times10^{-6}$ \\
Heads & $b_\psi$: zero-initialized MLP on $[h_\rho;f]$; value head: two-layer MLP on $30{+}6$ inputs (one task-type slot per IID benchmark) \\
AnchorTB & $\beta=2$; all subtrajectories equally weighted; gradient-norm clip 1.0 \\
Baselines & $(\kappa_q,\kappa_c,\kappa_g)=(0.5,4,2)$; prior success rate capped at 20 counts; bucket minimum 8; pseudo-count $n_0=8$; correction bound \mbox{$c_{\max}=0.25$} \\
Skill slots & at most three validated skills and one candidate per task type \\
Batch & 32 tasks per step, each with $2\theta+2\rho$ natural trajectories, plus paired rollouts (192 episodes per step); asynchronous prefetch \\
Orchestration & team size, repair count, and episode length chosen by the policy; bounded only by the shared budget and the learned \textsc{stop} \\
Budget & shared resource budget of 98,304 tokens, 50 tool calls, and 600 s per episode \\
Temperatures & orchestrator 1.0; executor 0.3 \\
Skills & at most one new candidate per author window; library $\le60$; $\le12$ per family \\
Validation & sequential sign test at each validation round; $\alpha=0.05$ with two-level alpha spending \\
\bottomrule
\end{tabular}
\caption{Configuration of the frozen implementation used for the experiments.}
\label{tab:config}
\end{table}
\vspace{2pt}

\begin{remark}[Implementation details]
The paired design intervenes only at the first \textsc{add\_agent}, and
the active skill set exposes slot identifiers rather than full skill text. 
\end{remark}

\paragraph{Data routing.}
Natural healthy non-paired reference trajectories supply the hierarchical
task, category, and global statistics. Natural and paired paths can supply the
AnchorTB regression. Structural buckets also receive paired reference prefixes, without the
forced-prefix filter of the value-head data and with their visit
multiplicities retained. Conditional flow statements apply to a fixed $C$.

\subsection{Proofs of the Main-Text Propositions}
\label{sec:evs-C6}

The first main-text proposition is procedural: legal repairs execute
before the next decision and preserve the history representation
(Lemma~\ref{lem:evs-history-tree}). The second follows from the action
coefficient in Proposition~\ref{prop:evs-credit} and the estimator components
in Definition~\ref{def:evs-implemented-flow}. The third combines candidate-slot
and status-transition rules with Proposition~\ref{prop:budget}'s nominal
accounting; its statistical interpretation is fixed-look validity in
Proposition~\ref{prop:evs-fixed-sign} and the conditional whole-run result of
Proposition~\ref{prop:evs-conditional-fwer}.

\section{Experimental Details}
\label{sec:exp-details}

\paragraph{Benchmarks and splits.}
The six IID benchmarks (HotpotQA, NQ-Open, MedQA, AIME 2026, MBPP+, ALFWorld)
supply the training tasks, and their test items are disjoint from those tasks;
all 30 AIME 2026 problems are test items, and the mathematics training tasks are
compiled from earlier AIME problems. The six OOD benchmarks (TriviaQA, MuSiQue,
GPQA, MATH-Hard, SWE-Bench Verified, WebShop) are evaluation-only. Every other
test set has 128 items, so one run's 0/1 metric lies on a $k/128$ grid ($k/30$ for
AIME 2026) and a reported five-run mean on a $k/640$ grid ($k/150$); GPQA uses 128 random
Diamond questions, and MATH-Hard is the level-5 subset of MATH. The exact training and
test splits and the full configuration are released in our code repository.

\paragraph{Metrics.}
Answer exact match (Ans EM) and token-level F1 (Ans F1) use SQuAD-style answer
normalization and the maximum over reference aliases. Accuracy (Acc.) is the
share of correct final answers: the chosen option on MedQA and GPQA and the final
value on AIME 2026 and MATH-Hard. Pass@1 on MBPP+ and the success rate (SR) on
WebShop follow the official evaluators; SR on ALFWorld is the share of episodes
that complete the task within 50 steps. The resolved rate is the share of
SWE-Bench Verified issues whose patch passes the official tests.

\paragraph{Baselines and fairness.}
Unless a column is labeled otherwise, every method uses Qwen3.5-9B both as the
frozen executor and as the model it trains, and every trained method uses the
same training tasks; SFT fits the reference answers of those tasks. Each baseline
runs in the best configuration reported by its authors. The DeepSeek-V4-Flash
column prompts that model directly, as a larger reference point. GRPO$^\dagger$
in Table~\ref{tab:main} fine-tunes the backbone itself, whereas GRPO in
Figure~\ref{fig:mechanism}(a,b) trains $\pi_\theta$ inside the EvoSteer
architecture with the same harness, actions, and tools.

\paragraph{Ablation variants and fixed paradigms.}
Each ablation in Table~\ref{tab:ablation} removes one mechanism and restores the usual
alternative: $-$ interleaved execution builds the whole graph before running it;
$-$ execution features keeps interleaved execution and the textual feedback
$o_t^{\mathrm{exec}}$ but removes the feature vector $f$ from the state of $\pi_\theta$,
so the orchestrator reads the textual history and the value estimate but not $f$, while the
value head and $b_\psi$ are unchanged;
$-$ learnable repair masks \textsc{rerun} and \textsc{drop}; $-$ reference value head
withholds $\hat v_k$ and $\Delta\hat v$ from the state; $-$ measured flows learns $u_q$
as a free scalar instead of reading it from $\hat p_q$; $-$ flow corrections turns off
$c(g)$ and $b_\psi$; $-$ skill evolution removes the whole skill mechanism (no author
window, no candidates, no paired rollouts, and an empty skill library), so the team uses
roles and tools only; $-$ sequential validation admits a candidate after a fixed number
of successes. The four fixed paradigms decide the team before execution, on the same
executor and budget as EvoSteer: a ReAct-style agent with all tools
\citep{yao2022react}, one hand-designed multi-agent template for all tasks, a planner
that writes the graph once and then runs it, and a workflow searched offline on the
training tasks and then frozen.

\paragraph{Test-time protocol and transfer.}
At test time EvoSteer runs the trained $\pi_\theta$ with the value head and the
skill library frozen; no reference, paired, or validation rollouts are drawn. The
transfer study of Figure~\ref{fig:transfer} reuses this orchestrator unchanged
and swaps only the executor, called through its API identifier: gpt-5.6-luna
(GPT-5.6 Luna), grok-4.5 (Grok~4.5), claude-haiku-4-5 (Claude Haiku~4.5),
deepseek-v4-flash (DeepSeek-V4-Flash), gemini-3.5-flash (Gemini~3.5 Flash), and
glm-5.3-flash (GLM-5.3-Flash). Frozen-backbone scores prompt each model directly,
as in the v4-flash column of Table~\ref{tab:main}. Each OOD benchmark uses its IID
counterpart's task type (e.g., MuSiQue that of HotpotQA, SWE-Bench that of MBPP+), so
OOD results test generalization across matched task types.

\paragraph{Skill author.}
The skill author only writes candidate skills in author windows; validation, training, and
testing are unchanged. Table~\ref{tab:author} replaces DeepSeek-V4-Flash with a stronger
author, GPT-5.6-Luna, and with the executor itself, Qwen3.5-9B, and reports the averages of
Table~\ref{tab:main}. The stronger author raises every average, self-written skills lower
them by at most 1.56 points, and all three settings stay above the strongest baseline of
Table~\ref{tab:main} in every column.

\begin{table}[h]
\centering\small
\begin{tabular}{lcccc}
\toprule
 & \multicolumn{2}{c}{IID avg.} & \multicolumn{2}{c}{OOD avg.} \\
\cmidrule(lr){2-3}\cmidrule(lr){4-5}
Skill author & Ans EM & Acc. & Ans EM & Acc. \\
\midrule
DeepSeek-V4-Flash (main setting, Table~\ref{tab:main}) & 88.67 & 87.42 & 90.08 & 78.55 \\
GPT-5.6-Luna & 90.16 & 89.28 & 90.55 & 79.34 \\
Qwen3.5-9B (the executor) & 87.81 & 86.64 & 88.98 & 76.99 \\
\bottomrule
\end{tabular}
\caption{EvoSteer with different skill authors: five-run means averaged as in
Table~\ref{tab:main} (Ans EM over the question-answering benchmarks, Acc.\ over the others).}
\label{tab:author}
\end{table}

\paragraph{Objectives and cost.}
In Figure~\ref{fig:mechanism}(a,b) every objective trains the same untrained architecture
of Table~\ref{tab:ablation}, with the harness, action space, tool set, and skill library
unchanged. Tempered TB follows the implementation of \citet{zhang2026skillflow}: the log
reward is tempered as $\beta\log(R+\epsilon)$, $\log Z$ and the backward policy
are learned, and each edge's log-probability is normalized by its token count.
GPU time per episode is the compute of the parameter update (PPO's counts actor
and critic, GRPO's includes its KL term). Tokens per problem is the rollout cost,
prefill plus decode; training rollouts run the same orchestration as inference,
so it is also the inference cost. Training EvoSteer for 240 steps used 95.98 H800 GPU hours
and 1,580,691,323 tokens in total. The fixed-count admission of
Table~\ref{tab:ablation} admits a candidate after $k=5$ successes; every $k$
from 1 to 10 gives nearly the same results.

\paragraph{Diagnostics.}
In Figure~\ref{fig:transfer}(b), the reference curve is the accuracy of the
frozen $\rho$ on the natural reference rollouts of the same batch, which changes
with the batch although $\rho$ does not; the anchor-only level evaluates the
AnchorTB loss with $b_\psi=0$ and $\Delta\ell_t=0$, so that each flow equals
its measured part $\operatorname{clip}_{[0,\beta]}(u_{q,x}+c(g))$. In
Figure~\ref{fig:mechanism}(c), the predictor is the leave-one-out task anchor
$\hat p_q$ (equivalently $u_q$, a monotone transform with the same AUC), the
target is the final correctness of every natural reference rollout after step 2,
and the prediction is made before the episode starts; the task-family mean uses
the same reference statistics at the same step. In
Figure~\ref{fig:mechanism}(d), the gain of a skill is its paired effect: the same
task and initial state, the first action forced to bind the skill or not, and
both arms continued by the frozen $\rho$. In Figure~\ref{fig:mechanism}(e), an
edit is better, the same, or worse when the task grader's score of the
designated answer rises, stays, or falls from just before to just after the edit
on the same trajectory. In Figure~\ref{fig:mechanism}(f), the orchestrator is trained once with each of the Qwen3.5-9B
and DeepSeek-V4-Flash executors, and the gain of value-guided replanning is
measured during training.

\paragraph{Runs and aggregation.}
Every test score is the mean over five independent runs. Table~1
also reports their standard deviation; for the Avg.\ rows it
is $\sqrt{\sum_b\sigma_b^2}/n$ over $n$ benchmarks.

\section{Case Study}
\label{sec:cases}

\subsection{Code That Passes the Visible Test}
We present a case from MBPP+ in which the executor returns code that looks finished: it is
short, it runs, and it passes the only test given in the task. The code is wrong, and
$\hat v$ falls by $0.216$ at that point; the orchestrator reruns the agent instead of
submitting, and the rerun passes every hidden test.

\begin{probbox}{MBPP+ Problem 801 (training step 87)}
Write a python function to count the number of equal numbers from three given integers.
Your code should pass this test:\\
\texttt{assert test\_three\_equal(1,1,1) == 3}\\[2pt]
\textbf{Ground Truth (hidden tests; only the first is shown in the task):}
\texttt{test\_three\_equal(1,1,1) == 3}, \texttt{test\_three\_equal(-1,-2,-3) == 0},
\texttt{test\_three\_equal(1,2,2) == 2}
\end{probbox}

\noindent\textbf{Final Team:} \nd{planner n0} with the skill below (output agent)\qquad(6
rounds, 4 executor calls)

\begin{skillbox}{Skill on the Menu: Implement small Python functions from behavioral specifications}
\textbf{Purpose:} Translate the task wording into the simplest function that preserves the
specified semantics, especially indexing, ordering, and duplicate-handling rules, then return
executable code rather than planning or verification prose.\\
\textbf{Use when:} Use when asked to write a short Python function from a natural-language
requirement and one or more example assertions.\\
\textbf{Steps:}
\begin{enumerate}[leftmargin=1.6em,itemsep=0pt,topsep=1pt,parsep=0pt]
\item Extract the exact function name, parameters, return value, and indexing convention from
the prompt.
\item Identify semantic keywords before coding: positional ``kth element'' means access the
original sequence unless sorting is explicitly requested; [\ldots]
\item Choose the direct operation or standard algorithm matching those semantics. [\ldots]
\item Mentally test more than the supplied example, including empty or boundary positions
where applicable, duplicates for insertion behavior, unsorted input for positional retrieval,
and negative values for extrema.
\item Return a complete, syntactically valid function definition with the exact requested
name. Keep the final response focused on executable code.
\end{enumerate}
\textbf{Pitfalls to avoid:} Do not infer behavior solely from a single assertion: sorting an
array can accidentally pass a positional-access example while changing the required
semantics. Do not confuse right insertion with lower-bound insertion; [\ldots] Do not output
only a plan, review, or statement that code is correct. [\ldots] Avoid adding arbitrary
validation or changed exception behavior unless required by the specification.\\
\textbf{Constraint:} Preserve the input's order unless the task explicitly requests sorting
or an order statistic. Match the requested function name and argument order exactly. Provide
closed, valid Python code and ensure the function returns the result rather than merely
printing it. Do not include source-specific answers or rely only on the visible test case.
\end{skillbox}

\loghead{Round-by-Round Interaction Log}

\begin{roundbox}{Round 1--2: Planner n0 Writes a Plan}
\rl{Round 1 - Observation} Empty team; $\hat v=0.605$; 15 legal actions; the menu holds the
skill above and one other Python skill.
\rl{Round 1 - Action} \act{\{"kind":"\kd{ADD\_AGENT}", "node\_id":"n0", "role\_id":"planner"\}}
\rl{Round 1 - Input Prompt}\\
\textit{Role: Write a short step-by-step plan for the public task. Do not solve it.}\\
\textit{Task:} \slot{problem above}\\
\textit{Perform your role using the task. Return your result.}
\rsep
\rl{Round 2 - Observation} n0 holds a plan without code; $\hat v=0.623$
($\Delta\hat v=\up{+0.018}$); 20 legal actions.
\rl{Round 2 - Action} \act{\{"kind":"\kd{RERUN\_AGENT}", "node\_id":"n0"\}}
\tcblower
\rl{Round 1 - Executor Output} (951 characters; excerpt)\\\relax
[\ldots] 3. \textbf{Compare Pairs}: Check if the first number equals the second. [\ldots]
\textit{Note: Depending on the specific interpretation of ``equal numbers'' in the context of
the test case \texttt{test\_three\_equal(1,1,1) == 3}, the logic likely needs to count every
matching pair found.} [\ldots]
\par\smallskip
\rl{Round 2 - Executor Output} (981 characters; excerpt)\\\relax
[\ldots] \textit{Note: To satisfy the test case \texttt{test\_three\_equal(1,1,1) == 3}, the
logic must count every unique pair comparison that results in equality.}
4. \textbf{Increment Logic}: If any of the pair comparisons are true, increment the
\texttt{count} variable by 1. [\ldots]
\par\smallskip
\rl{Round 1--2 - Feedback} [Status]: answered, no code $\mid$ [Tokens]: 86 in, 243 out; 338
in, 256 out
\rl{Team State} (empty) $\to$ \nd{planner n0: plan}
\rl{Grader (hidden)} \xmark{no code}
\end{roundbox}

\begin{roundbox}{Round 3: Bind the Skill to n0; the Code Looks Finished}
\rl{Round 3 - Observation} n0 holds a plan that counts equal pairs; $\hat v=0.569$
($\Delta\hat v=\dn{-0.054}$); 20 legal actions.
\rl{Round 3 - Action} \act{\{"kind":"\kd{BIND\_SKILL}", "node\_id":"n0",
"skill\_id":"evosteer-2744c58b1d3dd27b8cb7"\}}
\rl{Round 3 - Input Prompt} (the role asks for a plan and the skill asks for executable code)\\
\textit{Role: Write a short step-by-step plan for the public task. Do not solve it.}\\
\textit{Task:} \slot{problem above}\\
\textit{Procedure to follow (skill evosteer-2744c58b1d3dd27b8cb7):} \slot{skill above}\\
\textit{Your previous output:} \slot{the plan of Round 2}\\
\textit{Perform your role using the task, the procedure and your previous output. Return your
result.}
\tcblower
\rl{Round 3 - Executor Output} (177 characters; full)
{\ttfamily\footnotesize
\par\noindent def test\_three\_equal(a, b, c):
\par\noindent\hspace*{2em}count = 0
\par\noindent\hspace*{2em}if a == b:
\par\noindent\hspace*{4em}count += 1
\par\noindent\hspace*{2em}if b == c:
\par\noindent\hspace*{4em}count += 1
\par\noindent\hspace*{2em}if a == c:
\par\noindent\hspace*{4em}count += 1
\par\noindent\hspace*{2em}return count\par}
\par\smallskip
\rl{Round 3 - Feedback} [Status]: answered $\mid$ [Answer]: code $\mid$ [Tokens]: 872 in, 65
out
\rl{Surface Check} \up{passes the visible test} $(1,1,1)\mapsto3$; also $(-1,-2,-3)\mapsto0$
\rl{Grader (hidden)} \xmark{wrong: $(1,2,2)\mapsto1$, expected $2$}
\end{roundbox}

\begin{roundbox}{Round 4: $\hat v$ Falls, and n0 Is Rerun Instead of Submitted}
\rl{Round 4 - Observation} n0 holds the code above in full, with an answer present and no
failure; $\hat v=0.353$ ($\Delta\hat v=\dn{-0.216}$), the largest drop of the episode; 19
legal actions, including \textsc{set\_output}.
\rl{Round 4 - Action} \act{\{"kind":"\kd{RERUN\_AGENT}", "node\_id":"n0"\}}
\rl{Round 4 - Input Prompt}\\
\textit{Role:} \slot{planner instruction}\\
\textit{Task:} \slot{problem above}\\
\textit{Procedure to follow (skill evosteer-2744c58b1d3dd27b8cb7):} \slot{skill above}\\
\textit{Your previous output:} \slot{the code of Round 3}\\
\textit{Perform your role using the task, the procedure and your previous output. Return your
result.}
\tcblower
\rl{Round 4 - Executor Output} (162 characters; full)
{\ttfamily\footnotesize
\par\noindent def test\_three\_equal(a, b, c):
\par\noindent\hspace*{2em}if a == b == c:
\par\noindent\hspace*{4em}return 3
\par\noindent\hspace*{2em}elif a == b or b == c or a == c:
\par\noindent\hspace*{4em}return 2
\par\noindent\hspace*{2em}else:
\par\noindent\hspace*{4em}return 0\par}
\par\smallskip
\rl{Round 4 - Feedback} [Status]: answered $\mid$ [Answer]: code $\mid$ [Tokens]: 681 in, 59
out
\rl{Team State} \nd{planner n0: corrected code}
\rl{Grader (hidden)} \cmark{} (3 of 3 tests)
\end{roundbox}

\begin{roundbox}{Round 5--6: Set Output and Stop}
\rl{Round 5 - Observation} $\hat v=0.396$ ($\Delta\hat v=\up{+0.043}$); 19 legal actions.
\rl{Round 5 - Action} \act{\{"kind":"\kd{SET\_OUTPUT}", "node\_id":"n0"\}}
\rl{Round 6 - Observation} n0 is the output agent; $\hat v=0.521$ ($\Delta\hat v=\up{+0.125}$);
18 legal actions.
\rl{Round 6 - Action} \act{\{"kind":"\kd{STOP}"\}}
\rl{Final Status} [Output agent]: \nd{planner n0} $\mid$ [Tests]: 3 of 3 $\mid$ [Reward]:
\up{1.0}
\end{roundbox}

\noindent\textbf{Key Observations:} At Round~4 every visible sign says the
task is done: the call returned an answer without failure, the code is short and valid, and it
passes the only test in the task. The code is still wrong, because it counts equal pairs and
returns 1 on $(1,2,2)$. $\hat v$, which reads execution features and the task type, not
the code, falls by $0.216$ to $0.353$ at this state, its largest drop in the
episode, and the orchestrator reruns the agent instead of setting it as the output. The rerun
rewrites the function to count equal numbers and passes all three hidden tests, and $\hat v$
rises by $0.043$ and $0.125$. The rewrite follows the bound skill's first pitfall, which
warns against inferring behavior from a single assertion. Four of the other five
rollouts of this problem in the same batch fail: three, including both reference rollouts,
submit code that passes the visible test and fails a hidden one, and one returns a Boolean.
The only other rollout that passes is the paired one that binds the same skill.

\end{document}